\documentclass[sts]{imsart}

\usepackage{amsthm}
\usepackage{amsfonts}
\usepackage{amsmath}
\usepackage{latexsym}
\usepackage{bm}
\newcommand{\blind}{0}
\usepackage{lipsum} % for filler text
\usepackage{adjustbox}
\usepackage{array} % control line height in tab
\usepackage{booktabs} % control line height in tab 
\usepackage{enumerate}
\RequirePackage[numbers]{natbib}
\RequirePackage[colorlinks,citecolor=blue,urlcolor=blue]{hyperref}
\RequirePackage{graphicx}

\usepackage{url} 

\usepackage{float}
\usepackage{hyperref}
\usepackage{color}
\usepackage{xcolor}
\usepackage{comment}
\usepackage{multirow}
\usepackage{enumitem} % enumerate sub-assumptions
\usepackage{bbm} %\mathbbm{1} To display indicator 1
\usepackage[ruled,linesnumbered]{algorithm2e}

\SetCommentSty{mycommfont}
\usepackage[title, titletoc, toc,header]{appendix} % appendix table of content
\usepackage[]{minitoc}
\noptcrule

\usepackage{mathabx} % to use wide check
\usepackage{tabularx} % auto break-line 
\newtheorem{assump}{Assumption}

\newcommand{\floor}[1]{\left\lfloor #1 \right\rfloor}
\usepackage{xr}
\usepackage{subfiles} % Use subfiles

\startlocaldefs
\theoremstyle{plain}

\newtheorem{theorem}{Theorem}[section]
\newtheorem{proposition}[theorem]{Proposition}
\newtheorem{lemma}[theorem]{Lemma}
\newtheorem{corollary}[theorem]{Corollary}

\theoremstyle{remark}
\newtheorem{definition}[theorem]{Definition}
\newtheorem{remark}[theorem]{Remark}
\newtheorem{example}{Example}

\endlocaldefs

\begin{document}

\newcommand{\E}{{\mathrm E}}
\newcommand{\Var}{{\mathrm{Var}}}
\newcommand{\Cov}{{\mathrm {Cov}}}

\newcommand{\R}{\mathbb R}
\newcommand{\N}{\mathbb N}
\newcommand{\Z}{\mathbb Z}
\newcommand{\F}{\mathbb F}
\newcommand{\Q}{\mathbb Q}
\newcommand{\C}{\mathbb C}

% shortcuts for symbol names that are too long to type
\newcommand{\eps}{\epsilon}
\newcommand{\lam}{\lambda}
\renewcommand{\l}{\ell}
\newcommand{\la}{\langle}
\newcommand{\ra}{\rangle}
\newcommand{\wh}{\widehat}
\newcommand{\wt}{\widetilde}

% calligraphic letters
\newcommand{\calA}{{\cal A}}
\newcommand{\calB}{{\cal B}}
\newcommand{\calC}{{\cal C}}
\newcommand{\calD}{{\cal D}}
\newcommand{\calE}{{\cal E}}
\newcommand{\calF}{{\cal F}}
\newcommand{\calG}{{\cal G}}
\newcommand{\calH}{{\cal H}}
\newcommand{\calI}{{\cal I}}
\newcommand{\calJ}{{\cal J}}
\newcommand{\calK}{{\cal K}}
\newcommand{\calL}{{\cal L}}
\newcommand{\calM}{{\cal M}}
\newcommand{\calN}{{\cal N}}
\newcommand{\calO}{{\cal O}}
\newcommand{\calP}{{\cal P}}
\newcommand{\calQ}{{\cal Q}}
\newcommand{\calR}{{\cal R}}
\newcommand{\calS}{{\cal S}}
\newcommand{\calT}{{\cal T}}
\newcommand{\calU}{{\cal U}}
\newcommand{\calV}{{\cal V}}
\newcommand{\calW}{{\cal W}}
\newcommand{\calX}{{\cal X}}
\newcommand{\calY}{{\cal Y}}
\newcommand{\calZ}{{\cal Z}}

% bold letters (useful for random variables)
\renewcommand{\a}{{\boldsymbol a}}
\renewcommand{\b}{{\boldsymbol b}}
\renewcommand{\c}{{\boldsymbol c}}
\renewcommand{\d}{{\boldsymbol d}}
\newcommand{\e}{{\boldsymbol e}}
\newcommand{\f}{{\boldsymbol f}}
\newcommand{\g}{{\boldsymbol g}}
\newcommand{\h}{{\boldsymbol h}}
\renewcommand{\i}{{\boldsymbol i}}
\renewcommand{\j}{{\boldsymbol j}}
\renewcommand{\k}{{\boldsymbol k}}
\newcommand{\m}{{\boldsymbol m}}
\newcommand{\n}{{\boldsymbol n}}
\renewcommand{\o}{{\boldsymbol o}}
\newcommand{\p}{{\boldsymbol p}}
\newcommand{\q}{{\boldsymbol q}}
\renewcommand{\r}{{\boldsymbol r}}
\newcommand{\s}{{\boldsymbol s}}
\renewcommand{\t}{{\boldsymbol t}}
\renewcommand{\u}{{\boldsymbol u}}
\renewcommand{\v}{{\boldsymbol v}}
\newcommand{\w}{{\boldsymbol w}}
\newcommand{\x}{{\boldsymbol x}}
\newcommand{\y}{{\boldsymbol y}}
\newcommand{\z}{{\boldsymbol z}}
\newcommand{\A}{{\boldsymbol A}}
\newcommand{\B}{{\boldsymbol B}}
\newcommand{\D}{{\boldsymbol D}}
\newcommand{\G}{{\boldsymbol G}}
\renewcommand{\H}{{\boldsymbol H}}
\newcommand{\I}{{\boldsymbol I}}
\newcommand{\J}{{\boldsymbol J}}
\newcommand{\K}{{\boldsymbol K}}
\renewcommand{\L}{{\boldsymbol L}}
\newcommand{\M}{{\boldsymbol M}}
\renewcommand{\O}{{\boldsymbol O}}
\renewcommand{\S}{{\boldsymbol S}}
\newcommand{\T}{{\boldsymbol T}}
\newcommand{\U}{{\boldsymbol U}}
\newcommand{\V}{{\boldsymbol V}}
\newcommand{\W}{{\boldsymbol W}}
\newcommand{\X}{{\boldsymbol X}}
\newcommand{\Y}{{\boldsymbol Y}}

\newcommand{\poly}{\mathrm{poly}}

% useful for Fourier analysis
\newcommand{\bits}{\{-1,1\}}
\newcommand{\bitsn}{\{-1,1\}^n}
\newcommand{\bn}{\bitsn}
\newcommand{\isafunc}{{: \bitsn \rightarrow \bits}}
\newcommand{\fisafunc}{{f : \bitsn \rightarrow \bits}}

% if you want
\newcommand{\half}{{\textstyle \frac12}}

%%%%%%%%%%%%%%%%%%%%%%%%%%%%%%%%%%%%%%%%%%%%%%%%%%%%%%%%%%%%%%%%%

% My newcomands

\newcommand{\vu}{\hat V_u}
\newcommand{\Skk}{S^{(2k)}}
\newcommand{\hh}{h_{(0)}}
\newcommand{\fk}{F^{(k)}}
\newcommand{\bfk}{\bar{F}^{(k)}}
\newcommand{\unr}{\underline{r}}
\newcommand{\unx}{\underline{x}}
\newcommand{\uny}{\underline{y}}
\newcommand{\unrMarg}{\underline{r}^{*}}
\newcommand{\pdiff}{p_{\underline{x}, \underline{y}, d_1, d_2, c}^{(diff)} }
\newcommand{\covh}{\rho}
\newcommand{\covhh}{\rho(\unr, d_1, d_2)}
\newcommand{\covhhh}{\rho(\unr', d_1, d_2)}
\newcommand{\rhort}{\tilde \rho(\unr)}
\newcommand{\rhor}{\rho(\unr)}
\newcommand{\checksigone}{\widecheck{\sigma}_{1, 2k}^2}
\newcommand{\checksigc}{\widecheck{\sigma}_{c, 2k}^2}
\newcommand{\etacd}{\eta_{c, 2k}^2 (d_1, d_2)}

% abbr. for kernels
\newcommand{\phiib}{\hib}
\newcommand{\phibbar}{\bar h_{(b)}}
\newcommand{\phibbarinc}{\tilde h_{(b)}}
\newcommand{\phibar}{U_n}
\newcommand{\phisqbar}{\frac{1}{M\BB} \sum_{b = 1}^{\BB} \sum_{i = 1}^M (h(S_{i}^{(b)}))^2}
\newcommand{\phibsqbar}{\frac{1}{\BB} \sum_{b = 1}^{\BB}  (\bar h_{(b)})^2}
\newcommand{\hib}{h(S_{i}^{(b)})}
\newcommand{\Sib}{S_{i}^{(b)}}

\newcommand{\allassump}{Assumptions \ref{assumption: k = n^(1/2 - epsilon)}-\ref{assumption:4th moment}}
\newcommand{\alterassump}{Assumptions \ref{assumption: k = n^(1/2 - epsilon)}, \ref{assumption:2nd moment}, \ref{assumption:non-negative and ordinal} - \ref{assumption:low dependency (relaxed)}}

\newcommand{\propEstFull}{Matched Sample Variance Estimator}
\newcommand{\propSmoothFull}{Matched Sample Smoothing Variance Estimator}
\newcommand{\rfVarEst}[1]{ {\hat \sigma_{RF}^2 (\bm x^*_{#1}) } }
\newcommand{\rfVarEstS}{ {\overline{\hat \sigma_{RF}^2} (\bm x^*)} }

\newcommand{\propEst}{MS}
\newcommand{\propSmooth}{MS-s}

\newcommand{\varh}{\Var (h(S))}
\newcommand{\xid}{\xi_{d, k}^2}
\newcommand{\xidk}[1]{\xi_{#1, k}^2}
\newcommand{\xidtilde}{\tilde{\xi}_{d, k}^2}
\newcommand{\hatxidtilde}{\hat{\tilde\xi}_{d, k}^2}
\newcommand{\varU}{\Var(U_n)} 
\newcommand{\gammad}{\gamma_{d,k,n}}
\newcommand{\hatgammad}{{\gamma}_{d,B}}
\newcommand{\Uinc}{U_{n, B}}
\newcommand{\varhtilde}{\widetilde{\Var}_B(h(S))}
\newcommand{\varhtildeComp}{\widetilde{\Var}_{B^*}(h(S))}
\newcommand{\Nd}{N_{d,k,n}}
\newcommand{\Gn}{\calG_{n, k, M}}

% variance estimator
\newcommand{\hatVarU}{\widehat{\Var}(U_n)}
\newcommand{\vh}{V^{(h)}}
\newcommand{\vs}{V^{(s)}}
\newcommand{\hatvh}{\hat{V}^{(h)}}
\newcommand{\hatvs}{\hat{V}^{(s)}}
\newcommand{\hatvhB}{\hat{V}_{B, M}^{(h)}}
\newcommand{\hatvsB}{\hat{V}_{B, M}^{(s)}}
\newcommand{\UincMatch}{U_{n, B, M}}
\newcommand{\BB}{\mathbbm{B}}
\newcommand{\NN}{\mathbbm{N}}
\newcommand{\hij}{h(S_{i}^{(j)})}
\newcommand{\hji}{h(S_{i'}^{(j')})}
\newcommand{\vijij}{\hat v_{(i,j,i',j')}}

\newcommand{\est}[1]{\widehat {\Var} (\hat f(#1))}
\newcommand{\red}{\color{red}}
\newcommand{\nTrees}{\texttt{nTrees}}
\newcommand{\mtry}{\texttt{mtry}}

% marginal r
\newcommand{\rx}[1]{r_{#1*}}
\newcommand{\ry}[1]{r_{*#1}}
% d1 d2 are default parameters
%\newcommand{\px}[1][d_1]{p_{(\underline{r}_{x \cdot}, #1, c)}}
%\newcommand{\py}[1][d_2]{p_{(\underline{r}_{\cdot y}, #1, c)}}
\newcommand{\px}[1][d_1]{p_{(\rx{0}, \rx{1}, \rx{2}, #1, c)}}
\newcommand{\py}[1][d_2]{p_{(\ry{0}, \ry{1}, \ry{2}, #1, c)}}
\newcommand{\pxx}[1][d]{p_{(x_{0}, x_{1}, x_{2}, #1, c)}}
\newcommand{\pyy}[1][d]{p_{(y_{0}, y_{1}, y_{2}, #1, c)}}
\newcommand{\pr}{p_{(\unr, d_1, d_2, c)}}
\newcommand{\prrMarg}{p_{(\unr,\, \unrMarg)}}
\newcommand{\Rrand}{\mathbf{R}}
\newcommand{\RrandM}{\mathbf{R^*}}
\newcommand{\Rx}[1]{\mathbf{R}_{#1 *}}
\newcommand{\Ry}[1]{\mathbf{R}_{* #1}}

\newcommand{\dpg}{\overline{\Delta pg}(\unrMarg, d_1, d_2, c)}
\newcommand{\pg}[2]{\overline{pg}(\unrMarg, #1, #2, c)}

\newcommand{\tabSpace}{\specialrule{0em}{4pt}{4pt}}

\begin{frontmatter}
\title{On Variance Estimation of Random Forests with Infinite-Order U-statistics}
%Variance Estimation of Infinite-Order U-statistics with Application to Quantifying Uncertainty of Random Forests}

%     \author{\name Tianning Xu
%     \email tx8@illinois.edu \\
%     \addr Department of Statistics \\
%     University of Illinois Urbana-Champaign \\ Champaign, IL, 61820, USA
%     \AND
%     \name Ruoqing Zhu
%     \email rqzhu@illinois.edu \\
%     \addr Department of Statistics \\
%     University of Illinois Urbana-Champaign \\ Champaign, IL, 61820, USA
%     \AND
%     \name Xiaofeng Shao
%     \email xshao@illinois.edu \\
%     \addr Department of Statistics \\
%     University of Illinois Urbana-Champaign \\ Champaign, IL, 61820, USA
% }
% \editor{}

\begin{aug}
%%%%%%%%%%%%%%%%%%%%%%%%%%%%%%%%%%%%%%%%%%%%%%%
%% ORCID can be inserted by command:         %%
%% \orcid{0000-0000-0000-0000}               %%
%%%%%%%%%%%%%%%%%%%%%%%%%%%%%%%%%%%%%%%%%%%%%%%
\author[A]{\fnms{Tianning}~\snm{Xu}\ead[label=e1]{tx8@illinois.edu}},
\author[B]{\fnms{Ruoqing}~\snm{Zhu}\ead[label=e2]{rqzhu@illinois.edu}}
\and
\author[C]{\fnms{Xiaofeng}~\snm{Shao}\ead[label=e3]{xshao@illinois.edu}}

\address[A]{Tianning Xu is Ph.D. Student, Department of Statistics,
University of Illinois Urbana-Champaign, Champaign, IL, 61820\printead[presep={\ }]{e1}.}

\address[B]{Ruoqing Zhu is Associate Professor, Department of Statistics,
University of Illinois Urbana-Champaign, Champaign, IL, 61820\printead[presep={\ }]{e2}.}

\address[C]{Xiaofeng Shao is Professor, Department of Statistics,
University of Illinois Urbana-Champaign, Champaign, IL, 61820\printead[presep={\ }]{e3}.}

\end{aug}

\maketitle

\begin{abstract}

Infinite-order U-statistics (IOUS) has been used extensively on subbagging ensemble learning algorithms such as random forests to quantify its uncertainty. While normality results of IOUS have been studied extensively, its variance estimation approaches and theoretical properties remain mostly unexplored. Existing approaches mainly utilize the leading term dominance property in the Hoeffding decomposition. However, such a view usually leads to biased estimation when the kernel size is large or the sample size is small. On the other hand, while several unbiased estimators exist in the literature, their relationships and theoretical properties, especially the ratio consistency, have never been studied. These limitations lead to unguaranteed performances of constructed confidence intervals. To bridge these gaps in the literature, we propose a new view of the Hoeffding decomposition for variance estimation that leads to an unbiased estimator. Instead of leading term dominance, our view utilizes the dominance of the peak region. Moreover, we establish the connection and equivalence of our estimator with several existing unbiased variance estimators. Theoretically, we are the first to establish the ratio consistency of such a variance estimator, which justifies the coverage rate of confidence intervals constructed from random forests. Numerically, we further propose a local smoothing procedure to improve the estimator's finite sample performance. Extensive simulation studies show that our estimators enjoy lower bias and archive targeted coverage rates.

\end{abstract}


\end{frontmatter}

% Seperate ToC for appendix
\doparttoc % Tell to minitoc to generate a toc for the parts
\faketableofcontents % Run a fake tableofcontents 
\part{} % Start the document part

\section{Introduction}

Given a set of $n$ i.i.d. observations ${\cal D}_n = \{X_i\}_{i=1}^n$ and an unbiased estimator, $h (X_1, \ldots, X_k)$, of the parameter of interest $\theta$ with $k \leq n$, the U-statistic \citep{hoeffding1948} defined in the following is a minimum-variance unbiased estimator of $\theta$:
\begin{align}
\label{eq:def U-stat}
    U_{n}   
    &= \binom{n}{k} ^{-1} \sum_{S_i \subset {\cal D}_n} h(S_i)  \\
    &= \binom{n}{k} ^{-1} \sum_{1 \leq j_{1}<\cdots<j_{k} \leq n} h\left(X_{j_{1}}, \ldots, X_{i_{k}}\right), \nonumber
\end{align}
where each $S_i$ is a subset of $k$ samples from the original ${\cal D}_n$, where $k$ is called the kernel size. When $k$ grows with $n$, $U_n$ becomes an Infinite-Order U-statistic (IOUS) \citep{frees1989infiniteOrder}. Without the risk of ambiguity, we drop $k$ in the notation.

In recent years, there has been an increasing interest in statistical inference with IOUS, with application to subbagging ensemble approaches, such as random forests \citep{breiman2001random, geurts2006extremely}. It is easy to see that large $\binom{n}{k}$ renders the computationally infeasible to exhaust all subsamples. Instead, random forests sample $B$ subsamples from ${\cal D}_n$ to build trees and average. This leads to incomplete U-statistics \citep{ustatBook}. Further incorporating randomness in the kernel function $h$, \citet{lucas:old} first show the asymptotic normality of random forests under the U-statistics framework when $k$ grows at the rate of $o(\sqrt{n})$. \citet{Romano2020} further relax its assumptions. \citet{lucas:new} set the connection between U- and V-statistics. \citet{peng2019asymptotic} extend the kernel size to $k = o(n)$ under a generalized U-statistic framework. We also note, but mainly omit, a large literature outside the applications of random forests. For example, for incomplete high-dimensional U-statistics, where $h \in \R^d$, \citet{chen2019randomized} and \citet{song2019high_dim_ustat} study the asymptotic normality for fixed and growing $k$, respectively.

With the normality of random forest estimators established under the U-statistics \citep{lucas:old} or other frameworks \citep{2018wager:InfJack, athey2019generalized}, another line of the topic is the variance estimation. \citet{wager2014confidence} propose to use jackknife and infinitesimal jackknife (IJ, \citet{efron2014IJ}). \citet{lucas:old} use Monte Carlo methods to estimate the leading term in the Hoeffding decomposition of $\Var(U_n)$. Recent developments include \citet{lucas:new}, who propose a computationally efficient approach and set the connection with the IJ estimator. \citet{peng2021bias} further study the bias and consistency of the IJ estimator.

However, an essential practical issue is that these estimators can display a significant amount of bias when the sample size $n$ is small or $k$ is large compared to $n$. In practice, it is common to use a fixed proportion of the total sample size \citep{geurts2006extremely} as the kernel size $k$. Variance estimators in the aforementioned literature often suffer from this bias issue because they all rely on some form of leading term dominance phenomenon. However, when $k$ is large compared to $n$, such dominance is weak. Searching through the literature, several unbiased estimators have been proposed in different forms and aspects based on the U-statistics view. Some of them can handle a subsampling size $k$ as large as $n/2$.
\citet{folsom1984probability} propose a variance estimator of complete U-statistics following a sequence of literature on sampling design \citep{horvitz1952generalization, yates1953selection, sen1953estimate}. 
\citet{schucany1989small} propose to estimate all terms in the Hoeffding decomposition \citep{hoeffding1948} of the variance of an order-2 complete U-statistic. However, they do not extend the estimator to a general case with $k \leq n/2$. 
Note that \citet{folsom1984probability, schucany1989small} do not consider the incomplete case; hence their estimators are computationally infeasible for large $k$ or large $n$. 
More recently, \citet{wang2014variance} propose partition-based, unbiased variance estimators of both complete and incomplete U-statistics motivated from the second-moment expression $E(U_n^2) - E^2(U_n)$. \citet{wang2022quantifying} further apply this estimator to random forest variance estimation. However, there is a lack of theoretical justification for these estimators in terms of their ratio consistency, which is crucial for achieving a proper coverage rate based on the derived confidence interval. Moreover, there is a lack of understanding of their connections and differences with the estimators mentioned previously. 

To address these limitations in the literature, the major contribution of our paper is three-fold. 
First, we re-analyze the Hoeffding decomposition and propose a peak region dominance view of the variation estimation of U-statistics to address the bias issue. This leads to a class of unbiased estimation approaches for both complete and incomplete U-statistics, called \textit{\propEstFull}, which can handle a subsampling size $k$ as large as $n/2$. Computationally, our incomplete variance estimator is efficient and can be directly applied to random forests. Besides, we discuss two extensions of our estimators. One is a local smoothing strategy to mitigate negative variance estimations \citep{schucany1989small, wang2014variance}, and the other extends our method to $k > n/2$.
Secondly, we are the first to establish the connection and equivalence of the three existing estimators \citep{folsom1984probability, schucany1989small, wang2014variance}. We show that our proposed estimator coincides with each under specific settings (see Section \ref{sec:connection} for a detailed discussion). 
Thirdly, we establish the ratio consistency for our complete variance estimator under $k = o(\sqrt{n})$. To the best of our knowledge, this is the first result for such estimators, even for fixed $k$. This is a crucial step to achieve the nominal coverage level when we plug in the variance estimator in constructing a confidence interval. To this end, we fill a significant gap in the literature by proposing a set of interpretable conditions. 

We proceed with additional notation and preliminaries of U-statistics to motivate the proposed variance estimator and establish the peak region dominance view.

\section{Variance of U-statistics}
\label{sec:into U-stat}

Our analysis starts with a classical result of the variance of U-statistics. We first review the Hoefdding decomposition of the variance of a complete U-statistics. Then, we present the connection between the complete and incomplete versions. In particular, the variance of an order-$k$ complete U-statistics is given by \citet{hoeffding1948}:
\begin{align}
\label{eq:Decomposition of Var(U_n) Hoeffding}
    \Var\left(U_{n}\right)
    = \binom{n}{k}^{-1} \sum_{d=1}^{k} \binom{k}{d} \binom{n-k}{k-d} \xi^2_{d, k},
\end{align}
where $\xi^2_{d, k}$ is the covariance between two kernels $h(S_1)$ and $h(S_2)$ with $S_1$ and $S_2$ sharing $d$ overlapping observations, i.e., $\xi_{d, k}^2 = \Cov \left( h(S_1), h(S_2) \right)$, with $|S_1 \cap S_2| = d$. Here both $S_1$ and $S_2$ are size-$k$ subsamples. Alternatively, we can represent $\xi_{d, k}^2$ as \citep{ustatBook}
\begin{align}
\label{eq:xi_d,k, Var E(h|X)}
    \xi_{d, k}^2 = \Var\left[\E \left (h(S) | X_1, ..., X_d \right) \right].
\end{align}
This form will be utilized later.

When $k$ grows with $n$, it is computationally almost infeasible to exhaust all subsamples due to large $\binom{n}{k}$. Instead, it is typical in random forests and other ensemble algorithms to build incomplete infinite-order U-statistics \citep{ustatBook} by sampling $B$ many $S_i$'s, which gives
\begin{align}
\label{eq:incomplete U}
    U_{n, B} = \frac{1}{B} \sum_{i = 1}^B h(S_i).
\end{align}
The gap between variances of an incomplete U-statistic and its complete counterpart can be understood as
\begin{align}
\label{eq:var Uinc general}
    \Var(\Uinc) 
    &= \Var\left[ \E(\Uinc | \calX_n)  \right] + \E\left[ \Var(\Uinc | \calX_n)  \right] \\
    &= \varU + \E\left[ \Var(\Uinc | \calX_n)  \right] \nonumber,
\end{align}
where the additional term $\E[ \Var(\Uinc | \calX_n) ]$ depends on the subsampling scheme. In particular, when all subsamples are drawn with replacements from the collection of all such subsamples \citep{ustatBook}, we have
\begin{align}
\label{eq:var Uinc = varU + xik}
    \Var(\Uinc) = (1-\frac{1}{B}) \varU + \frac{1}{B} \xidk{k}.
\end{align}
This suggests that we can close the gap by using a large $B$. Hence, we will first discuss the complete U-statistics setting and then propose the incomplete one. We also note that for applications to random forests, random kernels (trees) are involved. However, the difference can be negligible when using a large $B$ \citet{lucas:old}.

\section{Methodology}
\label{sec:methodology}

The main technical challenge for estimating the variance is when $k$ is relatively large compared with $n$. Besides the aforementioned obvious computational issue in the complete version, most existing methods will also encounter a significant bias due to only estimating the leading term in the Hoefdding decomposition. By establishing a peak region dominance view, we develop a new unbiased estimator for $\varU$ in both complete and incomplete forms whenever $k \leq \frac{n}{2}$. Its connection with existing methods will be discussed in Section \ref{sec:connection}. Its extension to $n/2 < k < n$ setting will be presented in Section \ref{sec:subsec:extension k>n/2}. We demonstrate the application to random forests in Section \ref{sec:apply to rf}, where we also introduce a locally smoothed version for better numerical performances. 

\subsection{Existing Methods and Limitations}

Continuing from the decomposition of $\varU$ in Equation \eqref{eq:Decomposition of Var(U_n) Hoeffding}, we define $\gammad = \binom{n}{k}^{-1} \binom{k}{d} \binom{n-k}{k-d}$ for convenience. Then $\Var(U_n) = \sum_{d = 1}^k \gammad \xi_{d, k}^2$.
It is easy to see that $\gammad$ corresponds to the probability mass function of a hypergeometric distribution with parameters $n, k$ and $d$. A graphical demonstration of such coefficients under different $k$ and $d$ settings, with $n = 100$, is provided in Figure \ref{fig:HG distribution}. Many existing methods \citep{lucas:old, Romano2020} rely on the asymptotic approximation of $\Var(U_n)$ when $k$ is small, e.g., $k = o(n^{1/2})$. Under such settings, the first coefficient $\gamma_{1,k,n} = [1+o(1)] \frac{k^2}{n}$ dominates all remaining ones, as we can see in Figure \ref{fig:HG distribution} when $k= 10$. In this case, to estimate $\varU$, it suffices to estimate the leading covariance term $\xi_{1, k}^2$ if $\xi_{k, k}^2 / (k \xi_{1, k}^2)$ is bounded. 

\begin{figure}[ht]
    \centering
    \includegraphics[width=0.5\textwidth]{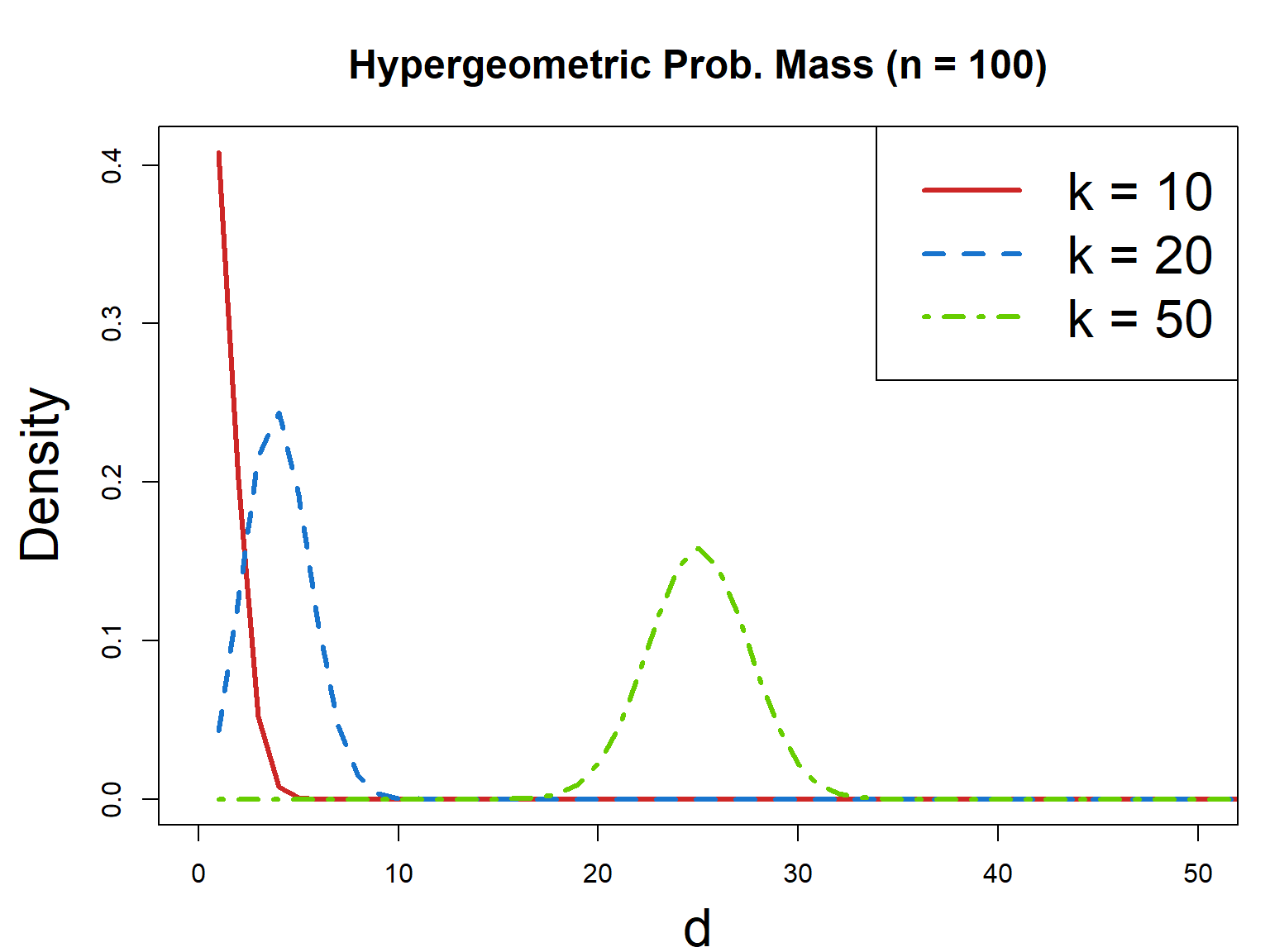}
    \caption{Probability mass function of hypergeometric distribution with $n$ = 100 and different $k$.}
    \label{fig:HG distribution}
\end{figure}

However, as $k$ becomes larger, the density of the hypergeometric distribution concentrates around $d = \beta^2 n$ instead of $d = 1$. Hence, the variance will be mainly determined by terms in a range of large $d$ values, which we refer to as the peak region. In comparison, estimating just $\xi_{1, k}^2$ will introduce a significant bias even if we are able to exhaust all possible subsamples. 

%     \begin{align*}
%         \gammad \sim \frac{e^{-a x^2 /2}}{\sqrt{2 \pi kt(1-t)^2}}, 
%     \end{align*}
%     where $a:= 1/(1-t)$ and $x_d := (d -kt) / \sqrt{kt(1-t)} \to x$.

Another source of bias for using the leading term dominance property is the lack of samples to estimate $\xidk{1}$ realistically. Note that the definition involves approximating the $\Var$ and $\E$ operations in Equation \eqref{eq:xi_d,k, Var E(h|X)} \cite{lucas:old, lucas:new}. A natural strategy is to hold one shared sample, e.g., $X_{(1)}$, and vary the remaining samples in $S$ among existing observations ${\cal D}_n$ to approximate $\E \left (h(S) | X_1\right)$. However, this causes trouble for the variance estimator since we won't have enough samples to independently produce estimators of $\E \left (h(S) | X_i\right)$ with varying $X_i$ when $k$ becomes slightly larger. Overall, a new strategy is needed to better utilize the Hoefdding decomposition.

We also note that another theoretical strategy proposed by \citet{2018wager:InfJack, peng2019asymptotic} can be used for $k = o(n)$ if the $U$-statistic can be understood through the Hajek projection with additional regularity conditions. In this case, the variance of a U-statistic can be well approximated by the variance of a linearised version, while the infinitesimal jackknife procedure \citet{efron1981jackknife} provides a valid estimator. However, it is difficult to assess whether the kernel function satisfies these assumptions. In practice, a significant bias can still occur, as seen in the simulation section.

\subsection{An Alternative View}

At this point, estimating $\xi_{d, k}^2$'s for some $d$ values seems inevitable. However, we may utilize the law of total variance to change the estimation procedure, which could gain a significant computational advantage. Note that for any given $d$,
\begin{align}
\label{eq:xi_d,k, decomposition}
    \xi_{d, k}^2 
    =&  \Var\left[\E \left (h(S) | X_1, ..., X_d \right) \right] \nonumber \\
    =&  \Var(h(S)) - \E\left[\Var(h(S) | X_1, ..., X_d) \right] \nonumber \\
    :=& \vh - \xidtilde,
\end{align}
where we define $\xidtilde := \E [\Var(h(S) | X_1, ..., X_d) ]$. In this representation, $\vh$ is equivalent to $\xidk{k}$, the variance of a single kernel. It is also equivalent to $\tilde{\xi}_{0, k}^2$ since $\xi_{0, k}^2 = 0$. Incorporating these into the decomposition formula in Equation \eqref{eq:Decomposition of Var(U_n) Hoeffding}, we obtain an interesting connection:
\begin{align} \label{eq:varU = varh - sum}
    \varU =& \sum_{d = 1}^k \gammad \left(\vh - \xidtilde \right) \nonumber \\
    =& \sum_{d = 0}^k \gammad \vh - \sum_{d=0}^k \gammad \xidtilde \nonumber \\
    :=& \vh - \vs,
\end{align}
where we define $\vs$ as $\sum_{d=0}^k \gammad \xidtilde$. Note that in the second equation, we add and subtract a term with $d = 0$ to complete the hypergeometric distribution coefficients. 

While this alternative view is valid for all $k$, the difficulty lies in finding a computationally feasible estimator, especially when we have to deal with incomplete U-statistics, instead of the complete version. In particular, when $k \leq n/2$, both terms can be unbiasedly estimated with a proper sampling design. In the following, we first present a straightforward formula for estimating $\vh$ and $\vs$ in a complete U-statistics version. The main result is Theorem \ref{prop:var_subTree}, which shows that $\vs$ can be estimated using a sample variance of all trees. Section \ref{sec:var est for incomplete U} extends these estimators to incomplete versions.

\subsection{Variance Estimation for Complete U-statistics}
\label{sec:complete est}

Our goal is to create estimators of $\vs$ and $\vh$ such that they can be directly computed  from the trees (kernels) fitted in the random forest itself. This seems to be a challenging task given that we are estimating an infinite sum $\vs$. However, the fundamental idea we will utilize is to estimate $\Var(h(S) | X_1, ..., X_d)$ using pairs of trees. We proceed with the complete case when all trees are already available.

\subsubsection{Joint Estimation of the Infinite Sum \texorpdfstring{$\vs$}{Vars}.}
\label{sec::est joint}

Suppose we pair subsamples $S_i$ and $S_j$ among $\binom{n}{k}$ subsamples and let $d = |S_i \cap S_j| = 0, 1, ..., k$. 
Then for each $d$, there exist $N_{d,k,n} = \binom{n}{k}^2 \gammad = \binom{n}{k} \binom{k}{d} \binom{n-k}{k-d}$ pairs of subsamples $S_i, S_j$ such that $|S_i \cap S_j| = d$. 
Note that for any such pair, $(h(S_i) - h(S_j))^2 /2$ is an unbiased estimator of $\xidtilde = \Var(h(S) | X_1, ..., X_d)$.
We may then construct an unbiased estimator of $\xidtilde$ by averaging them: 
\begin{align}
\label{eq:def:est of tilde xi}
\hatxidtilde =
    \frac{1}{\binom{n}{k} \binom{k}{d} \binom{n-k}{k-d}}
     \sum_{ |S_i \cap S_j| = d  } \left(h(S_i) - h(S_j)\right)^2 / 2.
\end{align}
This motivates us to combine all such terms in the infinite sum, which surprisingly leads to the sample variance of all kernels. The result is given in the following proposition, with its proof collected in Appendix \ref{sec:append:methoddology}.

\begin{proposition}
\label{prop:var_subTree}
Given a complete U-statistic $U_n$, and the estimator $\hatxidtilde$ defined in Equation \eqref{eq:def:est of tilde xi}, when $k \leq n/2$, we have the following unbiased estimator of $\vs$: 
%denote $\hatgammad := \binom{B}{2}^{-1} \Nd$.
\begin{align}
\label{eq:hat vs}
\hatvs :=& {\binom{n}{k}}^{-1} \sum_{i} \left( h(S_i) - U_n \right)^2 \\
=& \sum_{d= 0}^k \gammad \hatxidtilde \nonumber
\end{align}

Furthermore, when $k > n/2$, the first $2k - n$ terms in the summation $\sum_{d=0}^k$ is removed, since corresponding $\gammad$ terms are zero. 
\end{proposition}

Since $\hatvs$ enjoys a sample variance form, its incomplete version would also be easy to calculate. The advantage is that it can be computed without any hassle because all $h(S_i)$'s are ready to use when we calculate $U_n$. However, additional consideration may facilitate the estimation of $\vh$ so that both $\vs$ and $\vh$ can be done using the same set of $h(S_i)$'s.

\subsubsection{Estimation of Kernel Variance \texorpdfstring{$\vh$}{Varh}.}
\label{sec:est varh}

Estimating $\vh$ may follow the same idea using $\big[h(S_i) - h(S_j)\big]^2/2$ if the pair $S_i$ and $S_j$ are disjoint. However, this is only possible when $k \leq n/2$, given a finite sample. In this case, following Equation \eqref{eq:def:est of tilde xi}, we have an unbiased estimator of $\vh$:
\begin{align}
\label{eq:hat vh}
    \hatvh = \hat{\tilde\xi}_{0, k}^2 = 
    \frac{1}{\binom{n}{k} \binom{n-k}{k}} \sum_{|S_i \cap S_j| = 0} \left[h(S_i) - h(S_j)\right]^2 / 2.
\end{align}

Therefore, we combine estimators $\hatvs$ \eqref{eq:hat vs} and $\hatvh$ \eqref{eq:hat vh} to get an unbiased estimator of $\varU$:
\begin{align}
\label{eq:our est, complete}
    \widehat{\Var}(U_n) = \hatvh - \hatvs.
\end{align}

\subsection{Variance Estimation for Incomplete U-statistics}
\label{sec:var est for incomplete U}

In random forests and other ensemble learning models, we often construct incomplete U-statistics by drawing random subsamples instead of exhausting all $\binom{n}{k}$ subsamples. This creates difficulties in calculating $\hatvh$ since very few of these subsamples would be mutually exclusive ($d = 0$). Hence, a new subsampling strategy is needed to allow sufficient pairs of subsamples to estimate both $\vh$ and $\vs$. 

The following ``matched sample'' sampling scheme is proposed to have enough disjoint samples to estimate $\hatvs$. For any $2 \leq M \leq \floor{n/k}$, we can sample a set of the matched sample group that consists $M$ mutually exclusive subsamples $\{S_1, \ldots, S_M\}$ from ${\cal D}_n$. This enables us to estimate $\vh$ by the sample variance of $\left\{ h(S_1), \ldots, h(S_M) \right\}$. Then, we repeat this procedure $B$ times to average the estimator. To be precise, denote the subsamples in the $b$-th matched sample group as $S_1^{(b)},  S_2^{(b)}, ..., S_M^{(b)}$, such that $S_i^{(b)} \cap S_{i'}^{(b)} = \emptyset$ for any $i \neq i'$. Define 
\begin{align}
    \label{eq:def:UincMatch}
    \UincMatch = \frac{1}{MB} \sum_{i=1}^M \sum_{b=1}^B \hib.
\end{align}
This differs from the conventional incomplete U-statistic due to the new sampling scheme. Though $M = 2$ is enough for estimating $\vh$, we recommend using $M = \floor{n/k}$ for a smaller variance. This is guaranteed by the following proposition.

\begin{proposition}
\label{prop:var incU, match}
For an incomplete U-statistic with $M \cdot B$ samples obtained using the matched sample sampling scheme, 
\begin{align}
\label{eq:var Uinc, partition}
    \Var(\UincMatch) = \left( 1-\frac{1}{B} \right) Var(U_n) +  \frac{1}{MB} \vh.
\end{align}
\end{proposition}
The proof is collected in Appendix \ref{sec:append:varUmatch}. 
We should note that when fixing the total number of kernels, $M \cdot B$ and let $M \geq 2$, the variance of $\UincMatch$ is always smaller than the variance of $\Uinc$ given in \eqref{eq:var Uinc = varU + xik}. However, these two are identical when $M = 1$.

Based on this new sampling scheme, we can propose estimators $\hatvhB$ and $\hatvsB$ as analogs to $\hatvh$ and $\hatvs$, respectively.  
Denote the collection of kernels as $\{ \hib\}_{i, b}$, for $i = 1, 2, \ldots, M$, $b = 1, 2, \ldots, B$. A sample variance within each group $b$, $\frac{1}{M-1}\sum_{i = 1}^M [\hib - \bar h^{(b)}]^2$,  is an unbiased estimator of $\vh$. Here $\bar h^{(b)} = \frac{1}{M} \sum_{i=1}^M \hib$ is the group mean. Hence, the average over all groups becomes $\hatvhB$. 
\begin{align}
\label{eq:varh est inc}
    &\hatvhB
    = \frac{1}{B} \sum_{b=1}^{B} \frac{1}{M-1}\sum_{i = 1}^M \left(\hib - \bar h^{(b)} \right)^2.
\end{align}
Similarly, with some algebra, we can define $\hatvsB$ as
\begin{align}
\label{eq:hatvsB}
    &\hatvsB = \frac{1}{MB-1} \sum_{b=1}^B \sum_{i=1}^M 
    \left( \hib  - \UincMatch \right)^2.
\end{align}
Note that $\hatvhB$ is still an unbiased estimator of $\vh$ while $\hatvsB$ introduces a small bias when estimating $\vs$ because these subsamples are not randomly obtained --- there is an over-representation of non-overlapping pairs. The following proposition quantifies this bias. 

\begin{proposition}
\label{prop:bias of var subsamp}
For the sample variance estimator $\hatvsB$ defined on the matched sample groups subsamples with $M \cdot B \geq 2$, we denote $\delta_{B, M} := \frac{M-1}{MB- 1}$. Then, 
\begin{align}
    \E \left( \hatvsB \right) = (1 - \delta_{M, B}) \vs + \delta_{M, B} \vh.
\end{align}
\end{proposition}

This proposition leads to the following unbiased estimator of $\Var(\UincMatch)$. The proofs of both Propositions \ref{prop:bias of var subsamp} and \ref{prop:inc var est unbiase} are collected in Appendix \ref{sec:append:methoddology}. 

\begin{proposition}
\label{prop:inc var est unbiase}
Given $M \cdot B$ subsamples from the matched sample sampling scheme, with $B \geq 1$ and $M \geq 2$, the ``{\propEstFull}'' given below is an unbiased estimator of $\Var(\UincMatch)$:
\begin{align}
    \label{eq:est var Uinc match}
    \widehat{\Var}(\UincMatch) = \hatvhB -  \frac{MB-1}{MB} \hatvsB.
\end{align}
\end{proposition}

\subsection{Unifying Existing Unbiased Estimators}\label{sec:connection}

To conclude this section, we discuss the relationships and differences between our view of the variance decomposition versus existing approaches. As noted in the introduction, various variance estimators appeared in the literature to correct the bias when the leading term does not dominate. \citet{folsom1984probability} and \citet{schucany1989small} primarily focus on unbiased estimators for complete U-statistics with a very small sample. In particular, \citet{schucany1989small} propose two estimators of $\xi_{1, 2}^2$ in the Hoeffding decomposition (denoted as $\hat \zeta_{1}^2$ and $\tilde{\zeta}_1^2$ in their page 418 and 422, respectively). Interestingly, the estimators of $\xidk{1}$ introduced in \citet{lucas:old} and \citet{lucas:new} are efficient incomplete approximations of the former, $\hat \zeta_{1}^2 := \frac{1}{n-1} \sum_{i=1} [\hat h_1(X_i) - U_n]^2$, where $\hat h_1(X_i) = \binom{n-1}{k-1}^{-1} \sum_{S_j: X_i \in S_j} h(S_j)$. Meanwhile, our estimator $\hatvh - \hat{\tilde\xi}_{1, k}^2$ is equivalent to the latter, $\tilde{\zeta_1^2}$. A comprehensive derivation is provided in Appendix \ref{sec:append:equivalence,complete}.

\citet{wang2014variance} propose an unbiased estimator motivated by $E(U_n^2) - E^2(U_n)$. Their complete variance estimator \citep[page 1120]{wang2014variance} is,
\begin{align*}
    N_{k}^{-1} \sum_{P_k} h(S_i) h(S_j) - N_{0}^{-1} \sum_{P_0} h(S_0) h(S_0),
\end{align*}
where $P_c = \{(S_i,S_j) \,\, \text{s.t.} \,\, | S_i \cap S_j| \leq c \}$, and $N_c$ is the cardinality of $P_c$. Motivated by this formulation, they further propose an ANOVA form of the estimator and its corresponding incomplete version. 

Although various unbiased estimators exist in the literature, they are all motivated by entirely different perspectives. The unique motivation of our estimator is its peak-region dominance phenomenon and the corresponding conditional variance view, which allows unbiased estimation. However, it is interesting that the connections among existing estimators have never been investigated. To complete our analysis, we further established several connections. In Appendix \ref{sec:append:equivalence,complete} and \ref{sec:append:equivalence,incomplete} we show that all existing unbiased complete estimators are essentially the same estimator, presented in different formats and settings. In particular, we show that \citet{folsom1984probability}'s formula is identical to our complete version and also equivalent to \citet{wang2014variance}. We further restrict a setting with $k = 2$ for a direct comparison with \citet{schucany1989small}. In Appendix \ref{sec:append:equivalence,incomplete}, we show the equivalence between our incomplete estimators and \citet{wang2014variance}.

\section{Theoretical Results}
\label{sec:theory results}

To the best of our knowledge, ratio consistency of variance estimator in the context of infinite-order U-statistics has not been investigated. In this section, we attempt to fill some gaps in the literature by establishing the results of our proposed estimator, meaning that we want to show
\begin{align*}
\frac{\widehat{\Var}(U_n)}{\E (\widehat{\Var}(U_n))} \xrightarrow{P} 1.
\end{align*}
where $\xrightarrow{P}$ denotes convergence in probability. The notion of ratio consistency is important here since the variance of U-statistics would naturally converge to 0 as $n$ grows. Hence any variance estimator that converges to 0 is consistent. However, a consistent estimator does not guarantee normal coverage. In the following, we shall rewrite $\widehat{\Var}(U_n)$ as $\vu$, and show a sufficient condition of the above 
\begin{align*}
    \Var(\vu) /\E^2 (\vu) =  \Var(\vu) /\Var^2 (U_n) \to 0, \,\text{as}\, n \to \infty.
\end{align*}

We want to note that such a result under general $k$ settings is likely impossible without strong assumptions or knowledge of the specific form of t he kernel $h$. The main difficulty in the proof is caused by the fourth-order term in the form of $\Cov [h(S_1) h(S_2), \allowbreak h(S_3) h(S_4)]$ which naturally appears in the variance of $\vu$. Untangling the dependencies of the fourth-order term under large $k$ is a difficult task. Hence, we focus on the $k = o(\sqrt{n})$ setting in which the result is more attainable, although computationally, the estimator can still be applied whenever $k \leq n/2$. Even though the $k = o(\sqrt{n})$ setting is somewhat restrictive, it is still the first in the literature under the context of this paper. And further investigations may be established by extending the proposed strategy to higher orders. 

Our main strategy can be summarized as follows. First, we observe that the proposed estimator $\vu = \hatvh - \hatvs$ can be written an order-$2k$ U-statistic:
\begin{align}
\label{eq:def:vu U-stat size-2k}
    \hat{V}_{u}
    = \binom{n}{2k}^{-1} \sum_{\Skk \subseteq \mathcal{X}_{n}} \psi\left(\Skk \right),
\end{align}
where $\Skk$ is a size-$2k$ subsample set and  $\psi\left(\Skk \right)$ is the corresponding size-$2k$ kernel, defined as 
\begin{align*}
\psi\left(\Skk \right) :=  \psi_{k}\left(\Skk \right)  -  \psi_{0}\left(\Skk \right).
\end{align*}
Here $\psi_{k'}\left(\Skk \right)$ for $k' = 0, 1, 2, \ldots, k$ satisfies
\begin{align}
    \label{eq:psi: a size-2k kernel for Q(k')}
    \psi_{k'}\left(\Skk \right)
    = N_{n, k, k'}
    \sum_{d=0}^{k'} 
    \frac{1}{N_d} 
    \sum_{\substack{S_1, S_2 \subset \Skk \\ | S_1 \cap S_2 | = d}}
    h\left(S_1\right) h\left(S_2\right),
\end{align}
where $N_{n, k, k'} = \binom{n}{2k} \binom{n}{k}^{-1} \binom{n-k +k'}{k}^{-1}$ and $N_d = \binom{n-2k+d}{d}$ is the number of different size-$2k$ sets such that its two size-$k$ subsets $S_1$ and $S_2$ share $d$ overlaps. We remark that in this paper, $S$ refers to a size-$k$ set, and $\Skk$ refers to a size-$2k$ set. 

Similar to a regular U-statistic, the variance of an order-$2k$ U-statistic $\vu$ can be decomposed as
\begin{align}
\label{eq:Var(Vu) expression}
    \Var\left(\vu\right)
    = \binom{n}{2k}^{-1} \sum_{c=1}^{2k} \binom{2k}{c} \binom{n-2k}{2k-c} \sigma_{c, 2k}^2,
\end{align}
where $\sigma_{c, 2k}^2$ is the covariance between $\psi (\Skk_1)$ and $\psi (\Skk_2)$ for $| \Skk_1 \cap \Skk_2 | = c$ and $c = 1, 2, \ldots, 2k$:
\begin{align}
\label{eq:def:sigma_c^2}
    \sigma_{c, 2k}^2 
    := \Cov \left(\psi (\Skk_1), \psi (\Skk_2) \right).
\end{align}

If we follow the existing literature, it is common to impose high-level assumptions on the kernel $\psi$ and also bound the ratio of the last term, $\sigma_{2k, 2k}^2$ over the first term $\sigma_{1, 2k}^2$ \citep{Romano2020}. However, not only such assumptions are difficult to verify and can be possibly violated (see discussion in Appendix \ref{sec:append:assump:discuss Fc, counter example for linear rate}), but also $\psi$ is viewed as some form of a ``black box'', which does not help in analyzing the convergence of $\Var(\vu) / \Var^2(U_n)$. Therefore, we have to establish some connection between $\psi$ kernel $h$. 

Hence, the key strategy of our approach is to avoid direct assumption on $\vu$'s kernel $\psi$ and instead only impose assumptions on a fourth-order term of $h$: $\Cov [h(S_1) h(S_2), \allowbreak h(S_3) h(S_4)]$. This leads to the main technical challenge in this work: since $\Var(\vu)$ is the variance of variance estimator of U-statistics, it becomes inevitable to study the fourth-order term of $h$ instead of a second-order term. To be specific, $\xi_{d, k}^2 = \Cov(h(S_1), h(S_2))$ in the classical U-statistics only involves the overlaps between $S_1, S_2$ while $\Cov [h(S_1) h(S_2), \allowbreak h(S_3) h(S_4)]$ involves the 4-way overlaps among $S_1, S_2, \allowbreak S_3, S_4$, although it shares similar intuition as $\xi_{d, k}^2$. 

We first establish the Double U-statistics notion of $\vu$ in Section \ref{sec:double U-stat}. The double U-statistic structure in Proposition \ref{prop: double U-statistic: alternative form of phi} shows a cancellation effect (see Appendix \ref{sec:technical lemmas}) inside of $\vu$, which helps accelerate the convergence rate of $\Var(\vu)$. Using this structure, we can further decompose each $\sigma_{c, 2k}^2$ in the Hoeffding decomposition \eqref{eq:Var(Vu) expression} into $\etacd$ terms (see Proposition \ref{prop:sigma_c^2 written as eta_c^2}).
Then, we bound all $\etacd$'s by decomposing each term into a basic covariance term $\Cov [h(S_1) h(S_2), \allowbreak h(S_3) h(S_4)]$. Hence, it suffices to impose primitive assumptions on $\Cov [h(S_1) h(S_2), \allowbreak h(S_3) h(S_4)]$ to analyze the behavior of $\vu$. We should highlight the challenge that we need to use 11 parameters to describe the 4-way overlapping among $S_1, S_2, S_3, S_4$. Details are left in the discussion in the assumption section (Section \ref{sec:assumptions and notations}). We also remark that it is easier to understand the difficulties and strategies related to the nature of \emph{Double U-statistic} structure through a simplified example, the linear average kernel, presented in Appendix \ref{sec:append:low order kernel}. And finally, in Section \ref{sec:main result}, we present the ratio consistency. Section \ref{sec:append:roadmap} is used to summarize a roadmap of the proof.

\subsection{Double U-statistic Structure}
\label{sec:double U-stat}

We define a notion of \emph{Double U-statistic} to facilitate our discussion and show that $\vu$ is a \emph{Double U-statistic}. The advantage of this tool is to break down our variance estimator into lower-order terms, which alleviates the difficulty involved in analyzing $\sigma_{c, 2k}^2$. 

\begin{definition}[Double U-statistic]
\label{def:double-Ustat}
For an order-$k$ U-statistic, we call it \emph{Double U-statistic} if its kernel function $h$ is a weighted average of U-statistics.
\end{definition}

Essentially, a \emph{Double U-statistic} is a ``U-statistic of U-statistic''. 
By \eqref{eq:def:vu U-stat size-2k}, 
$
\hat{V}_{u}
= \binom{n}{2k}^{-1} \sum_{\Skk \subseteq \mathcal{X}_{n}} \psi\left(\Skk \right)
$.
$\vu$ involves a size-2k kernel  $\psi$. However, by Equation \eqref{eq:psi: a size-2k kernel for Q(k')}, the kernel $\psi$ has a complicated form.
The following proposition shows that we can further decompose $\psi$ into linear combinations of $\varphi_d$'s, which are still U-statistics. 

\begin{proposition}[$\vu$ is a Double U-statistic]

\label{prop: double U-statistic: alternative form of phi}
The order-$2k$ U-statistic $\vu$ defined in Equation \eqref{eq:def:vu U-stat size-2k} is a \emph{Double U-statistic}. Its kernel $\psi\left(\Skk \right)$ can be represented as a weighted average of U-statistics, such that
\begin{align}
     \psi\left(\Skk \right) 
% \label{eq:alternative (primary) form of phi, 0 to k}
%      &:= \sum_{d = 0}^{k} w_d \varphi_d \left(\Skk \right) \\
\label{eq:alternative form of phi}
     := \sum_{d = 1}^{k} w_d \left[\varphi_d \left(\Skk \right) - \varphi_0 \left(\Skk \right)\right].
\end{align}
Here, for $ d = 0, 1, ..., k$, $\varphi_d$ is the U-statistic with size-$(2k-d)$ asymmetric kernel as following
\begin{align}
\label{eq:def:phi}
    \varphi_d \left(\Skk \right) = M_{d, k}^{-1} \sum_{\substack{S_1, S_2 \subset \Skk \\ | S_1 \cap S_2 | = d}} h(S_1) h(S_2);
\end{align}
$M_{d, k} := \binom{2k}{d} \binom{2k-d}{d} \binom{2k-2d}{k-d}$,  which is the number of pairs $S_1, S_2 \subset \Skk$, s.t. $| S_1 \cap S_2 | = d$;
and 
% \begin{align*}
%     w_d &=   \binom{n}{2k} \binom{n}{k}^{-2} 
%         \binom{2k}{d} \binom{2k-d}{d} \binom{2k-2d}{k-d} 
%         / \binom{n-2k+d}{d}, \quad \forall d \geq 1;\\
%     w_0 &=     \left[ \binom{n}{k}^{-1} - \binom{n-k}{k}^{-1} \right] \binom{n}{2k}  \binom{n}{k}^{-1} 
%          \binom{2k}{k} .
% \end{align*}
$w_d :=   \binom{n}{2k} \binom{n}{k}^{-2} 
        \binom{2k}{d} \binom{2k-d}{d} \binom{2k-2d}{k-d} 
        / \binom{n-2k+d}{d}, \,\, \forall d \geq 1$, 
    $w_0 =    \left( \binom{n}{k}^{-1} - \binom{n-k}{k}^{-1} \right) \binom{n}{k}^{-1}  \binom{n}{2k}  
         \binom{2k}{k}.$
The $w_d$'s defined above satisfy the following.
$\sum_{d=0}^{k} w_d= 0.   w_d > 0, \, \forall d > 0$.
\begin{align}
\label{eq:weights for phi's: 2}
     w_d  = \calO \left (\frac{k^{2d}}{d!\, n^d}  \right), \,
     \text{for } d = 1, 2, ..., k.
\end{align}
Particularly, for fixed $d$, 
\begin{align}
     w_d  = 
      \left({1 + o(1)} \right) \frac{k^{2d}}{d! \, n^d}.
\end{align}
\end{proposition}
The proof is collected in Appendix \ref{sec:proof for double U}. 
We observe that given $k = o(\sqrt{n})$, $w_d$ decays with $d$ at a speed even faster than the geometric series. In our later analysis, we can show that the first term, $w_1 [\varphi_1 (\Skk ) - \varphi_0 \Skk)]$, can be a dominating term in $\psi(\Skk)$.
Moreover, with kernel $\varphi_d(\Skk)$, we introduce the following decomposition of $\sigma_{c, 2k}^2$.

\begin{proposition}[Decomposition of $\sigma_{c, 2k}^2$]
\label{prop:sigma_c^2 written as eta_c^2}
For any size-$2k$ subsample sets $\Skk_1, \Skk_2$, s.t. $|\Skk_1 \cap  \Skk_2 | = c$ and $1 \leq c \leq 2k, 1 \leq d_1, d_2 \leq k$, we define
\begin{align}
\label{eq:def:eta}
    \etacd := \Cov \big[ \varphi_{d_1}\left(\Skk_1 \right) - \varphi_0\left(\Skk_1 \right), \\
    \varphi_{d_2}\left(\Skk_2 \right) - \varphi_0\left(\Skk_2 \right) \big]. \nonumber
\end{align}
Then, we can represent $\sigma_{c, 2k}^2$ as a weighted sum of $\etacd$'s.
\begin{align}
\label{eq:sigma_c^2 decomposition}
    \sigma_{c, 2k}^{2} 
    = \sum_{d_1 = 1}^{k} \sum_{d_2 = 1}^{k} w_{d_1} w_{d_2} \etacd.
\end{align}
\end{proposition}

This proposition can be directly concluded by combining the alternative form of $U_n$'s kernel $\psi$ in Equation \eqref{eq:alternative form of phi} and the definition of $\etacd$. 
With the help of the \emph{Double U-statistic} structure, upper bounding $\sigma_{c, 2k}^2$ can be boiled down to analyzing $\etacd$. Detailed analysis of this connection is provided in Section \ref{sec:append:roadmap} and Appendix \ref{sec:technical lemmas}. Note that we can further decompose $\etacd$ (see Appendix \ref{sec:proof:lemma:eta_c}), so $\sigma_{c, 2k}^2$ can be viewed as a weighted sum of $\Cov[h(S_2) h(S_2), h(S_3) h(S_4)]$'s. 

\subsection{Assumptions}
\label{sec:assumptions and notations}

Assumption \ref{assumption: k = n^(1/2 - epsilon)} limits the kernel size $k$ as a lower-order of $\sqrt{n}$, while Assumption \ref{assumption:2nd moment} controls the growth rate of $\xi_{d, k}^2$ with $d$. 
Assumption \ref{assumption:low dependency}, \ref{assumption:non-negative and ordinal}, and \ref{assumption:4th moment} are related to $\Cov [h(S_1) h(S_2), \allowbreak h(S_3) h(S_4)]$. 
As previously mentioned, $\Cov [h(S_1) h(S_2), \allowbreak h(S_3) h(S_4)]$ can be viewed as an extension of $\xi_{d, k}^2 = \Cov[h(S_1), h(S_2)]$, the classical covariance of two kernels. While $\xi_{d, k}^2$ only depends on one parameter, i.e., $d = |S_1 \cap S_2|$, 11 parameters are needed to fully determine $\Cov [h(S_1) h(S_2), \allowbreak h(S_3) h(S_4)]$, since it involves a 4-way overlapping structure. This can be visualized in Figure \ref{fig:set_relationship} in Appendix. We denote the number of parameters as ``Degree of Freedom (\emph{DoF})'' of the covariance. Essentially, Assumptions \ref{assumption:low dependency}, \ref{assumption:non-negative and ordinal}, and \ref{assumption:4th moment} are about reducing this \emph{DoF} and controlling the growth of $\Cov [h(S_1) h(S_2), \allowbreak h(S_3) h(S_4)]$ with overlapping samples. 

In Appendix \ref{sec:append:assump}, we provide further discussion and examples of our assumptions. In Appendix \ref{sec:proof under relaxed dependency}, we propose a relaxation of Assumption \ref{assumption:low dependency} and present the proof of the main results under the new assumptions. 

\begin{assump}
\label{assumption: k = n^(1/2 - epsilon)}
     There exist a constant  $\epsilon \in (0, 1/2)$, so that the growth rate of kernel size $k$ regarding sample size $n$ is bounded as $k = \calO(n ^{1/2 - \epsilon})$. 
\end{assump}

\begin{assump}
\label{assumption:2nd moment}
$\forall k \in \N^+, \xi_{1, k}^2 > 0$ and $\xi_{k, k}^2 < \infty$.
There exist a universal constant $a_1 \geq 1$ independent of $k$, satisfying that
\begin{align*}
   \sup_{d = 2, 3, ..., k} \frac{\xi_{d, k}^2}{d^{a_1} \xi_{1, k}^2} = \calO(1).
\end{align*}
\end{assump}

Note that a smaller $a_1$ in Assumption \ref{assumption:2nd moment} implies a stronger assumption. It is well known that $k \xi_{d, k}^2 \leq d \xi_{k, k}^2$ \citep{ustatBook}, the smallest possible value of $a_1$ is 1, which is used in the existing literature \citep{lucas:old, Romano2020, lucas:new, peng2019asymptotic}. Hence, if we force $a_1 = 1$ and only focus on the upper bound of $\Var(U_n)$, the growth rate of $k$ in Assumption \ref{assumption: k = n^(1/2 - epsilon)} can be relaxed to $o(n)$. 
However, this trade-off between Assumptions \ref{assumption: k = n^(1/2 - epsilon)} and \ref{assumption:2nd moment} cannot be applied to ratio consistency directly. 

To motivate our other assumptions, we provide a brief discussion on the 4-way overlap of $\Cov [h(S_1) h(S_2), \allowbreak h(S_3) h(S_4)]$. 
As we mentioned before, the goal is to avoid direct assumptions of $\psi(\Skk)$ and its covariance $\sigma_{c, 2k}^2$ and study the fourth-moment term $\Cov [h(S_1) h(S_2), \allowbreak h(S_3) h(S_4)]$.
To simplify the notation, we let 
\begin{align}
\label{eq:def:rho = Cov(h1h2, h3h4)}
\covh:=
    \Cov \left[h(S_1) h(S_2), h(S_3) h(S_4)\right].
\end{align}

Then $\covh$ involves 11 different overlap schemes, 
\begin{align*}
    \text{2-set: }& |S_1 \cap S_2|,\, |S_1 \cap S_3|,\, |S_1 \cap S_4|,\\
    &\, |S_2 \cap S_3|,\, |S_2 \cap S_4|,\, |S_3 \cap S_4|; \\
    \text{3-set: }& |S_1 \cap S_2 \cap S_3|,\, |S_1 \cap S_2 \cap S_4|,\\ 
    &|S_1 \cap S_3 \cap S_4|,\, |S_2 \cap S_3 \cap S_4|;\\
    \text{4-set: }& |S_1 \cap S_2 \cap S_3 \cap S_4|. 
\end{align*}

Hence, 11 parameters are needed to describe $\covh$. We denote the number of these parameters as the ``Degrees of Freedom'' (\emph{DoF}) of $\covh$. Furthermore, there are two types of these parameters: $d_1 = |S_1 \cap S_2|$ and $d_2 = |S_3 \cap S_4|$ describes the overlapping within $\Skk_1$ and $\Skk_2$ respectively; while other 9 overlapping sets are subsets of $\Skk_1 \cap \Skk_2$, so they describe the overlapping between $\Skk_1$ and $\Skk_2$. 
We can describe these 9 overlapping sets by a 9-dimensional vector $\unr$, whose definition is collected in Appendix \ref{sec:append:r}. Hence, the 11 \emph{DoF} can be denoted by tuple $(\unr, d_1, d_2)$.

However, it may not be necessary to know all $\unr, d_1, d_2$ values to calculate this covariance $\covh$. For example, in the linear average kernel (Example \ref{example:linear kernel for low depent. assump} in Appendix \ref{sec:assump:relax low dependency assumption}), $\covh$ only depends on $\unr$. This may be expected for an estimator that is approximately linear. Hence, we propose the following assumption.

\begin{assump}
\label{assumption:low dependency}
$\covh$ only depends on the 9 \emph{DoF} vector $\unr$. Hence, without the risk of ambiguity, we define a function $\rhor(\cdot)$ with
\begin{align}
\label{eq:def F(r)}
    \rhor = \covh.
\end{align}
\end{assump} 

The assumption implies that the within $\Skk_1$ or $\Skk_2$ overlapping counts have no impact on $\covh$. This simplifies a cancellation pattern when analyzing $\eta_{c, 2k}(d_1, d_2)$ \eqref{eq:def:eta}. 
A comprehensive discussion of this assumption can be found in Appendix  \ref{sec:append:assump}. We first demonstrate that this assumption is valid for the linear average kernel, as previously mentioned. Next, we provide an example to illustrate the challenges of reducing \emph{DoF} below 9 by only considering two-way overlaps, indicating that further simplification of this assumption may require specific assumptions about the kernel functions. In addition, in Appendix \ref{sec:proof under relaxed dependency}, we suggest a relaxation of this assumption and provide an alternative proof of the main results based on this relaxed assumption. 

\begin{assump}[Ordinal Covariance]
\label{assumption:non-negative and ordinal}

For all size-$k$ subsets $S_1, S_2, S_3, S_4$ and $S_1', S_2', S_3', S_4'$, let $\covh$ and $\covh'$ denote the corresponding covariance as defined in Equation \ref{eq:def:rho = Cov(h1h2, h3h4)} with DOFs $\unr$ and $\unr'$ (defined in Appendix \ref{sec:append:r}), respectively. Then, we have:
\begin{align}
&\covh \geq \covh',
\quad \text{if $r_{ij} \geq r_{ij}', \forall i, j = 0, 1, 2$.}
\nonumber
\end{align}

Moreover, given size-$k$ sets $S'$, $S''$, and $\unr$ such that $|S' \cap S''| = c$ and $|\unr| = c$, we have:

\begin{align}
\covh \leq \Cov[h(S')^2, h(S'')^2] =: F_c^{(k)}.
\label{eq:F_c as upper bound}
\end{align}

\end{assump}

Assumption \ref{assumption:non-negative and ordinal} implies that more overlapping leads to larger $\covh$. This is a reasonable result to expect. 
For every $c = |\Skk_1 \cap \Skk_2|$, it also provides an upper bound of $\covh$, where $F_c^{(k)}$ refers to $\rho$ with ``maximum possible overlaps'' given $c$ such that $S_1 = S_2, S_3 = S_4$. 
The overlapping associated with $\fk$ is visualized in Figure \ref{fig:example of ordinal cov} in Appendix \ref{sec:append:assump}.
It's easy to see that $\covh \geq 0$, an analog to $\xi_{d, k}^2 \geq 0$ in a regular U-statistics setting \citep{ustatBook}.

\begin{assump}
\label{assumption:4th moment}
For $\fk_c$ defined in Assumption \ref{assumption:non-negative and ordinal}, when $c = |S_1 \cap S_2| = 1$, we have
\begin{align}
\label{eq: F_1 upper bound}
     \frac{F_1^{(k)}}{\xi_{1,k}^4} = \frac{\Cov[h(S_1)^2, h(S_2)^2]}{\left[\Cov[h(S_1), h(S_2)] \right]^2} = \calO(1)
\end{align}

In addition, there exist a universal constant $a_2 \geq 1$ independent of $k$, satisfying
\begin{align}
\label{eq:fourth order growth}
    \sup_{c = 2, 3, ..., 2k}
    \frac{F_c^{(k)}}{c^{a_2} \,F_1^{(k)}} = \calO(1).
\end{align}
\end{assump}
Equation \eqref{eq: F_1 upper bound} states that a fourth-moment term cannot exceed a second-moment term $\xi_{1, k}^2$. This can be verified for the linear average kernel with basic moment conditions. Similarly to the polynomial growth rate of $\xi_{d, k}^2$ specified in Assumption \ref{assumption:2nd moment}, Equation \eqref{eq:fourth order growth} controls a polynomial growth rate of $\covh$ with respect to $c$, as $\fk$ is an upper bound of $\covh$. It is worth noting that Assumption \ref{assumption:4th moment} can be implied by Assumption \ref{assumption:2nd moment} for certain specific kernels (see Example \ref{example:fourth moment under no cumulant} in Appendix \ref{sec:append:assump}).

\subsection{Main Results}
\label{sec:main result}

We now present our main results. As a direct consequence of the following theorem, the ratio consistency property is provided in Corollary \ref{thm:ratio consistency}.

\begin{theorem}[Asymptotic variance of $U_n$ and $\vu$]
\label{thm:var of u_n and v_n}
Under \allassump, we can bound  $\Var(U_n)$  \eqref{eq:Decomposition of Var(U_n) Hoeffding} and $\Var(\vu)$  \eqref{eq:Var(Vu) expression} as
\begin{align}
\label{eq:Var(Un) Theta bound}
    \Var(U_n) &= \left(1 + o(1)\right) \frac{k^2 \xi_{1, k}^2}{n}, \\
\label{eq:Var(Vu) upper  bound}
    \Var(\vu) &= \calO \left(\frac{k^2 \check \sigma_{1, 2k}^2}{n}\right),
\end{align}
where $\check \sigma_{1, 2k}^2 \asymp \frac{k^2 \xi_{1,k}^4}{n^2} $ is the upper bound of $\sigma_{1, 2k}^2$ given by Proposition \ref{lemma:bound each sigma_c^2 for finite c} in Appendix \ref{sec:technical lemmas}. Here, ``$f \asymp g$'' implies $f = \calO(g)$ and $g = \calO(f)$.
\end{theorem}

The proof of the results is provided in Appendix \ref{sec:proof of theorem Var Un Var Vu}. The calculation of $\Var(U_n)$ in \eqref{eq:Var(Un) Theta bound} and $\Var(\vu)$ in \eqref{eq:Var(Vu) upper bound} requires controlling the growth of $\xi_{d, k}^2$ and $\sigma_{c, 2k}^2$.
In particular, \eqref{eq:Var(Un) Theta bound} can be derived from a general proposition (Proposition \ref{thm:leading variance of U-stat}) provided below. However, the proof of \eqref{eq:Var(Vu) upper bound} is more complex, as it relies on the \emph{double U-statistic} structure of $\vu$.
A proof roadmap is presented in Section \ref{sec:append:roadmap}, and technical lemmas to upper bound $\etacd$ \eqref{eq:def:eta} and $\sigma_{c, 2k}^2$ are provided in Appendix \ref{sec:technical lemmas}.

\begin{proposition}
[Leading covariance domination]
\label{thm:leading variance of U-stat}
For a complete U-statistic $U_n$ with size-$k$ kernel and $k = o(\sqrt{n})$, assume that $\xi_{1, k}^2 > 0$ and there exists a non-negative constant $C$ such that 
$$\limsup_{k\to \infty, 2 \leq d\leq k}\xi_{d, k}^2 / \left(d! {\xi_{1, k}^2} \right) = C$$.
Then,
$$\lim_{n \to \infty} \Var (U_n) / \left(k^2 \xi_{1, k}^2 / n \right) = 1.$$
\end{proposition}

The proof of this proposition can be found in Appendix \ref{sec:append:proof:thm:leading variance of U-stat}. This proposition relaxes the conditions from Theorem 3.1 in \citet{Romano2020} and provides a foundation for our approach to bounding $\Var(\vu)$.
Specifically, our condition allows for the ratio $\xi_{d, k}^2 / \xi_{1, k}^2$ to grow at a factorial rate of $d$, whereas the conditions in \citep{lucas:old, lucas:new, Romano2020} only allow for linear growth. A comparison between our assumption on $\xi_{d, k}^2 / \xi_{1, k}^2$ and existing literature is provided in Section \ref{sec:assumptions and notations}.

\begin{corollary}[Ratio consistency of $\vu$]
\label{thm:ratio consistency}
Under \allassump,
\begin{align*}
\frac{\Var (\vu )}{ \left (\E (\vu ) \right)^2}
%= \frac{\Var (\vu )}{ \left (\Var (U_n) \right)^2} 
= \calO \left(\frac{1}{n} \right),
\end{align*}
which implies that
$
     {\vu} / { \E (\vu )} \xrightarrow{P} 1.
$
\end{corollary}

This result is a corollary of Theorem \ref{thm:var of u_n and v_n} and demonstrates the consistency of the variance estimator $\vu$ in terms of ratios. The proof can be found in Appendix \ref{sec:proof of main theorem, ratio consistency}.
To the best of our knowledge, this is the first proof of the ratio consistency of an unbiased variance estimator for growing order U-statistics.

\subsection{Proof Roadmap}
\label{sec:append:roadmap}

The roadmap to upper bound $\Var(\vu)$ \eqref{eq:Var(Vu) upper bound} in Theorem \ref{thm:var of u_n and v_n} is provided in Equation \eqref{eq:road map}. The relevant technical lemmas are summarized in Appendix \ref{sec:technical lemmas}.
\begin{align}
\label{eq:road map}
\Var(\vu) 
     &=  \sum_{c = 1}^{2k} v_c \sigma_{c, 2k}^2 
      \leq  \sum_{c = 1}^{2k} v_c \checksigc \\
    &\overset{(*)}{\asymp} \sum_{c = 1}^{T_1} v_c 
     \checksigc
    \overset{(\dagger)}{\asymp}  v_1 \checksigone 
    \overset{(\ddagger)}{\asymp}  \frac{k^4}{n^3} \fk_1. \nonumber
\end{align}

The quantity $v_c:= \binom{n}{2k}^{-1} \binom{2k}{c} \binom{n-2k}{2k-c}$ represents the coefficients in the Hoeffding decomposition of $\Var(\vu)$ \eqref{eq:Var(Vu) expression}; $\checksigc$ is the upper bound of $\sigma_{c, 2k}^2$ given by Propositions \ref{lemma:bound each sigma_c^2 for finite c} and \ref{lemma:bound sigma_c^2 for any c}; and "f $\asymp$ g" means that $f = \calO(g)$ and $g = \calO(f)$. The inequalities in \eqref{eq:road map} should be interpreted as follows.
\begin{itemize}
\item The first inequality $\leq$ is a result of replacing $\sigma_{c, 2k}^2$ with either its tighter bound for $c = 1, 2, ..., T_1$ (Proposition \ref{lemma:bound each sigma_c^2 for finite c}) or its looser bound for $c = T_1 + 1, 2, ..., 2k$ (Proposition \ref{lemma:bound sigma_c^2 for any c}).
Each $\sigma_{c, 2k}^2$ can be decomposed into $\etacd$'s \eqref{eq:sigma_c^2 decomposition}. Propositions \ref{lemma:bound each sigma_c^2 for finite c} and \ref{lemma:bound sigma_c^2 for any c} are based on the tighter and looser upper bounds of $\etacd$ (Lemma \ref{lemma:eta_c (d1, d2)} and \ref{lemma:bound every eta_c^2}).

\item The first asymptotic notation $\asymp$ (denoted with $*$) is concluded from Lemma \ref{lemma:Finite truncated variance I}. The value of $T_1 = \floor{\frac{1}{\epsilon}}+1$ only depends on the growth rate of $k$, not $n$, as we assume $k = o(n^{1/2 -\epsilon})$ in Assumption \ref{assumption: k = n^(1/2 - epsilon)}.

\item The second asymptotic notation $\asymp$ (denoted with $\dagger$) is a result of comparing the finite $\checksigc$ terms for $c = 1, 2, ..., T_1$.

\item The last asymptotic notation $\asymp$ (denoted with $\ddagger$) is concluded from Lemma \ref{lemma:bound each sigma_c^2 for finite c}.
\end{itemize}

\section{Application to Random Forests}
\label{sec:apply to rf}

Random forests can be viewed as an incomplete infinite-order U-statistic with a random kernel \citep{lucas:old}. The purpose of this section is to present a comprehensive algorithm, as well as two extensions: one for the case when $k > n/2$ and another one that uses local smoothing to address the issue of negative estimation values. 

Notation-wise, we present the algorithm in the context of regression, where we observe a vector of covariates $\bm x_i \in \R^p$ and $y_i \in \R$ for observations $i$. Hence, define $X_i = (\bm x_i, y_i)$, and the kernel function $h(S_i)$ can be viewed as the tree prediction on a given target point $x^\ast$ with subsample $S_i$. The implementation of the variance estimator is straightforward using this setting and is summarized in Algorithm \ref{algo:var est}. We want to make a few comments. First, the original random forest \citep{breiman2001random} uses bootstrap samples, i.e., sampling with replacement, to build each tree. However, sampling without replacement \citet{geurts2006extremely} is also prevalent and achieves similar performances. Secondly, most random forest models utilize a random kernel instead of fixed ones. This is mainly due to the random feature selection \citep{breiman2001random} and random splitting point \citep{geurts2006extremely} when fitting each tree. \citet{lucas:old} show that U-statistics with random kernel converge in probability to its fixed kernel counterpart by viewing the fixed kernel version as the expectation of the random version. Under suitable conditions, given $B$ large enough, the theoretical analysis of random U-statistic can be reasonably reduced to analyzing the non-random counterpart, allowing our method to be applied. It is possible that both our estimators of $\vh$ and $\vs$ are inflated by the influence of the randomness due to their U statistic representation. However, such inflations are likely canceled out by the difference, and our simulation results in Section \ref{sec:simulation} confirm this speculation by showing that the estimator is mostly unbiased. 

\begin{algorithm}[htbp] 
\label{algo:var est}
    \caption{{\propEstFull} ($k \leq n/2$)}
    \KwIn{n, k, M,  B, training set $\calX_n$, and testing sample $x^*$}
    \KwOut{$\widehat{\Var}(\UincMatch)$}
    {\bf Construct matched samples:}\\
    \For{$b = 1, 2, \ldots, B$}{
        Sequentially sample $\{S_1^{(b)},  S_2^{(b)}, ..., S_M^{(b)}\}$ from $\calX_n$ such that $S_i^{(b)}$s are mutually exclusive, i.e., $S_i^{(b)} \cap S_{i'}^{(b)} = \emptyset$ for $i \neq i'$. \\
        }
    {\bf Fit trees and obtain predictions:}\\
    \quad Fit random trees for each subsample $S_i^{(b)}$ and obtain prediction $h(S_i^{(b)})$ on the target point $x^*$. \\
    {\bf Calculate the variance estimator components:}\\
    \quad Forest average: $\UincMatch = \frac{1}{MB} \sum_{i=1}^M \sum_{b=1}^B \hib$ \\
    \quad Within-group average: $\bar h^{(b)} = \frac{1}{M} \sum_{i=1}^M \hib$ \\
    \quad Tree variance \eqref{eq:varh est inc}: $\hatvhB
    = \frac{1}{B} \sum_{b=1}^{B} \frac{1}{M-1}\sum_{i = 1}^M (\hib - \bar h^{(b)} )^2$ \\
    \quad Tree sample variance \eqref{eq:hatvsB}: $\hatvsB = \frac{1}{MB-1} \sum_{i=1}^M \sum_{b=1}^B 
    ( \hib  - \UincMatch )^2$ \\
    {\bf The final variance estimator \eqref{eq:est var Uinc match}}\\
    \quad $\widehat{\Var}(\UincMatch) = \hatvhB - (1 - \frac{1}{MB}) \hatvsB$ 
\end{algorithm}

\subsection{Extension to \texorpdfstring{$k > n/2$}{k>n/2}}
\label{sec:subsec:extension k>n/2}

The previous estimator $\widehat{\Var}(\UincMatch)$ \eqref{eq:est var Uinc match} is restricted to $k \leq n/2$ due to the sampling scheme. However, this does not prevent the application of formulation \eqref{eq:varU = varh - sum}, $\Var(U_n) = \vh - \vs$. To the best of our knowledge, the existing literature does not provide further discussion under $k > n/2$ for general kernels, while some theoretical strategies such as \citet{wang2017pseudo} simplify the kernel into a low-order approximation. Alternatively, the infinitesimal jackknife \citet{wager2014confidence} has been shown to be almost equivalent to the leading term estimator in V statistics \citet{lucas:new}. Here, we discuss a generalization of our formulation for $k > n/2$. Re-applying Propositions \eqref{eq:varU = varh - sum} and \eqref{prop:var incU, match} with $M = 1$, we can obtain the variance of an incomplete $U$ statistic sampled randomly with replacement:
\begin{align*}
    \Var(U_{n, B, M=1}) =  \vh - \frac{B-1}{B} \vs.
\end{align*}
By Proposition \ref{prop:bias of var subsamp}, $\hat{V}_{B, M = 1}^{(S)}$ is still an unbiased estimator of $\vs$. However, $\vh$ has to be estimated with a different approach, since any pair of subsamples would share at least some overlapping samples. A simple strategy is to use bootstrapping. Hence, we generate another set of size-k samples, sampled with replacement, and evaluate the kernel, using their sample variance as an estimator of $\vh$. We remark that the bootstrap procedure introduced will introduce an additional computational burden.

\subsection{Locally Smoothed Variance Estimator for Random Forest}
\label{sec:subsec:smooth}

% Using the proposed variance estimator, we could construct confidence intervals for $\UincMatch$ accordingly, provided that the corresponding random forest estimator is asymptotically normal. 
% The asymptotic normality of random forests has been partially studied recently, e.g., \citet{athey2019generalized}, and is still an open question. Our focus is not on the properties of random forests themselves. 
% Instead, we are only interested in the behavior of various variance estimators, which can further lead to constructing confidence intervals. 
Even though the proposed estimator is unbiased, large variance of this estimator may still result in possible under-coverage of the corresponding confidence interval (CI).
Note that due to its variation, our variance estimator might be negative, though this rarely happens in our simulations. A similar phenomenon is also noticed by \citet{schucany1989small}, and \citet{wang2014variance}. To alleviate this issue, we propose a local smoothing estimator, namely {\propSmoothFull} ({\propSmooth}). The improvement is especially effective when the number of trees is small. This will be demonstrated in the simulation study, see, e.g., Table \ref{tab:cover over k and B} and Figure \ref{fig:3 panel results}. 

Denote a variance estimator on a future test sample $\bm x^*$ as $\rfVarEst{}$. We randomly generate $N$ neighbor points $\bm x^*_{1}, \ldots, \bm x^*_{N}$ and obtain their variance estimators $\rfVarEst{1}, \ldots, \rfVarEst{N}$. Then, the locally smoothed estimator is defined as the average:
\begin{align}
\label{eq:local avg}
\rfVarEstS = 
    \frac{1}{N+1} \Big[ \rfVarEst{} + \sum_{i=1}^N \rfVarEst{i} \Big].
\end{align}
The algorithm is presented as follows. 
\begin{algorithm}[htbp] 
\label{algo:var est smooth}
    \caption{{\propSmoothFull} ($k \leq n/2$)}
    \KwIn{n, k, M,  B, training set $\calX_{train}$,  testing sample $\bm x^*$ and number of neighbors $N$}
    \KwOut{Smooth Variance estimator $\rfVarEstS$}
    Find the closed distance $D_{min} = \min_{\bm x \in \calX_{train}} d(\bm x^*,  \bm x)$ \;
    Randomly generate $N$ neighbors $\bm x_{1}^*, ..., \bm x_{N}^*$ that satisfy $\bm x : d(\bm x, \bm x^*) \leq D_{min}$ or $\bm x : d(\bm x, \bm x^*) = D_{min}$ \; 
    Obtain variance estimators  $\rfVarEst{},  \rfVarEst{1}, ..., \rfVarEst{N}$ by Algorithm \ref{algo:var est} \;
    $\rfVarEstS = 
    \frac{1}{N+1} [ \rfVarEst{} + \sum_{i=1}^N \rfVarEst{i} ] \eqref{eq:local avg}$.
\end{algorithm}

In Algorithm \ref{algo:var est smooth}, ${d(\cdot,\cdot)}$ can be Euclidean distance for continuous covariates and other metrics for categorical covariates. 
In practice, we can pre-process data before fitting random forest models, such as performing standardization and feature selection.
Due to the averaging with local target samples, there is naturally a bias-variance trade-off in choosing $D_{min}$ and neighbors. This is a rather classical topic, and there can be various ways to improve such an estimator based on the literature. Our goal here is to provide a simple illustration. In the simulation section, we consider generating 10 neighbors on an $\ell_2$ ball centered at $\bm x^*$. The radius of the ball is set to be the Euclidean distance from $\x^*$ to the closest training sample. We found that the performance is not very sensitive to the choice of neighbor distance. Also, the computational cost of this smoothing estimator only involves new predictions, which is also minor compared to fitting random forests.

%%%%%%%%%%%%%%%%%%%%%%%%%%%%%%%%%%%%%%%%

%In this section, we first discuss the normality of $U_n$ under $k = \beta n$ setting and then show the connection between our proposed variance estimator with that in \citet{wang2014variance}'s work.
%The asymptotic analysis of our variance estimator is deferred to Section \ref{sec:theory results}. 

\subsection{A Discussion on Existing Normality Theories of Random Forests}
\label{sec:subsec:no normality}

Before demonstrating the simulation results, we would like to discuss the normality theories of random forests briefly. The main concern is that there is no universal guarantee of normality for random forests, and a variance estimator may not ensure the desired coverage rate. Hence, the use of any variance estimators should be done with a reasonable understanding of the random forest itself, especially by considering the impact of its tuning parameters. 

Many existing works in the literature have studied the asymptotic normality of $U_n$ given $k = o(\sqrt{n})$ to $o(n)$ under various regularity conditions \citep{lucas:old,2018wager:InfJack,Romano2020, lucas:new, peng2019asymptotic, athey2019generalized}. Existing empirical study also shows that the normality usually holds when $k$ is small while begins to break down for certain cases \citep[Table 2]{lucas:new}. As we will see in the following, there are both examples and counter-examples for the asymptotic normality of $U_n$ with a large $k$, depending on the specific form of the kernel.

Essentially, when a kernel $h(\cdot)$ is very adaptive to local observations without much randomness, e.g., 1-nearest neighbors and the kernel size is at the same order of $n$, there is too much dependency across different $h(S_i)$'s. This prevents the normality of $U_n$. On the other hand, when the kernel size is relatively small, there is enough variation across different kernel functions to establish normality. This is the main strategy used in the literature for establishing normality. The following example demonstrates these ideas. 

\begin{example}
\label{example:ratio consistency for 1nn}
    Given covariate-response pairs: $Z_1 = (x_1, Y_1), ...., Z_n = (x_n, Y_n)$ as training samples, where $x_i$'s are unique and deterministic numbers and $Y_i$'s i.i.d. $F$ such that $\E(Y_i) = \mu > 0, \Var(Y_i) = \sigma^2$, for $i = 1, 2, ..., n$.
    We want to predict the response for a given testing sample $x^*$.
    
    Suppose we have two size-k ($k = \beta n$) kernels: 1) a simple (linear) average kernel: $h(S) = \frac{1}{k} \sum_{Z_j \in S} Y_j$; 2) a 1-nearest neighbor (1-NN) kernel, which predicts using the closest training sample of $x^*$ based on the distance of $x$. Without loss of generality, we assume that $x_i$'s are ordered such that $x_i$ is the $i$-th nearest sample to $x^*$. We denote corresponding sub-bagging estimator as $U_\text{mean}$ and $U_\text{1-NN}$ respectively. 
    It is trivial to show that 
    \begin{align*}
        U_\text{mean} = \frac{1}{n} \sum_{i=1}^n Y_n, \quad
        U_\text{1-NN} = \sum_{i=1}^{n-k+1} a_i Y_i, 
    \end{align*}
    where $a_i = \binom{n-i}{k-1} / \binom{n}{k}$ and $\sum_{i=1}^{n-k+1} a_i = 1$.
    Accordingly, we have $\Var(U_\text{mean}) = \frac{1}{n} \sigma^2$ and $\Var(U_{1-NN}) \geq a_1^2 \Var(Y_1) = \frac{k^2}{n^2} \sigma^2 = \beta^2 \sigma^2$. Since $U_\text{mean}$ is a sample average, we still obtain asymptotic normality after scaling by $\sqrt{n}$. However, $\beta = k/n > 0$,  $a_1$ makes a significant proportion in the sum of all $a_i$'s and $\Var(U_{1-NN})$ does not decay to 0 as $n$ grows. Hence, asymptotic normality is not satisfied for $U_\text{1-NN}$. 
\end{example}

In practice, it is difficult to know apriori what type of data dependence structure these $h(S_i)$'s may satisfy. Thus, the normality of a random forest with a large subsampling size is still an open question and requires further understanding of its kernel. In our simulation study, we observe that the confidence intervals constructed with normal quantiles work well, given that data are generated with Gaussian noise (see Section \ref{sec:subsec:simulation setting}).

\section{Simulation Study}
\label{sec:simulation}

We present simulation studies to compare our variance estimator with existing methods \citep{lucas:new, 2018wager:InfJack} on random forests. We consider both the smoothed and non-smoothed versions, denoted as ``{\propSmooth}'' and ``{\propEst}'', respectively. The balance estimator and its bias-corrected version in \citet{lucas:new} are denoted as ``BM'' and ``BM-cor''. The infinitesimal jackknife in \citet{2018wager:InfJack} is denoted as ``IJ''. Our simulation does not include the Internal Estimator and the External Estimator in \citet{lucas:old}, since the BM method has been shown to be superior to these estimators \citep{lucas:new}. Note that the BM estimator works for both U-statistics and V-statistics \citep[Section4, paragraph 1]{lucas:new}. However, the V-statistics version is almost equivalent to IJ \citep[Theroem 3.3 and 3.4]{lucas:new}. Hence, in our simulation, we only include the U-statistics version.  

\subsection{Simulation Settings}
\label{sec:subsec:simulation setting}

We consider two regression settings:
\begin{enumerate}
    \item MARS: $f(\bm x) = 10  \sin(\pi x_1 x_2) + 20  (x_3-0.05)^2 + 10  x_4 + 5  x_5$; \quad $\calX = [0, 1]^{6}$.
    \item MLR: $f(\bm x) = 2  x_1 + 3  x_2 - 5  x_3 -x_4 + 1$; \quad
    $\calX = [0, 1]^6$.
\end{enumerate}
The MARS model is proposed by \citet{friedman1991multivariate} for the multivariate adaptive regression splines. It has been used previously by \citet{biau2012analysis, lucas:old}. The second model is a simple multivariate linear regression. In both settings, features $\bm x  = (x_1, \ldots, x_6)$ are generated uniformly from the feature space, and responses are generated by $f(\bm x) + \epsilon$, where $\epsilon \overset{iid}{\sim} \calN(0, 1)$.

We use $n = 200$ as the total training sample size and pick different subsample sizes: $k = 100, 50, 25$ when $k\leq n/2$ and $k = 160$ when $k>n/2$. The numbers of trees are \texttt{nTrees} $= B \cdot M = 2000, 10000, 20000$. For tuning parameters, we set {\mtry} as 3, which is half of the dimension, and set \texttt{nodesize} parameter to $2 \lfloor \log(n) \rfloor = 8$. We repeat the simulation $N_{mc} = 1000$ times to evaluate the performance of different estimators. Our proposed methods (\propEst, \propSmooth), BM and BM-cor estimators are implemented using \if0\blind {the \texttt{RLT} package available on GitHub}\fi \if1\blind{our package}\fi . The IJ estimators are implemented using \texttt{grf} and \texttt{ranger}. Each estimation method and its corresponding ground truth (see details in the following) is generated by the same package. Note that we do not use the honest tree setting by \citet{2018wager:InfJack}, since it is not essential for estimating the variance. However, it may affect the coverage rate due to the normality behavior. 

The performance of the variance estimator is evaluated in terms of its bias and the coverage rate of its corresponding confidence interval. We denote the random forest estimator as $\widehat f(\bm x^)$ and evaluate the coverage based on the mean of the random forest estimator, $\E (\widehat f(\bm x^))$, instead of the true model value, $f(\bm x^)$, as our focus is the variance estimation of $\widehat f(\bm x^)$ and the random forest itself may be a biased model. To obtain the ground truth of the variance, we generate the training dataset 10000 times and fit a random forest to each, using the mean and variance of the 10000 forest predictions as approximations of $\E (\widehat f(\bm x^))$ and $\Var ( \widehat f(\bm x^))$. The relative bias and the confidence interval (CI) convergence are the evaluation criteria, with the relative bias defined as the ratio of the bias to the ground truth of the variance estimation. The $1 - \alpha$ CI is constructed using $\hat f \pm Z_{\alpha/2}\sqrt{ \vu }$, where $Z_{\alpha/2}$ is the standard normal quantile.

We evaluate the variance estimation on two types of testing samples for both MARS and MLR data. The first is a central sample with $x^* = (0.5, \ldots, 0.5)$ and the second includes 50 random samples whose coordinates are independently sampled from a uniform distribution between $[0, 1]$. These testing samples are fixed for all experiments. The central sample is used to show the distribution of variance estimators over 1000 simulations, while the 50 random samples are used to evaluate the average bias and CI coverage rate. The results of the evaluation are presented in Figure \ref{fig:3 panel results} and Tables \ref{tab:cover over k and B} and \ref{tab:bias and sd}. A small difference in the ground truth generated by different packages is noted in Appendix \ref{sec:append:simu} due to subtle differences in the packages' implementations.

The computational cost of the different methods is similar, as the main cost lies in fitting trees in the random forest, rather than calculating the estimator from tree predictions. While the $\propSmooth$ method may incur additional costs for obtaining tree predictions for the target testing sample's neighboring samples, this added cost is low as it only involves making predictions using existing trees, rather than fitting new trees.

\subsection{Results for \texorpdfstring{$k \leq n/2$}{k<=n}}

\begin{figure*}[htbp]
    \centering
      \begin{minipage}[b]{\textwidth}
        \includegraphics[width=\textwidth]{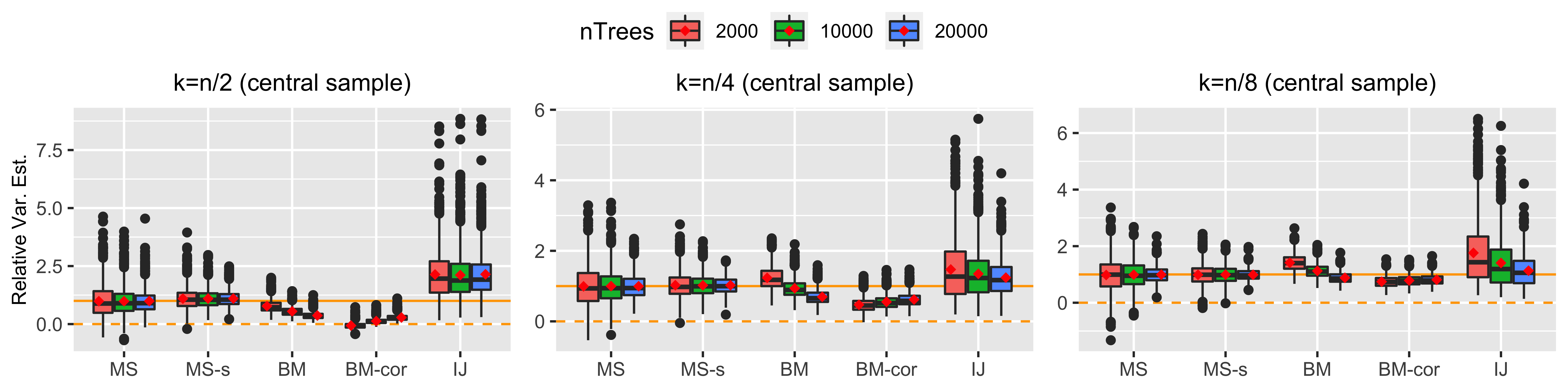}
      \end{minipage}
      \hfill
      \begin{minipage}[b]{\textwidth}
        \includegraphics[width=\textwidth]{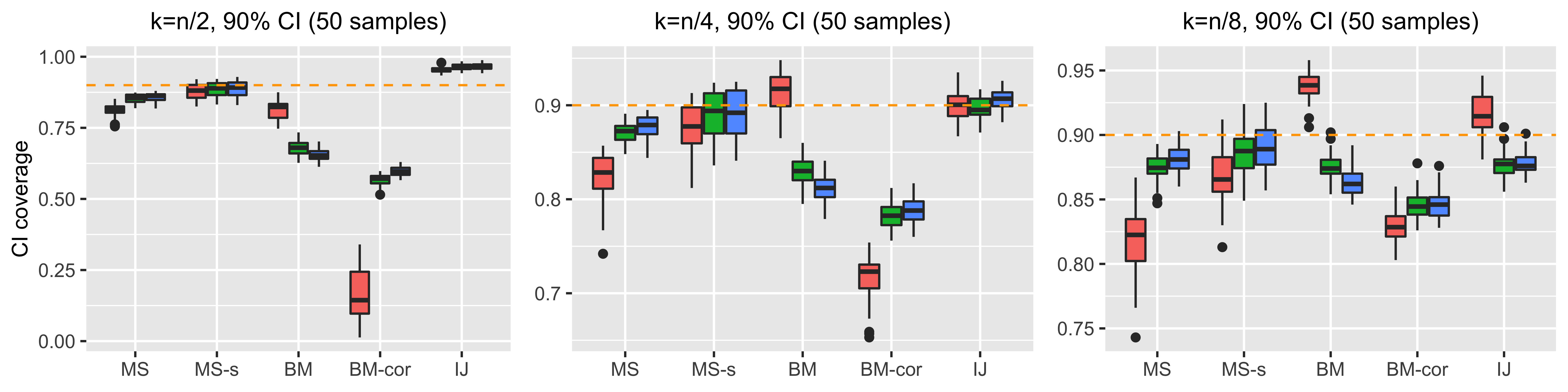}
      \end{minipage}
      \hfill
        \begin{minipage}[b]{\textwidth}
        \includegraphics[width=\textwidth]{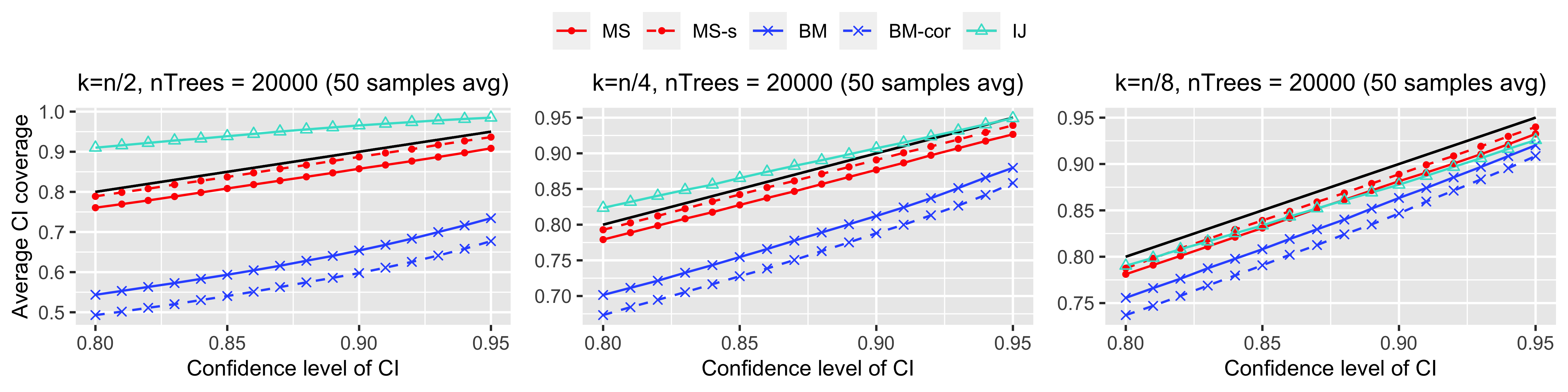}
      \end{minipage}
    \caption{A comparison of different methods on the MARS data is presented.
Each column in the figure represents a different tree size: $k = n/2, n/4, n/8$ respectively.
The first row displays boxplots of the relative variance estimators on a central test sample, evaluated over 1000 simulations. The mean is represented by the red diamond symbol in each boxplot.
The second row displays boxplots of the 90% confidence interval (CI) coverage for 50 testing samples.
The third row displays the average coverage rate over 50 testing samples, with {\nTrees} $= 20000$. The black reference line, $y=x$, represents the desired coverage rate.}
    \label{fig:3 panel results}
\end{figure*}

\begin{table*}[htbp]
\centering
\caption{$90\%$ CI Coverage Rate averaged on 50 testing samples. The number in the bracket is the standard deviation of coverage over 50 testing samples.}
\label{tab:cover over k and B}
\resizebox{\textwidth}{!}{
\begin{tabular}{lll|ll|ll}
\hline
       & \multicolumn{2}{c|}{$k=n/2$}    & \multicolumn{2}{c|}{$k=n/4$}    & \multicolumn{2}{c}{$k=n/8$}     \\ \hline
{\nTrees} & 2000           & 20000          & 2000           & 20000          & 2000           & 20000          \\ \hline
\textbf{MARS}   &                &                &                &                &                &                \\
MS     & 81.2\% (2.0\%) & 85.8\% (1.6\%) & 82.3\% (2.6\%) & 87.7\% (1.2\%) & 81.8\% (2.6\%) & 88.1\% (1.1\%) \\
MS-s & 87.7\% (2.7\%) & 88.7\% (2.7\%) & 87.7\% (2.6\%) & 89.1\% (2.5\%) & 86.9\% (2.0\%) & 88.9\% (1.7\%) \\
BM     & 81.3\% (3.2\%) & 65.4\% (2.0\%) & 91.4\% (1.9\%) & 81.2\% (1.5\%) & 93.8\% (1.1\%) & 86.3\% (1.1\%) \\
BM-cor & 16.7\% (9.0\%) & 59.8\% (1.6\%) & 71.7\% (2.3\%) & 78.8\% (1.4\%) & 83.0\% (1.1\%) & 84.7\% (1.1\%) \\
IJ     & 95.4\% (1.0\%) & 96.6\% (1.0\%) & 89.9\% (1.5\%) & 90.7\% (1.0\%) & 91.7\% (1.6\%) & 87.8\% (0.9\%) \\ \hline
\textbf{MLR}    &                &                &                &                &                &                \\
MS     & 83.3\% (1.4\%) & 86.4\% (1.2\%) & 84.5\% (1.5\%) & 88.2\% (1.0\%) & 84.1\% (1.6\%) & 88.9\% (1.0\%) \\
MS-s & 88.8\% (1.6\%) & 89.6\% (1.5\%) & 89.1\% (1.6\%) & 90.3\% (1.5\%) & 88.6\% (1.6\%) & 90.3\% (1.2\%) \\
BM     & 79.4\% (2.0\%) & 64.7\% (1.4\%) & 90.7\% (1.3\%) & 80.9\% (1.3\%) & 93.8\% (0.9\%) & 86.6\% (1.2\%) \\
BM-cor & 23.1\% (5.6\%) & 59.9\% (1.6\%) & 73.0\% (1.9\%) & 78.7\% (1.4\%) & 83.6\% (1.4\%) & 85.2\% (1.2\%) \\
IJ     & 95.6\% (0.8\%) & 96.5\% (0.6\%) & 89.5\% (1.1\%) & 91.1\% (1.1\%) & 91.4\% (1.1\%) & 88.1\% (1.2\%) \\ \hline
\end{tabular}
}
\end{table*}

\begin{table*}[htbp]
\centering
\caption{Relative bias (standard deviation) over 50 testing samples. For each method and testing sample, the relative bias is evaluated over 1000 simulations. }
\label{tab:bias and sd}
\resizebox{\textwidth}{!}{%
\begin{tabular}{lll|ll|ll}
\hline
       & \multicolumn{2}{c|}{$k=n/2$}      & \multicolumn{2}{c|}{$k=n/4$}    & \multicolumn{2}{c}{$k=n/8$}      \\ \hline
{\nTrees} & 2000            & 20000           & 2000           & 20000          & 2000            & 20000          \\ \hline
\textbf{MARS}   &                 &                 &                &                &                 &                \\
MS     & -0.3\% (1.7\%)  & -0.2\% (1.4\%)  & -0.2\% (2.0\%) & 0.1\% (1.3\%)  & 0.3\% (1.8\%)   & 0.5\% (1.3\%)  \\
MS-s & 2.0\% (13.0\%)  & 2.3\% (13.5\%)  & 1.8\% (12.2\%) & 1.9\% (12.5\%) & 0.8\% (8.5\%)   & 1.2\% (8.7\%)  \\
BM     & -28.8\% (8.6\%)  & -64.1\% (1.1\%)  & 20.6\% (12.2\%) & -30.9\% (1.6\%) & 40.5\% (9.1\%)  & -12.0\% (1.5\%) \\
BM-cor & -101.1\% (8.1\%) & -71.4\% (1.0\%)  & -52.4\% (3.9\%) & -38.3\% (0.9\%) & -24.4\% (1.7\%) & -18.6\% (1.1\%) \\
IJ     & 102.3\% (21.5\%) & 103.5\% (21.8\%) & 36.6\% (10.1\%) & 20.8\% (9.2\%)  & 67.4\% (15.4\%) & 11.5\% (6.7\%)  \\ \hline
\textbf{MLR}    &                 &                 &                &                &                 &                \\
MS     & 0.3\% (2.7\%)   & 0.1\% (2.1\%)   & -0.1\% (2.0\%) & 0.0\% (1.8\%)  & 0.0\% (2.1\%)   & -0.2\% (1.6\%) \\
MS-s & 6.0\% (7.4\%)   & 6.2\% (7.4\%)   & 5.8\% (7.1\%)  & 6.1\% (7.0\%)  & 4.8\% (4.9\%)   & 4.6\% (5.0\%)  \\
BM     & -36.2\% (3.8\%)  & -65.4\% (0.9\%)  & 11.4\% (5.9\%)  & -32.4\% (1.4\%) & 32.1\% (5.8\%)  & -13.7\% (1.5\%) \\
BM-cor & -95.0\% (3.2\%)  & -71.3\% (0.7\%)  & -50.1\% (1.8\%) & -38.6\% (1.1\%) & -24.7\% (1.2\%) & -19.6\% (1.1\%) \\
IJ     & 87.8\% (15.0\%) & 88.6\% (14.7\%) & 27.1\% (5.7\%) & 17.1\% (5.8\%) & 53.1\% (11.4\%) & 6.6\% (5.1\%)  \\ \hline
\end{tabular}
}
\end{table*}

Figure \ref{fig:3 panel results} presents the evaluation results for the MARS data. The subfigures show the distribution of variance estimators on the central test sample and the corresponding 90% confidence interval (CI) coverage on 50 testing samples. As previously mentioned, the bias of each estimator is compared to its population mean, eliminating the influence from different packages.
The results for the MLR data are provided in Appendix \ref{sec:append:simu} and show similar patterns.
Tables \ref{tab:cover over k and B} and \ref{tab:bias and sd} present the 90\% CI coverage rate and relative bias of the variance estimation, respectively. The coverage for each method is calculated as the average over 50 testing samples, and the standard deviation, indicated in the bracket, reflects the variation among these samples. Our simulation results show that the random forest estimators are approximately normally distributed, as the CIs constructed using the true variance achieve the desired confidence level (see Appendix \ref{sec:append:simu}).
In summary, {\propEst} and {\propSmooth} demonstrate consistently better performance compared to other methods, especially when the tree size $k$ is large, i.e., $k = n/2$. The improved performance can be seen in terms of accurate CI coverage and reduced bias.

First, the third row of Figure \ref{fig:3 panel results} shows that the {\propSmooth} method achieves the best CI coverage under every $k$, i.e., the corresponding line is nearest to the reference line: $y = x$. The {\propEst} method performs the second best when $k = n/2$ and $n/8$.
Furthermore, the CI coverages of the proposed methods are stable over different testing samples with a small standard deviation (less than 3\%), as seen in Table \ref{tab:cover over k and B}.
Secondly, with regards to the bias of the variance estimation, our methods show a much smaller bias than all other approaches (Figure \ref{fig:3 panel results}, first row).
More details of the relative bias are summarized in Table \ref{tab:bias and sd}. The average bias of {\propEst} is smaller than $0.5\%$ with a small standard deviation, mainly due to the Monte Carlo error. The {\propSmooth} method has a slightly positive average bias ($0\%$ to $6.2\%$), but it is still much smaller than the competing methods. The standard deviation of bias for {\propSmooth} is around $4.3\%$ to $13.6\%$, which is comparable to IJ.

On the other hand, the performance of the competing methods varies. When the tree size is $k = n/2$, the BM, BM-cor, and IJ methods show a large bias, but their performance improves for smaller tree sizes. It is worth noting that these methods are theoretically designed for small $k$. BM and BM-cor tend to underestimate the variance in most settings, while IJ tends to overestimate. In Table \ref{tab:bias and sd}, on the MARS data with 20000 {\nTrees}, the bias of both BM and BM-cor is more than $-50\%$, resulting in severe under-coverage (65.4\%, 59.8\%), while IJ leads to over-coverage. Even when the tree size is as small as $k = n/8$, these methods still display a noticeable bias. However, the proposed methods still outperform them when more trees ({\nTrees} = 20000) are used, as shown in the last column of Table \ref{tab:bias and sd}.

The results indicate that the choice of the number of trees has a significant effect on the performance of the estimators. This is to be expected due to the influence of the random kernels, the variation involved in incomplete U-statistics, and other theoretical aspects. As the number of trees increases, the variation of all estimators decreases, as can be seen in the first row of Figure \ref{fig:3 panel results}. Our estimators, being mostly unbiased, benefit from larger \texttt{nTrees} values. For instance, the $90\%$ CI coverages of the {\propEst} method on the MARS data increase from $81.2\%$ ($k=n/2$) and $81.8\%$ ($k=n/8$) with {\nTrees} = 2000 to $85.8\%$ and $88.1\%$ respectively with {\nTrees} = 20000. On the other hand, the performance of competing methods does not necessarily improve with an increase in {\nTrees}. For example, BM shows over-coverage with {\nTrees} = 2000 but under-coverage with {\nTrees} = 20000 when $k=n/4$ or $n/8$. This phenomenon, known as estimation inflation, has been discussed in \citet{lucas:new} and is addressed by the BM-cor method, which reduces the bias. When $k=n/8$, the gap between BM and BM-cor decreases as {\nTrees} increases. However, this trend is no longer evident when $k$ is large, as the dominating term used in their theory is no longer applicable.

Finally, we would like to emphasize the relationship between the bias of the estimator and the coverage rate of the confidence interval. Even though a random forest predictor is normally distributed and the variance estimator is unbiased, large fluctuations of the variance estimator can still lead to under-coverage. The same also applies to the IJ estimator. For example, on MARS data with $k = n/8$ and {\nTrees} = 20000, IJ has a positive bias ($11.5\%$), but its confidence interval is still under-coverage and even more severe than the proposed methods. Increasing the number of trees can improve this performance to some extend. An alternative strategy is to perform local averaging as implemented in the {\propSmooth} method, especially when {\nTrees} is relatively small. The heights of the boxplots in the figure clearly demonstrate the variance reduction effect. As a result, the {\propSmooth} method with 2000 trees shows better coverage than the {\propEst} method with 20000 trees when $k = n/2$ (see Table \ref{tab:cover over k and B}). However, this maybe at the cost of larger bias. Hence, we still recommend using a larger number of trees whenever it is computationally feasible.

\subsection{Results for \texorpdfstring{$k > n/2$}{k>n/2}}

As discussed in Section \ref{sec:subsec:extension k>n/2}, when $n/2 < k < n$, we cannot jointly estimate $\vh$ and $\vs$. Additional computational cost is introduced using the bootstrap approach for estimating $\vh$. In this simulation study, we attempt to fit additional {\nTrees} with bootstrapping (sampling with replacement) subsamples to estimate $\vh$ so we denote our proposed estimator and smoothing estimator as ``MS(bs)'' and ``MS-s(bs)''. We note that the \texttt{grf} package does not provide IJ estimator when $k > n/2$ so we generate the IJ estimator and corresponding ground truth by the \texttt{ranger} package.

\begin{table*}[ht]
\centering
\caption{90 \% CI coverage, relative bias, and standard deviation averaged on 50 testing samples. Tree size $k = 0.8 n$. The calculation follows previous tables.}
\label{tab:bias and sd large k}
\small
\begin{tabular}{llll|ll}
\hline
     &               & \multicolumn{2}{c|}{90\% CI Coverage} & \multicolumn{2}{c}{Relative Bias}    \\ \hline
Model & nTrees        & 2000              & 20000             & 2000              & 20000            \\ \hline
\textbf{MARS} & MS(bs)     & 94.2\% (2.8\%)    & 95.4\% (2.4\%)    & 128.4\% (64.8\%)  & 136.6\% (67.2\%) \\
     & MS-s(bs) & 97.7\% (1.5\%)    & 98.1\% (1.3\%)    & 132.2\% (66.7\%)  & 140.6\% (69.1\%) \\
     & BM            & 51.4\% (3.8\%)    & 33.9\% (1.7\%)    & -80.4\% (3.1\%)   & -92.1\% (0.5\%)  \\
     & BM-cor        & 0.0\% (0.0\%)     & 13.5\% (4.5\%)    & -143.0\% (12.1\%) & -98.3\% (1.3\%)  \\
     & IJ            & 88.0\% (4.6\%)    & 87.1\% (3.7\%)    & -0.8\% (25.2\%)   & -5.6\% (16.3\%)  \\ \hline
\textbf{MLR}  & MS(bs)            & 94.3\% (1.9\%)    & 95.2\% (1.7\%)    & 98.4\% (24.7\%)   & 103.9\% (25.4\%) \\
     & MS-s(bs)       & 96.6\% (1.3\%)    & 97.0\% (1.2\%)    & 104.8\% (24.9\%)  & 110.3\% (25.6\%) \\
     & BM            & 47.9\% (2.3\%)    & 32.4\% (1.5\%)    & -83.4\% (1.2\%)   & -92.6\% (0.3\%)  \\
     & BM-cor        & 0.0\% (0.0\%)     & 15.9\% (2.4\%)    & -132.7\% (4.3\%)  & -97.5\% (0.5\%)  \\
     & IJ            & 99.4\% (0.3\%)    & 99.2\% (0.3\%)    & 182.8\% (21.7\%)  & 175.8\% (16.7\%) \\ \hline
\end{tabular}
\end{table*}

As seen from Table \ref{tab:bias and sd large k}, all methods suffer from severe bias, but our methods and IJ are comparable and better than BM and BM-cor. More specifically, our proposed method generally over-covers due to overestimating the variance. The IJ method shows good accuracy on MARS data but has more severe over-coverage than our methods on MLR. Overall, to obtain a reliable conclusion of statistical inference, we recommend avoiding using $k > n/2$. This can be a reasonable setting when $n$ is relatively large, and $k = n/2$ can already provide an accurate model. 

\section{Real Data Illustration}
\label{sec:real data}

We use the Seattle Airbnb Listings dataset, which was obtained from Kaggle\footnote{https://www.kaggle.com/shanelev/seattle-airbnb-listings}. The purpose of this analysis is to predict the price of Airbnb units in Seattle. The dataset consists of 7515 samples and nine covariates, including latitude, longitude, room type, number of bedrooms, number of bathrooms, number of accommodates, number of reviews, presence of a rating, and the rating score. Further information about the dataset, including the missing value processing, can be found in Appendix \ref{sec:append:real}.

\begin{figure}[htbp]
    \centering
    \includegraphics[width=0.5\textwidth]{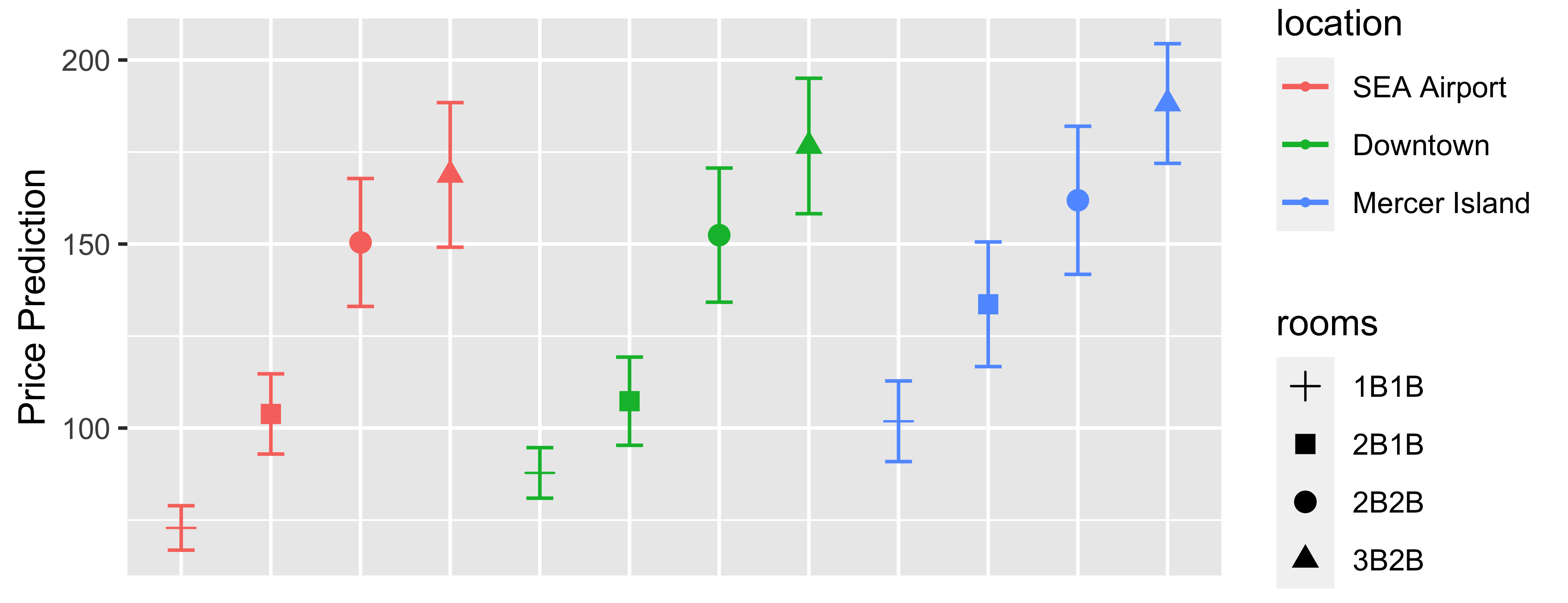}
    \caption{Random Forest prediction on Airbnb testing data. The $95\%$ confidence error bar is generated with our variance estimator, {\propEstFull}.  ``2B1B'' denotes the house/apartment has two bedrooms and one bathroom.}
    \label{fig:realData}
\end{figure}

Given the large sample size, we fit 40000 trees to obtain a variance estimator. The tree size is fixed as half of the sample size: $k = 3757$. We construct 12 testing samples at 3 locations: Seattle-Tacoma International Airport (SEA Airport), Seattle downtown, and Mercer Island. We further consider four bedroom/bathroom settings as $1B1B$, $2B1B$, $2B2B$, and $3B2B$. Details of the latitude and longitude of these locations and other covariates are described in Appendix \ref{sec:append:real}. 
The price predictions, along with 95\% confidence intervals, are presented in Figure \ref{fig:realData}. Overall, the predictions match our intuitions. 
In particular, we can observe that the confidence interval of 1B1B units at SEA Airport does not overlap with those corresponding to the same unit type at the other two locations. This is possible because the accommodations around an airport usually have lower prices due to stronger competition.
We also observe that 2-bathroom units at SEA Airport and downtown have higher prices than 1-bathroom units. However, the difference between 2B2B and 3B2B units at SEA Airport is insignificant. 

\section{Discussion}\label{sec:disc}

From the perspective of $U$-statistics, we have proposed a new framework of variance estimator for infinite-order U statistics. Instead of utilizing the leading term dominance property, we instead establish the peak region dominance notion, which addresses the bias issue under large subsampling size $k$ or small sample size $n$. Additionally, new tools and strategies have been developed to study the ratio consistency behavior which is crucial for obtaining a proper coverage rate. Here, we discuss several open issues and possible extensions for future research.

First, our current methods are computationally valid for $k \leq n/2$. The difficulty of extending to the $k > n/2$ region is to estimate the tree variance, i.e., $\vh$. We proposed to use bootstrapped trees to extend the method to $k > n/2$. However, this could introduce additional bias and also leads to large variation, as we can see in the simulation study. We suspect Bootstrapping may be sensitive to the randomness involved in fitting trees. Since we estimate $\vh$ and $\vs$ separately, the randomness of the tree kernel could introduce different added variances, which leads to non-negligible bias. When $k > n/2$,  \citet{wang2017pseudo} propose an asymptotic unbiased variance estimator for the U-statistic estimator of a Kullback-Leibler risk in the $k$-fold cross-validation. However, this depends on a specific approximation of the kernel of Kullback-Leibler risk. The problem remains open for a general kernel.

Secondly, we developed a new double-$U$ statistics tool to prove ratio consistency. This is the first work that analyzes the ratio consistency of a minimum-variance unbiased estimator (UMVUE) of a U-statistic's variance. The tool can be potentially applied to theoretical analyses of a general family of U-statistic problems. However, our ratio consistency result is still limited to $k = o(n^{1/2 - \epsilon})$, introducing a gap between theoretical and practical versions. 
The limitation comes from the procedure we used to drive the Hoeffding decomposition of the variance estimator's variance. In particular, we want the leading term to dominate the variance while allowing a super-linear growth rate of each $\sigma_{c, 2k}^2$ in terms of $c$. 
Hence, the extension to the $k = \beta n$ setting is still open and may require further assumptions on the overlapping structures of double-$U$ statistics.

Thirdly, in our smoothed estimator, the choice of testing sample neighbors can be data-dependent and relies on the forest-defined distance. It is worth considering more robust smoothing methods for future work.

Lastly, this paper focuses on the regression problem using random forest. This variance estimator can also be applied to the general family of subbagging estimators. Besides, we may further investigate the uncertainty quantification for variable importance, the confidence interval for classification probability, the confidence band of survival analysis, etc.

%\bibliographystyle{imsart-number} % Style BST file (imsart-number.bst or imsart-nameyear.bst)
%\bibliography{bibliography}       % Bibliography file (usually '*.bib')

%\bibliographystyle{asa}
\bibliographystyle{imsart-nameyear}
\bibliography{citation}

\begin{thebibliography}{30}
% BibTex style file: imsart-nameyear.bst, 2017-11-03
% Default style options (sort=1,type=nameyear).
% Used options (sort=1,type=nameyear).

\bibitem[\protect\citeauthoryear{Athey, Tibshirani and
  Wager}{2019}]{athey2019generalized}
\begin{barticle}[author]
\bauthor{\bsnm{Athey},~\bfnm{Susan}\binits{S.}},
  \bauthor{\bsnm{Tibshirani},~\bfnm{Julie}\binits{J.}} \AND
  \bauthor{\bsnm{Wager},~\bfnm{Stefan}\binits{S.}}
(\byear{2019}).
\btitle{Generalized Random Forests}.
\bjournal{The Annals of Statistics}
\bvolume{47}
\bpages{1148--1178}.
\end{barticle}
\endbibitem

\bibitem[\protect\citeauthoryear{Biau}{2012}]{biau2012analysis}
\begin{barticle}[author]
\bauthor{\bsnm{Biau},~\bfnm{G{\'e}rard}\binits{G.}}
(\byear{2012}).
\btitle{Analysis of a Random Forests Model}.
\bjournal{The Journal of Machine Learning Research}
\bvolume{13}
\bpages{1063--1095}.
\end{barticle}
\endbibitem

\bibitem[\protect\citeauthoryear{Breiman}{2001}]{breiman2001random}
\begin{barticle}[author]
\bauthor{\bsnm{Breiman},~\bfnm{Leo}\binits{L.}}
(\byear{2001}).
\btitle{Random Forests}.
\bjournal{Machine Learning}
\bvolume{45}
\bpages{5--32}.
\end{barticle}
\endbibitem

\bibitem[\protect\citeauthoryear{Chen and Kato}{2019}]{chen2019randomized}
\begin{barticle}[author]
\bauthor{\bsnm{Chen},~\bfnm{Xiaohui}\binits{X.}} \AND
  \bauthor{\bsnm{Kato},~\bfnm{Kengo}\binits{K.}}
(\byear{2019}).
\btitle{Randomized incomplete $ U $-statistics in high dimensions}.
\bjournal{The Annals of Statistics}
\bvolume{47}
\bpages{3127--3156}.
\end{barticle}
\endbibitem

\bibitem[\protect\citeauthoryear{Cochran}{2007}]{cochran1977sampling}
\begin{bbook}[author]
\bauthor{\bsnm{Cochran},~\bfnm{William~G}\binits{W.~G.}}
(\byear{2007}).
\btitle{Sampling Techniques},
\bedition{3} ed.
\bpublisher{John Wiley \& Sons}.
\end{bbook}
\endbibitem

\bibitem[\protect\citeauthoryear{DiCiccio and Romano}{2022}]{Romano2020}
\begin{barticle}[author]
\bauthor{\bsnm{DiCiccio},~\bfnm{Cyrus}\binits{C.}} \AND
  \bauthor{\bsnm{Romano},~\bfnm{Joseph}\binits{J.}}
(\byear{2022}).
\btitle{CLT for U-Statistics with Growing Dimension}.
\bjournal{Statistica Sinica}
\bvolume{32}
\bpages{1--22}.
\end{barticle}
\endbibitem

\bibitem[\protect\citeauthoryear{Efron}{2014}]{efron2014IJ}
\begin{barticle}[author]
\bauthor{\bsnm{Efron},~\bfnm{Bradley}\binits{B.}}
(\byear{2014}).
\btitle{Estimation and Accuracy After Model Selection}.
\bjournal{Journal of the American Statistical Association}
\bvolume{109}
\bpages{991--1007}.
\end{barticle}
\endbibitem

\bibitem[\protect\citeauthoryear{Efron and Stein}{1981}]{efron1981jackknife}
\begin{barticle}[author]
\bauthor{\bsnm{Efron},~\bfnm{Bradley}\binits{B.}} \AND
  \bauthor{\bsnm{Stein},~\bfnm{Charles}\binits{C.}}
(\byear{1981}).
\btitle{The Jackknife Estimate of Variance}.
\bjournal{The Annals of Statistics}
\bpages{586--596}.
\end{barticle}
\endbibitem

\bibitem[\protect\citeauthoryear{Folsom}{1984}]{folsom1984probability}
\begin{bphdthesis}[author]
\bauthor{\bsnm{Folsom},~\bfnm{Ralph~E}\binits{R.~E.}}
(\byear{1984}).
\btitle{Probability Sample U-statistics: Theory and Applications for Complex
  Sample Designs},
\btype{PhD thesis},
\bpublisher{The University of North Carolina at Chapel Hill}.
\end{bphdthesis}
\endbibitem

\bibitem[\protect\citeauthoryear{Frees}{1989}]{frees1989infiniteOrder}
\begin{barticle}[author]
\bauthor{\bsnm{Frees},~\bfnm{Edward~W}\binits{E.~W.}}
(\byear{1989}).
\btitle{Infinite Order U-statistics}.
\bjournal{Scandinavian Journal of Statistics}
\bpages{29--45}.
\end{barticle}
\endbibitem

\bibitem[\protect\citeauthoryear{Friedman}{1991}]{friedman1991multivariate}
\begin{barticle}[author]
\bauthor{\bsnm{Friedman},~\bfnm{Jerome~H}\binits{J.~H.}}
(\byear{1991}).
\btitle{Multivariate Adaptive Regression Splines}.
\bjournal{The Annals of Statistics}
\bpages{1--67}.
\end{barticle}
\endbibitem

\bibitem[\protect\citeauthoryear{Geurts, Ernst and
  Wehenkel}{2006}]{geurts2006extremely}
\begin{barticle}[author]
\bauthor{\bsnm{Geurts},~\bfnm{Pierre}\binits{P.}},
  \bauthor{\bsnm{Ernst},~\bfnm{Damien}\binits{D.}} \AND
  \bauthor{\bsnm{Wehenkel},~\bfnm{Louis}\binits{L.}}
(\byear{2006}).
\btitle{Extremely Randomized Trees}.
\bjournal{Machine Learning}
\bvolume{63}
\bpages{3--42}.
\end{barticle}
\endbibitem

\bibitem[\protect\citeauthoryear{Hoeffding}{1948}]{hoeffding1948}
\begin{barticle}[author]
\bauthor{\bsnm{Hoeffding},~\bfnm{Wassily}\binits{W.}}
(\byear{1948}).
\btitle{A Class of Statistics with Asymptotically Normal Distribution}.
\bjournal{Ann. Math. Statist.}
\bvolume{19}
\bpages{293--325}.
\bdoi{10.1214/aoms/1177730196}
\end{barticle}
\endbibitem

\bibitem[\protect\citeauthoryear{Horvitz and
  Thompson}{1952}]{horvitz1952generalization}
\begin{barticle}[author]
\bauthor{\bsnm{Horvitz},~\bfnm{Daniel~G}\binits{D.~G.}} \AND
  \bauthor{\bsnm{Thompson},~\bfnm{Donovan~J}\binits{D.~J.}}
(\byear{1952}).
\btitle{A Generalization of Sampling Without Replacement From a Finite
  Universe}.
\bjournal{Journal of the American Statistical Association}
\bvolume{47}
\bpages{663--685}.
\end{barticle}
\endbibitem

\bibitem[\protect\citeauthoryear{Lee}{1990}]{ustatBook}
\begin{bbook}[author]
\bauthor{\bsnm{Lee},~\bfnm{A~J}\binits{A.~J.}}
(\byear{1990}).
\btitle{U-statistics: Theory and Practice}.
\bpublisher{CRC Press}.
\end{bbook}
\endbibitem

\bibitem[\protect\citeauthoryear{Mentch and Hooker}{2016}]{lucas:old}
\begin{barticle}[author]
\bauthor{\bsnm{Mentch},~\bfnm{Lucas}\binits{L.}} \AND
  \bauthor{\bsnm{Hooker},~\bfnm{Giles}\binits{G.}}
(\byear{2016}).
\btitle{Quantifying Uncertainty in Random Forests via Confidence Intervals and
  Hypothesis Tests}.
\bjournal{Journal of Machine Learning Research}
\bvolume{17}
\bpages{841–881}.
\end{barticle}
\endbibitem

\bibitem[\protect\citeauthoryear{Peng, Coleman and
  Mentch}{2022}]{peng2019asymptotic}
\begin{barticle}[author]
\bauthor{\bsnm{Peng},~\bfnm{Wei}\binits{W.}},
  \bauthor{\bsnm{Coleman},~\bfnm{Tim}\binits{T.}} \AND
  \bauthor{\bsnm{Mentch},~\bfnm{Lucas}\binits{L.}}
(\byear{2022}).
\btitle{Rates of convergence for random forests via generalized U-statistics}.
\bjournal{Electronic Journal of Statistics}
\bvolume{16}
\bpages{232--292}.
\end{barticle}
\endbibitem

\bibitem[\protect\citeauthoryear{Peng, Mentch and
  Stefanski}{2021}]{peng2021bias}
\begin{barticle}[author]
\bauthor{\bsnm{Peng},~\bfnm{Wei}\binits{W.}},
  \bauthor{\bsnm{Mentch},~\bfnm{Lucas}\binits{L.}} \AND
  \bauthor{\bsnm{Stefanski},~\bfnm{Leonard}\binits{L.}}
(\byear{2021}).
\btitle{Bias, Consistency, and Alternative Perspectives of the Infinitesimal
  Jackknife}.
\bjournal{arXiv preprint arXiv:2106.05918}.
\end{barticle}
\endbibitem

\bibitem[\protect\citeauthoryear{Schucany and
  Bankson}{1989}]{schucany1989small}
\begin{barticle}[author]
\bauthor{\bsnm{Schucany},~\bfnm{William~R}\binits{W.~R.}} \AND
  \bauthor{\bsnm{Bankson},~\bfnm{Daniel~M}\binits{D.~M.}}
(\byear{1989}).
\btitle{Small sample variance Estimators for U-statistics}.
\bjournal{Australian Journal of Statistics}
\bvolume{31}
\bpages{417--426}.
\end{barticle}
\endbibitem

\bibitem[\protect\citeauthoryear{Sen}{1953}]{sen1953estimate}
\begin{barticle}[author]
\bauthor{\bsnm{Sen},~\bfnm{Amode~R}\binits{A.~R.}}
(\byear{1953}).
\btitle{On the Estimate of the Variance in Sampling With Varying
  Probabilities}.
\bjournal{Journal of the Indian Society of Agricultural Statistics}
\bvolume{5}
\bpages{127}.
\end{barticle}
\endbibitem

\bibitem[\protect\citeauthoryear{Sen}{1960}]{sen1960leading}
\begin{barticle}[author]
\bauthor{\bsnm{Sen},~\bfnm{Pranab~Kumar}\binits{P.~K.}}
(\byear{1960}).
\btitle{On some convergence properties of U-statistics}.
\bjournal{Calcutta Statistical Association Bulletin}
\bvolume{10}
\bpages{1--18}.
\end{barticle}
\endbibitem

\bibitem[\protect\citeauthoryear{Song, Chen and
  Kato}{2019}]{song2019high_dim_ustat}
\begin{barticle}[author]
\bauthor{\bsnm{Song},~\bfnm{Yanglei}\binits{Y.}},
  \bauthor{\bsnm{Chen},~\bfnm{Xiaohui}\binits{X.}} \AND
  \bauthor{\bsnm{Kato},~\bfnm{Kengo}\binits{K.}}
(\byear{2019}).
\btitle{Approximating high-dimensional infinite-order $ U $-statistics:
  Statistical and computational guarantees}.
\bjournal{Electronic Journal of Statistics}
\bvolume{13}
\bpages{4794--4848}.
\end{barticle}
\endbibitem

\bibitem[\protect\citeauthoryear{Wager and Athey}{2018}]{2018wager:InfJack}
\begin{barticle}[author]
\bauthor{\bsnm{Wager},~\bfnm{Stefan}\binits{S.}} \AND
  \bauthor{\bsnm{Athey},~\bfnm{Susan}\binits{S.}}
(\byear{2018}).
\btitle{Estimation and Inference of Heterogeneous Treatment Effects using
  Random Forests}.
\bjournal{Journal of the American Statistical Association}
\bvolume{113}
\bpages{1228-1242}.
\bdoi{10.1080/01621459.2017.1319839}
\end{barticle}
\endbibitem

\bibitem[\protect\citeauthoryear{Wager, Hastie and
  Efron}{2014}]{wager2014confidence}
\begin{barticle}[author]
\bauthor{\bsnm{Wager},~\bfnm{Stefan}\binits{S.}},
  \bauthor{\bsnm{Hastie},~\bfnm{Trevor}\binits{T.}} \AND
  \bauthor{\bsnm{Efron},~\bfnm{Bradley}\binits{B.}}
(\byear{2014}).
\btitle{Confidence intervals for random forests: The jackknife and the
  infinitesimal jackknife}.
\bjournal{The Journal of Machine Learning Research}
\bvolume{15}
\bpages{1625--1651}.
\end{barticle}
\endbibitem

\bibitem[\protect\citeauthoryear{Wang}{2012}]{wang2012investigation}
\begin{bphdthesis}[author]
\bauthor{\bsnm{Wang},~\bfnm{Qing}\binits{Q.}}
(\byear{2012}).
\btitle{Investigation of Topics in U-statistics and Their Applications in Risk
  Estimation and Cross-validation},
\btype{PhD thesis},
\bpublisher{Penn State University}.
\end{bphdthesis}
\endbibitem

\bibitem[\protect\citeauthoryear{Wang and Lindsay}{2014}]{wang2014variance}
\begin{barticle}[author]
\bauthor{\bsnm{Wang},~\bfnm{Qing}\binits{Q.}} \AND
  \bauthor{\bsnm{Lindsay},~\bfnm{Bruce}\binits{B.}}
(\byear{2014}).
\btitle{Variance Estimation of a General U-statistic with Application to
  Cross-validation}.
\bjournal{Statistica Sinica}
\bpages{1117--1141}.
\end{barticle}
\endbibitem

\bibitem[\protect\citeauthoryear{Wang and Lindsay}{2017}]{wang2017pseudo}
\begin{barticle}[author]
\bauthor{\bsnm{Wang},~\bfnm{Qing}\binits{Q.}} \AND
  \bauthor{\bsnm{Lindsay},~\bfnm{Bruce}\binits{B.}}
(\byear{2017}).
\btitle{Pseudo-kernel method in U-statistic variance estimation with large
  kernel size}.
\bjournal{Statistica Sinica}
\bpages{1155--1174}.
\end{barticle}
\endbibitem

\bibitem[\protect\citeauthoryear{Wang and Wei}{2022}]{wang2022quantifying}
\begin{barticle}[author]
\bauthor{\bsnm{Wang},~\bfnm{Qing}\binits{Q.}} \AND
  \bauthor{\bsnm{Wei},~\bfnm{Yujie}\binits{Y.}}
(\byear{2022}).
\btitle{Quantifying uncertainty of subsampling-based ensemble methods under a
  U-statistic framework}.
\bjournal{Journal of Statistical Computation and Simulation}
\bpages{1--21}.
\end{barticle}
\endbibitem

\bibitem[\protect\citeauthoryear{Yates and Grundy}{1953}]{yates1953selection}
\begin{barticle}[author]
\bauthor{\bsnm{Yates},~\bfnm{Frank}\binits{F.}} \AND
  \bauthor{\bsnm{Grundy},~\bfnm{P~Michael}\binits{P.~M.}}
(\byear{1953}).
\btitle{Selection Without Replacement from Within Strata with Probability
  Proportional to Size}.
\bjournal{Journal of the Royal Statistical Society: Series B (Methodological)}
\bvolume{15}
\bpages{253--261}.
\end{barticle}
\endbibitem

\bibitem[\protect\citeauthoryear{Zhou, Mentch and Hooker}{2021}]{lucas:new}
\begin{barticle}[author]
\bauthor{\bsnm{Zhou},~\bfnm{Zhengze}\binits{Z.}},
  \bauthor{\bsnm{Mentch},~\bfnm{Lucas}\binits{L.}} \AND
  \bauthor{\bsnm{Hooker},~\bfnm{Giles}\binits{G.}}
(\byear{2021}).
\btitle{V-statistics and Variance Estimation}.
\bjournal{Journal of Machine Learning Research}
\bvolume{22}
\bpages{1--48}.
\end{barticle}
\endbibitem

\end{thebibliography}

%%%%%%%%%%%%%%%%%%%%%%%%%%%
% Begin of appendix
%%%%%%%%%%%%%%%%%%%%%%%%%%%

% \documentclass[main.tex]{subfiles}

% This is the appendix (supplementary materials)
% \begin{document}
\newpage
\onecolumn

\section*{Supplementary Material (Appendices)}
\begin{appendices}

\addcontentsline{toc}{section}{Appendices} % Add the appendix text to the document TOC

\part{}

% \tableofcontents
{
\textbf{Table of Contents of Appendices}
  \hypersetup{linkcolor=blue}
  \parttoc % Insert the appendix TOC
  \clearpage
}

\section{Notations}
\label{sec:append:notation}

\begin{table}[H]
\small
    \centering
    % \begin{tabularx}{16cm}{l|X}
    \begin{tabular}{  m{3.2cm}   m{12cm}  } 
    \hline 
    Notations
        & Description
    \\ \hline \tabSpace
    
    $\calO$
        & $a = \calO(b)$: exists $C > 0$, s.t. $a \leq C b$. .
        \\ \tabSpace
    $\Omega$, $\asymp$
        & $a = \Omega(b) \Longleftrightarrow b = \calO(a)$. $a \asymp b \Longleftrightarrow a = \calO(b)$ and $a = \Omega(b)$ .
        \\ \tabSpace
    $U_n$, $h$
        & $U_n$ is the U-statistic with size-$k$ kernel $h$. 
        \\ \tabSpace
    $\vu$, $\psi$
        & $\vu$ denotes the estimator \eqref{eq:def:vu U-stat size-2k} of $\Var(U_n)$, which is a U-statistic  with size-$2k$ kernel $\psi$.\\ \tabSpace
    $S$
        & $S$ denotes the size-$k$ subsample set associated with kernel $h$. \\ \tabSpace
    $\Skk$    
        &  $\Skk$ denotes the size-$2k$ subsample set associated with kernel $\psi$. \\ \tabSpace
    $c, d_1, d_2$
        & Given $S_1, S_2 \subset \Skk_1, S_3, S_4 \subset \Skk_2$, $c = |\Skk_1 \cap \Skk_2|$, $d_1 = |S_1 \cap S_2|$, and $d_2 = |S_3 \cap S_4|$. \\ \tabSpace
    $\varphi_d, w_d$
        &  See  $\psi(\Skk) = \sum_{d = 0}^{k} w_d \varphi_d \left(\Skk \right)$ \eqref{eq:def:phi}. $\varphi_d (\Skk)$ is still a U-statistic. \\ \tabSpace
    $\check{w}_d$
        &$\check{w}_d = \calO({k^{2d}}/{(d! n^d)})$ is the upper bound of $w_d$ given by Equation \eqref{eq:weights for phi's: 2}. \\ \tabSpace
    $\xi_{d, k}^2$
    & $\xi_{d, k}^2 = \Cov[h(S_1), h(S_2)]$ is first used in \eqref{eq:Decomposition of Var(U_n) Hoeffding}.
    \\ \tabSpace
    $\sigma_{c, 2k}^2$
        & $\sigma_{c, 2k}^2 = \Cov[\psi (\Skk_1), \psi (\Skk_2)]$ is first used in \eqref{eq:Var(Vu) expression}. \\ \tabSpace
    $\etacd$
        & $\etacd$ is introduced by further decomposing $\sigma_{c, 2k}^2$ in \eqref{eq:def:eta}. \\ \tabSpace
    % $\psi^{(T_2)}$
    %    & $\psi^{(T_2)}$ is defined in \eqref{eq:def:psi, truncated}, the $T_2$-truncated $\psi$. \\ \hline
    % $\sigma_{c, 2k, (T_2)}^2$
    %     & $\sigma_{c, 2k}^2$ is defined used in \eqref{eq:def:truncated sigma_c^2}, the  $T_2$-truncated $\sigma_{c, 2k}^2$. \\ \hline
    $\widecheck{\sigma}_{c, 2k}^{2}$
        & $\widecheck{\sigma}_{c, 2k}^{2}$ is an upper bound of $\sigma_{c, 2k}^{2}$ given by Propositions \ref{lemma:bound each sigma_c^2 for finite c} and \ref{lemma:bound sigma_c^2 for any c}.
        \\ \tabSpace
    % $\Var^{(T_1)}(\vu)$
    %     & $\Var^{(T_1)}(\vu)$ is defined in \eqref{eq:def:Var(vu) T_1}, the $T_1$-truncated  $\Var^{(T_1)}$. \\ \hline
    % $\widecheck{\Var}^{(T_1)}(\vu)$
    %     & $\widecheck{\Var}^{(T_1)}(\vu)$ is the upper bound of $\Var^{(T_1)}(\vu)$, defined in Eq. \eqref{eq:def:check Var(vu) T_1}. \\
    %     \hline
    $\covh$
        &  $\covh := \Cov[h(S_1) h(S_2), h(S_3) h(S_4)]$  \eqref{eq:def:rho = Cov(h1h2, h3h4)}. 
        \\ \tabSpace
    \emph{DoF}
        & The number of free parameters to determine $\Cov[h(S_1) h(S_2), h(S_3) h(S_4)]$. \\ \tabSpace
    $\unr$, $|\unr|$
        & $\unr$ is a 9-dimensional vector defined in \eqref{eq:def:r vec}, describing the 4-way overlapping among $S_1, S_2, S_3, S_4$.
        $|\unr|$ is the $\ell_1$ vector norm of $\unr$.
        \\ \tabSpace
    $\rx{i}$, $\ry{j}$, $\unrMarg$
        &
         $\rx{i} = \sum_{j=0}^2 r_{ij}$, $\ry{j} = \sum_{i=0}^2 r_{ij}$, and 
        $\unrMarg = (\rx{0}, \rx{1}, \rx{2}, \ry{0}, \ry{1}, \ry{2})$ . \\ \tabSpace
    $\rhor$
        & $\rhor$  is the 9 \emph{DoF} representation of $\covh$  (see Assumption \ref{assumption:low dependency}). \\ \tabSpace
    $\fk_c$
        & $\fk_c$ \eqref{eq:F_c as upper bound} is the upper bound of $\covh$, given that $|\Skk_1 \cap \Skk_2| = c$. \\ \tabSpace
    $\covhh$
        & This is a notation emphasizing 11 \emph{DoF} of $\covh$ used in Appendix \ref{sec:proof under relaxed dependency}. \\ \tabSpace
    $\rhort$
    & $\rhort$  is the 9 \emph{DoF} benchmark of used in Assumption \ref{assumption:low dependency (relaxed)}. \\ \tabSpace
    \emph{Influential Overlaps}
        & The samples in $\Skk_1 \cap \Skk_2$. \\ \tabSpace
        \hline
    \end{tabular}
    \caption{Summary of Notations}
    \label{tab:summary of notations}
\end{table}

\section{Discussion of Assumptions}
\label{sec:append:assump}
% \documentclass[../main.tex]{subfiles}

% \begin{document}

In this section, we present discussion and validation examples for Assumption \ref{assumption:low dependency}-\ref{assumption:4th moment}, which are related to the covariance term $\rho = \Cov[h(S_1) h(S_2), h(S_3) h(S_4)]$ \eqref{eq:def:rho = Cov(h1h2, h3h4)}. 
Before the discussion, we first illustrate the definition of the 9-dimensional vector $\unr$, which plays an important role in quantifying $\rho$.

\subsection{Definition of \texorpdfstring{$\unr$}{r}}
\label{sec:append:r}

We present the definition of the 9-dimension vector $\unr$, which characterizes the overlaps between $\Skk_1, \Skk_2$ for $\covh$. 
This $\unr$ is used in Assumption \ref{assumption:low dependency}.

\begin{definition}[$\unr$]
First, we denote the samples in $\Skk_1 \cap \Skk_2$ as \emph{Influential Overlaps} of $\covh$ \eqref{eq:def:rho = Cov(h1h2, h3h4)}.

Secondly, given size-$2k$ subsample sets $\Skk_1, \Skk_2$, and size-$k$ subsample sets $S_1, S_2 \subset \Skk_1, S_3, S_4 \subset \Skk_2$, such that  $c = |\Skk_1 \cap  \Skk_2 |, d_1 = | S_1 \cap S_2 |, d_2 = | S_3 \cap S_4 |$. 
Denote $T_0 = S_1 \cap S_2$, $T_1 = S_1 \backslash S_2,  T_2 = S_2 \backslash S_1$, $T_0' = S_3 \cap S_4, T_1' = S_3 \backslash S_4$, and $T_2' = S_3 \backslash S_4$ (see Figure \ref{fig:set_relationship}).

Based on the above, we denote $R_{ij} := T_i \cap T_j'$, and $r_{ij} := |R_{ij}|$, for $i,j = 0, 1, 2$. Then, a 9-dimensional vector $\underline{r}$ is defined as follows: 
\begin{align}
\label{eq:def:r vec}
    \underline r := (r_{00}, r_{01}, r_{02}, r_{10}, r_{11},
    r_{12}, r_{20}, r_{21}, r_{22})^T.
\end{align}

Thirdly, we define the norm of $\unr$ as $|\underline{r}| = \sum_{i = 0}^{2} \sum_{j = 0}^{2} r_{ij}$. Note that each sample in $(S_1 \cup S_2 \cup S_3 \cup S_4) \cap (\Skk_1 \cap \Skk_2)$ is counted exactly once in $\unr$ so $|\unr| \leq c$.
\begin{figure}[H]
    \centering
    \includegraphics[scale = 0.5]{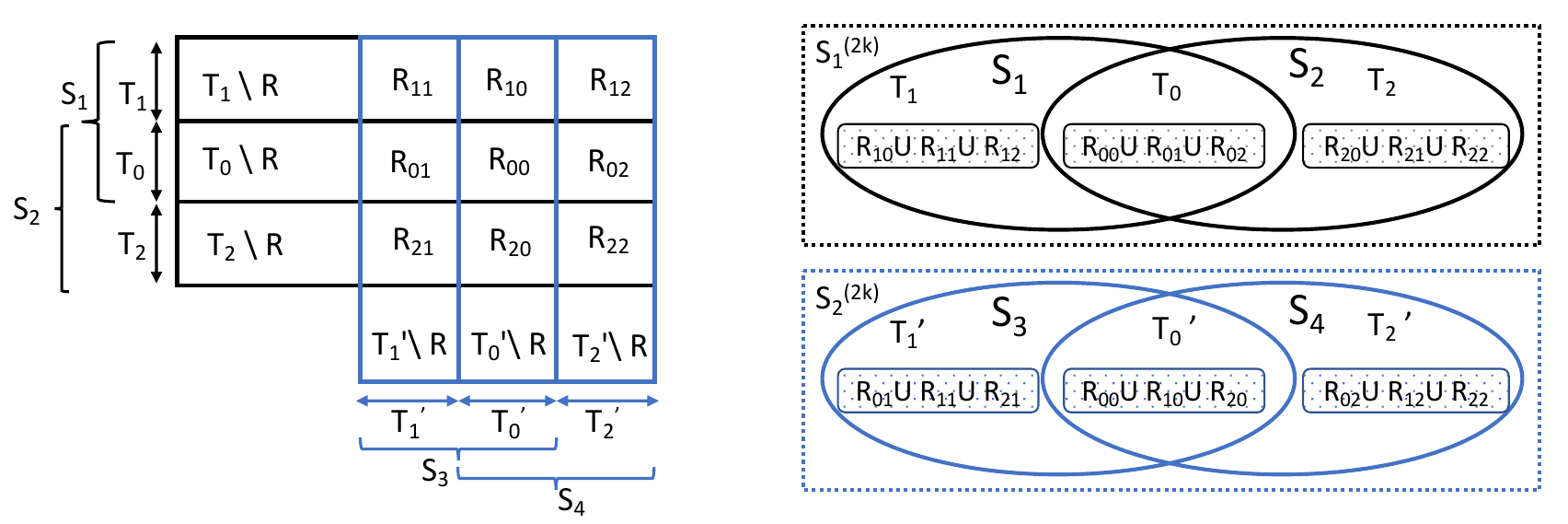}
    \caption{Illustration of the relationship of among $S_1, S_2, S_3, S_4, \Skk_1, \Skk_2$. Here, $R:= \Skk_1 \cap \Skk_2$.}
    \label{fig:set_relationship}
\end{figure}
\end{definition}

We remark that the value of $r_{ij}$ is naturally bounded by sample size in the corresponding overlapping set. For example, $\rx{0} = |(S_1 \cap S_2) \cap (S_3 \cup S_4)| \leq  |(S_1 \cap S_2)| = d_1$. 
Note that the 11 \emph{DoF} of $\covh$ can also be illustrated by the left panel of Figure \ref{fig:set_relationship}. There are 15 blocks but we have constraints $|S_1| = |S_2| = |S_3| = |S_4| = k$, so we get $11 = 15-4$ \emph{DoF}.

\subsection{Discussion of Assumption \ref{assumption:low dependency}}
\label{sec:assump:relax low dependency assumption}

Assumption \ref{assumption:low dependency} reduces \emph{DoF} of $\covh$ from 11 to 9, i.e., $\rho = \rhor$. It drops 2 ``within $\Skk_1, \Skk_2$ overlapping'' \emph{DoF} $d_1$ and $d_2$. 
The motivation is presented as follows. Given $c = |\Skk_1 \cap \Skk_2|$, $d_1$ and $d_2$ only describes the overlapping within $\Skk_1$ and $\Skk_2$. Hence, these two \emph{DoF} are expected to have smaller impact on $\covh = \Cov[h(S_1) h(S_2), h(S_3) h(S-4)]$ than the other 9 \emph{DoF}.
This assumption can be validated in the following examples of linear average kernel.

\begin{example}
\label{example:linear kernel for low depent. assump}
Suppose $h$ is a linear average kernel, $h(X_1, ..., X_k) = \frac{1}{k} \sum_{i=1}^{k} X_i$, where $X_1, ..., X_n$ $i.i.d.~ X$, where $\E X =0, \E X^2 = \mu_2,  \E X^3 = \mu_3, \E X^4 = \mu_4 $. 
Notice that 
\begin{align*}
    \E(X_1 X_2 X_3 X_4) = 
    \begin{cases} 
      \mu_2^2,    & \text{if}\, X_i = X_j,  X_k = X_l, X_i \neq X_k, \, \text{where} \,\{i, j, k, l\} = \{1, 2, 3, 4\}; \\
      \mu_4,     & \text{if}\, X_1 = X_2 = X_3 = X_4; \\
      0, & otherwise.
    \end{cases}
\end{align*}
Hence, we have
\begin{align*}
    \Cov(X_1 X_2, X_3 X_4) =
    \begin{cases} 
      cov_1 := \mu_2^2, & \text{if}\, X_1 = X_3 \neq X_2 = X_4;~ or~ X_1 = X_4 \neq   X_2 = X_3;\\
      cov_2 := \mu_4 - \mu_2^2, & \text{if}\, X_1 = X_2 = X_3 = X_4;\\
      cov_3 := 0, & otherwise.
    \end{cases}
\end{align*}
Given the above 3 cases of $\Cov(X_1 X_2, X_3 X_4) $, it is easy to verify that $\covh$ can be represented as a weighted average of $a_{1, n} cov_1 + a_{2, n} cov_2 + a_{3, n} cov_3$, where $cov_3$ is 0. 
Moreover, by the definition of $cov_1$ and $cov_2$, it is easy to verify that $a_{1, n}$ only depends on $r_{ij}$ for $(i, j) \neq (0, 0)$ and $a_{2, n}$ only depends on $r_{0, 0}$. 
This verifies Assumption \ref{assumption:low dependency} on this linear average kernel. 
Besides, we can also show that $\fk$ in \eqref{eq:F_c as upper bound} (see Assumption \ref{assumption:non-negative and ordinal}) is a quadratic function of $c$ for this kernel.
\end{example}

In addition to the above example, the following discussion shows that we may not be able to further reduce the \emph{DoF} from 9 to 4 by stronger assumptions. 
When $\E(h(S)) = 0$, it is natural to consider the following fourth cumulant of $\covh$:
\begin{align}
\label{eq:cumulants}
cum_4[h(S_1), h(S_2), h(S_3), h(S_4)] = &  \covh - 
        \Cov[h(S_1), h(S_3)] \Cov [h(S_2), h(S_4)] \nonumber \\
        &  - \Cov[h(S_1), h(S_4)] \Cov [h(S_2), h(S_3)].
\end{align}
If $cum_4[h(S_1), h(S_2), h(S_3), h(S_4))]$ in \eqref{eq:cumulants} is a lower order term of $\covh$, the \emph{DoF} can be reduced to 4, i.e., $|S_1 \cap S_3|$, $|S_1 \cap S_4|$, $|S_2 \cap S_3|$, and $|S_2 \cap S_4|$. %These 4 \emph{DoF} can also be represented by the 9-dimensional $\unr$.
However, Example \ref{example:linear average kernel: cumulant} shows that this does not hold even for a linear average kernel.

\begin{example}
\label{example:linear average kernel: cumulant}
Given size-$k$ sets $S_1, S_2, S_3, S_4$ s.t. $S_l = (X_1, Y_1^{(l)}, ..., Y_{k-1}^{(l)})$, $X_1, Y_j^{(l)}$ are i.i.d. $E(X_1) = 0, \Var(X_1) >0$, for  $j = 1, 2, .., k-1$ and $l = 1, 2, 3, 4$.
By \eqref{eq:cumulants} and some direct calculations, we have
$\covh 
    = \frac{\Var (X_1^2) }{k^4}$,
$\Cov[h(S_1), h(S_3)] 
    = \frac{\Var (X_1)}{k^2}.
$
Plugging in the above equations, we have 
\begin{align}
\label{eq:cumulant ratio}
    \frac{cum_4[h(X_1), h(X_2), h(X_3), h(X_4)]}{2 \Cov[h(S_1), h(S_3)]^2}
    = \frac{\Var (X_1^2) - 2\Var^2 (X_1)}{2\Var^2 (X_1)}.
\end{align}
As long as $\Var (X_1^2) - 2\Var^2 (X_1) > 0$, which is common for non-Gaussian $X_1$,  Equation \eqref{eq:cumulant ratio} is larger than $o(1)$. This implies that the fourth cumulant is not always a lower order of $\covh$.
\end{example} 
We can further verify that given the kernel function is simple quadratic average kernel $h(S_1) = h(X_1, .., X_k) = \frac{1}{k^2} [\sum_{i = 1}^k X_i]^2$, even if $X_i$'s are i.i.d. standard Gaussian, the fourth cumulant is still not a lower order term of $\covh$,.

\subsection{Discussion of Assumption \ref{assumption:non-negative and ordinal}}

In Equation \eqref{eq:F_c as upper bound},  $\fk_c$ is defined as an upper bound for $\covh$ for a given $c$. As illustrated by Figure \ref{fig:example of ordinal cov}: the more samples shared by $S_1$ and $S_2$, $h(S_1) h(S_2)$ becomes closer to $h(S')^2$. Therefore, given that $|\Skk_1 \cap \Skk_2| = c$, $F_c$ has the most overlapping among all $\covh$.
\begin{figure}[H]
    \centering
    \includegraphics[scale = 0.5]{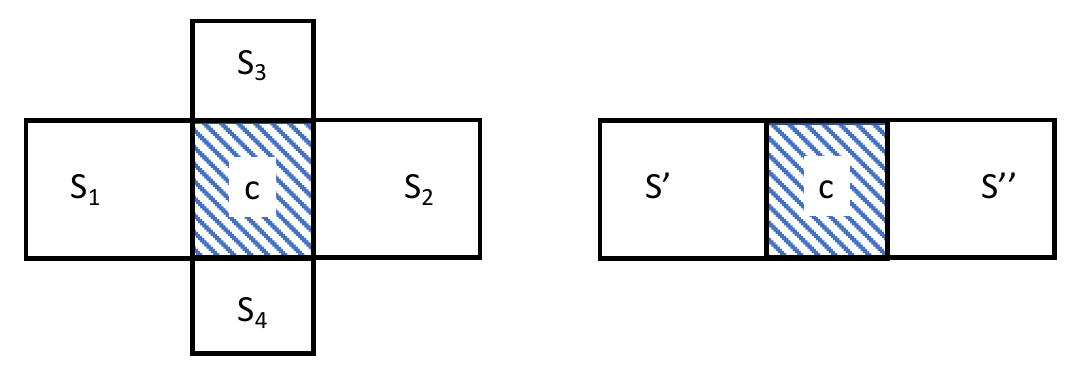}
    \caption{An example of ordinal covariance assumption}
    \label{fig:example of ordinal cov}
\end{figure}

This assumption is trivial when 
considering the linear average kernel again:  $h(S_l) = \frac{1}{k} \sum_{i = 1}^k X_i^{(l)}$, for $l =1, 2, 3, 4$. 
In particular, considering $(S_1, S_2, S_3, S_4)$ s.t. 
$$
|S_1 \cap S_2 \cap S_3 \cap S_4| = |(X_1^{(1)}, ..., X_c^{(1)})| =c
$$ 
and $(S', S'')$ s.t. $|S' \cap S''| = c$,  the equality in Equation \eqref{eq:F_c as upper bound} attains, i.e.,
\begin{align}
\label{eq:Fc eg}
    \covh 
        = k^{-4} \Var \left [(X_1^{(1)}+ ... + X_c^{(1)})^2 \right]
        =  \Cov \left[ h(S')^2, h(S'')^2 \right].
\end{align}
It is also straightforward to verify the assumption  under simple quadratic average kernel function  $h(S_1)= \frac{1}{k^2} [\sum_{i = 1}^k X_i]^2$

%%%%%%%%%%
\iffalse
\begin{remark}
    \item $|S_1 \cap S_2 \cap S_3 \cap S_4| = |(X_1, ..., X_c)| =c$ will give the largest $\covh$ under the constraint $| \Skk_1 \cap \Skk_2 | = |S_1 \cap S_2| = |S_3 \cap S_4| = c$.
    This assumption can also be verified for simple quadratic average kernel function  $h(S_1) = h(X_1, .., X_k) = \frac{1}{k^2} [\sum_{i = 1}^k X_i]^2$.
\end{remark}
\fi
%%%%%%%%%%
\subsection{Discussion of Assumption \ref{assumption:4th moment}}
\label{sec:append:assump:discuss Fc, counter example for linear rate}
Assumption \ref{assumption:2nd moment} shows a polynomial growth rate of the second moment term $\xi_{d, k}^2$ while Assumption \ref{assumption:4th moment}  shows a polynomial growth rate of the fourth moment term $\fk$.  Assumption  \ref{assumption:4th moment} also assumes $F_1^{(k)}$ is not a higher order term of $\xi_{1, k}^4$.
To better illustrate the idea of Assumption \ref{assumption:4th moment}, 
we consider the following example with an oversimplified setting: the $cum_4$ term is 0 in Equation \eqref{eq:cumulants}. Note that linear average kernel with i.i.d. standard Gaussian $X_i$'s satisfies this setting.
\begin{example}
\label{example:fourth moment under no cumulant}
Suppose there is no fourth order cumulant term in Equation \eqref{eq:cumulants}, by Equation \eqref{eq:Fc eg}, $F_c^{(k)}$ can be simplified as $\xi^2_{r_{22}, k} \xi^2_{r_{22}, k}+  \xi^2_{r_{22}, k} \xi^2_{r_{22}, k} = 2 \xi_{r_{22}, k}^4 = 2\xi_{c, k}^4$. 
This also implies \eqref{eq: F_1 upper bound}: $F_1^{(k)} / \xi_{1, k}^4 = \calO(1)$. 
We further remark that in this example, Assumption  \ref{assumption:4th moment} can be implied by Assumption $\ref{assumption:2nd moment}$. To be more specific, as demonstrated in the following equation, 
$a_2$ in \eqref{eq:fourth order growth} is $2a_1$, where $a_1$ is provided in Assumption \ref{assumption:2nd moment}. 
\begin{align}
\label{eq:fc, no cumulants}
    \frac{F_c^{(k)}}{F_1^{(k)}} = \frac{2\xi_{c, k}^4}{\xi_{1, k}^4} =  2\left( \frac{\xi_{c, k}^2}{\xi_{1, k}^2}\right)^2 = \calO(c^{2a_1}), \, \text{for} \, c = 1, 2, ..., k.
\end{align}
\end{example}

By Lemma \ref{lemma:bound sigma_c^2 for any c} (see Appendix \ref{sec:technical lemmas}), a natural upper bound for $\sigma_{c, 2k}^2$ is $F_c^{(k)}$ \eqref{eq:F_c as upper bound}. 
Hence, even if $\xi_{d, k}^2$ has a linear growth rate regarding $d$, $\fk$ still can growth at a quadratic rate of $c$ (see \eqref{eq:fc, no cumulants}).
Therefore, we cannot assume $\sigma_{2k, 2k}^2 / (2k \sigma_{1, 2k}^2) = \calO(1)$., which is the common assumption on the counterpart $\xi_{1, 2k}^2$ \citep{lucas:old, lucas:new, Romano2020}. That is one reason that Assumptions \ref{assumption:low dependency}-\ref{assumption:4th moment} are imposed on the primitive term $\covh$ instead of on $\sigma_{c, 2k}^2$.

% However, by Example \ref{example:fourth moment under no cumulant} in Appendix \ref{sec:append:assump}, under the  oversimplified settings, an expression for $F_c^{(k)}$ is $2\xi_{c, k}^4$.
% Thus, even though $\xi_{c, k}^2$ grows with $c$ at a linear rate, the upper bound of $\sigma_{c, 2k}^2$ can grow at a quadratic instead of the linear rate.
% Therefore, we need more primitive assumptions to analyze the order of $\sigma_{c, 2k}^2$.

% \end{document}

\section{Proof of Main Results}

\subsection{Proof of Corollary \ref{thm:ratio consistency}}
\label{sec:proof of main theorem, ratio consistency}

\begin{proof}[Proof of Corollary \ref{thm:ratio consistency}]
To show  $\frac{\vu}{E(\vu)} \xrightarrow{P} 1$ as $n \to \infty$, it suffice to show the $L_2$ convergence of $\vu / E(\vu)$, i.e., $\Var (\vu) /  \left( \E (\vu ) \right)^2\to 0$ as $n \to \infty$. 

By plugging Equation \eqref{eq:Var(Un) Theta bound} and \eqref{eq:Var(Vu) upper  bound} from Theorem \ref{thm:var of u_n and v_n}, we have
\begin{align}
\label{eq:main thm, step 1}
\frac{\Var (\vu)}{ \left( \E (\vu ) \right)^2} 
    =  \frac{\calO \left( \frac{k^2}{n} \check \sigma_{1, 2k}^2\right)}
    { \left[\left(1 + o(1)\right) \frac{k^2}{n} \xi_{1, k}^2\right]^2}
    = \calO\left( \frac{n}{k^2} \frac{\check \sigma_{1, 2k}^2}{\xi_{1, k}^4}\right),
\end{align}
where $\check \sigma_{1, 2k}^2 \asymp \frac{k^2}{n^2} \fk_1$ is the upper bound of $\sigma_{1, 2k}^2$ given by Proposition \ref{lemma:bound each sigma_c^2 for finite c}. By Assumption \ref{assumption:4th moment}, $\fk_1 = \calO ( \xi_{1,k}^4)$.
Plugging $\check \sigma_{1, 2k}^2$ into  Equation \eqref{eq:main thm, step 1}, we conclude that
\begin{align*}
\frac{\Var (\vu)}{ \left( \E (\vu ) \right)^2} = \calO\left(\frac{1}{n} \right).
\end{align*}
\end{proof}

\subsection{Proof of Theorem \ref{thm:var of u_n and v_n}}
\label{sec:proof of theorem Var Un Var Vu}

We first present a technical proposition to be used soon.  
\begin{proposition}
\label{prop:geometric ratio}
For any integer $c$, s.t. $1 \leq c \leq k$ and $k = o(\sqrt{n})$, 
\begin{align*}
    \binom{n}{k}^{-1} \binom{k}{c} \binom{n-k}{k-c} \leq \frac{1}{c!} \left( \frac{k^2}{n-k-1}\right)^c.
\end{align*}
\end{proposition}

\begin{proof}[Proof of Proposition \ref{prop:geometric ratio}]
This proof is provided by \citet{Romano2020}.
We first write the combinatorial numbers as factorial numbers
\begin{align}
    \binom{n}{k}^{-1} \binom{k}{c} \binom{n-k}{k-c} 
    &= \frac{(n-k)! k!}{n!} \frac{k!}{(k-c)! c!} \frac{(n-k)!}{(k-c)!(n-2k+c)!} \nonumber\\ 
    \label{eq:square bracket 0}
    &=\frac{1}{c!} \left[ \frac{k! k!}{(k-c)! (k-c)!} \right] \left[\frac{(n-k)!(n-k!)}{n! (n-2k+c)!}\right].
\end{align}
It suffices to upper bound things inside two square brackets separately. We have
\begin{align}
\label{eq:square bracket 1}
    \left[ \frac{k! k!}{(k-c)! (k-c)!} \right] \leq k^{2c},
\end{align}

\begin{align}
\label{eq:square bracket 2}
    \frac{(n-k)!(n-k!)}{n! (n-2k+c)!} 
    = \frac{(n-k) (n-k-1) ... (n-2k+c+1))}{n (n-1) ... (n-k+1)} \leq \left[\frac{1}{n-k+1}\right] ^c.
\end{align}

Combining \eqref{eq:square bracket 1} and \eqref{eq:square bracket 2} into \eqref{eq:square bracket 0}, we  completes the proof.
\end{proof}

\begin{proof}[Proof of Theorem \ref{thm:var of u_n and v_n}]

First, we show Equation \eqref{eq:Var(Un) Theta bound}. By Proposition \ref{thm:leading variance of U-stat} and Assumption \ref{assumption:2nd moment}, we can conclude that
\begin{align*}
    \lim_{n \to \infty} \frac{\Var (U_n)}{\frac{k^2}{n} \xi_{1, k}^2} = 1. 
\end{align*}

Secondly, we show Equation \eqref{eq:Var(Vu) upper  bound}.
\begin{align}
\label{eq:varV 1}
    \Var(\vu) \asymp \Var^{(T_1)}(\vu) 
    = &  \sum_{c=1}^{T_1} \binom{n}{2k}^{-1} \binom{2k}{c} \binom{n-2k}{2k-c} \sigma_{c, 2k}^2 \\
\label{eq:varV 2}
    = & \sum_{c=1}^{T_1} \binom{n}{2k}^{-1} \binom{2k}{c} \binom{n-2k}{2k-c} \calO(\frac{k^{2}}{n^2} F_c^{(k)}) \\
\label{eq:varV 3}
    = & \sum_{c=1}^{T_1}  \calO(\frac{k^{2c +2}}{n^{c+2}} F_c^{(k)}) \\
\label{eq:varV 4}
    = & \calO(\frac{k^4}{n^3} F_1^{(k)}).
\end{align}
Here, \eqref{eq:varV 1} is concluded by Lemma \ref{lemma:Finite truncated variance I}. 
\eqref{eq:varV 2} is concluded by $\sigma_{c, 2k}^2 = \calO(\frac{k^2}{n^2} F_c^{(k)})$ (Proposition \ref{lemma:bound each sigma_c^2 for finite c}).
\eqref{eq:varV 3} is concluded by  $\binom{n}{2k}^{-1} \binom{2k}{c} \binom{n-2k}{2k-c} = [1 + o(1)] \frac{k^{2c}}{n^c}$ for $c = 1, 2, ..., T_1$ (Proposition \ref{prop:geometric ratio}).
\eqref{eq:varV 4} is concluded by the bounded growth rate of $\fk$ in Assumption \ref{assumption:4th moment} and finite $T_1$.
If we denote $\frac{k^2}{n^2} F_1^{(k)}$ as $\widecheck{\sigma}_{1, 2k}^2$, we conclude that  $\Var(\vu) = \calO(\frac{k^2}{n} \widecheck{\sigma}_{1, 2k}^2)$.
\end{proof}

\subsection{Proof of Proposition \ref{thm:leading variance of U-stat}}
\label{sec:append:proof:thm:leading variance of U-stat}

\begin{proof}[Proof of Proposition \ref{thm:leading variance of U-stat}]

For $k = o(\sqrt{n})$, we want to show $\lim_{n \to \infty} \frac{\Var (U_n)}{\frac{k^2}{n} \xi_{1, k}^2} = 1$. 
First notice that for the coefficient leading term $\binom{n}{k}^{-1} \binom{k}{1} \binom{n-k}{k-1} \xi_{1, k}^2$, we have
\begin{align}
    \label{eq:leading term, asymp form}
    \frac
    {\binom{n}{k}^{-1} \binom{k}{1} \binom{n-k}{k-1}  }
    {\frac{k^2}{n} }
   =\frac{(n-k) !(n-k) !}{(n-1) !(n-2 k+1) !} \rightarrow 1,\, \text{as $n \rightarrow \infty$.}
\end{align} 
Therefore, it suffices to show that the rest part of $\Var(U_n)$ is dominated by the leading term:
\begin{align*}
    \lim_{n \to \infty} 
    \frac
        {\binom{n}{k}^{-1} \sum_{d=2}^{k} \binom{k}{d} \binom{n-k}{k-d} \xi^2_{d, k}}
        {\frac{k^2}{n} \xi_{1, k}^2} = 0.
\end{align*}
By Proposition \ref{prop:geometric ratio}, the numerator of the above can be bounded as 
\begin{align}
\label{eq:geometric sum: pre-form}
    \binom{n}{k}^{-1} \sum_{d=2}^{k} \binom{k}{d} \binom{n-k}{k-d} \xi^2_{d, k}
    \leq \sum_{d=2}^{k} \frac{k^{2 d}}{d !(n-k+1)^d} \xi_{d, k}^2: \sum_{d=2}^k (d!)^{-1} b_n^d \xi_{d, k}^2, 
\end{align}
where $b_n = \frac{k^2}{n-k-1}$.
Notice that $n < 2(n -k + 1)$, so we have
\begin{align}
    \label{eq:geometric sum}
    \frac
    {\binom{n}{k}^{-1} \sum_{d=2}^{k} \binom{k}{d} \binom{n-k}{k-d} \xi^2_{d, k}}
    {\frac{k^2}{n} \xi_{1, k}^2}
    \leq \frac{n}{n-k+1} \sum_{d=2}^{k} \frac{1}{d!} b_n^{d-1} \frac{\xi_{d,k}^2}{\xi_{1,k}^2}
    \leq \sum_{d=2}^{k} \frac{2}{d!} b_n^{d-1} \frac{\xi_{d,k}^2}{\xi_{1,k}^2}.
\end{align}
By Assumption \ref{assumption:2nd moment}, the growth rate of $\xi_{d, k}^2$ is bounded, there exists a uniform constant $C$ s.t. $\frac{\xi_{d, k}^2}{\xi_{1,k}^2} \leq C d!$ for $d = 2, 3, ..., k$. Therefore, the RHS of Equation \eqref{eq:geometric sum} is bounded as
\begin{align}
    \label{eq:geometric sum 2}
    \sum_{d=2}^{k} \frac{1}{d!} b_n^{d-1} \frac{\xi_{d,k}^2}{\xi_{1,k}^2}
    \leq \sum_{d=2}^{k} \frac{1}{d!} b_n^{d-1} \frac{C d! \xi_{1,k}^2}{\xi_{1,k}^2} 
    \leq C \sum_{d=2}^{k} b_n^{d-1}
    = C b_n \frac{1 - b_n^{k-1}}{1-b_n}
    \leq C  \frac{b_n}{1-b_n}.
\end{align}
The RHS of \eqref{eq:geometric sum 2} goes to 0 when $n \to \infty$, since $b_n = \frac{k^2}{n-k-1} \to 0$. 
This completes the proof. 

\end{proof}

\section{Proof of \emph{Double U-Statistics}}

\subsection{Proof of Proposition \ref{prop: double U-statistic: alternative form of phi}}
\label{sec:proof for double U}

\begin{proof}[Proof of Proposition \ref{prop: double U-statistic: alternative form of phi}]
\label{proof:prop:double U-statistic}

\noindent \textbf{Proof of Equation \eqref{eq:alternative form of phi}.}

We first show the following equation.
\begin{align}
    \psi\left(\Skk \right) 
\label{eq:alternative (primary) form of phi, 0 to k}
     = \sum_{d = 0}^{k} w_d \varphi_d \left(\Skk \right)
\end{align}
\citet{wang2014variance} have demonstrated that $\hat V_u$ is an U-statistic with size-2k kernel (Equation \eqref{eq:def:vu U-stat size-2k}):
\begin{align*}
    \hat{V}_{u}=Q(k)-Q(0)
    = \binom{n}{2k}^{-1} \sum_{\Skk \subseteq \mathcal{X}_{n}} \left[ \psi_{k}\left(\Skk \right)  -  \psi_{0}\left(\Skk \right) \right],
\end{align*}
\begin{align*}
&\psi_{k}\left( \Skk \right)
    = \underbrace{\binom{n}{2k} \binom{n}{k}^{-1} \binom{n}{k}^{-1}}_
    {A_1} 
    \sum_{d=0}^{k}  \frac{1}{N_d}
        \sum_{\substack{S_1, S_2 \subset \Skk \\ | S_1 \cap S_2 | = d}}  
            h\left(S_1\right) h\left(S_2\right),
\\
& \psi_{0}\left( \Skk \right)   
    = \underbrace{\binom{n}{2k} \binom{n}{k}^{-1} \binom{n-k}{k}^{-1}}_
    {A_{1, 0}}
    \frac{1}{N_0}
    \sum_{\substack{S_1, S_2 \subset \Skk \\  | S_1 \cap S_2 | = 0}}     
     h\left(S_1\right) h\left(S_2\right),
\end{align*}
where $N_d = \binom{n-2k+d}{d}$.

Denote 
\begin{align}
\label{eq:def A1 A0}
    A_1 := \binom{n}{2k} \binom{n}{k}^{-2}, \quad A_{1, 0} := \binom{n}{2k} \binom{n}{k}^{-1} \binom{n-k}{k}^{-1}.
\end{align}

Rewrite $\psi_{k}\left( \Skk \right) - \psi_{0}\left( \Skk \right)$ by the order of $d$. Notice that
there is a $\sum_{d=0}^{k}$ in $\psi_{k}\left( \Skk \right)$ but $d$ can only be $0$ in $\psi_{0}\left(\Skk\right)$. Hence, there is a cancellation for $h(S_1) h(S_2)$ s.t. $d = |S_1 \cap S_2| = 0$, thus we have
\begin{small}
\begin{align*}
\psi_{k}\left(\Skk\right) -  \psi_{0}\left(\Skk\right)
    = &A_1 \sum_{d=1}^{k} 
        \frac{1}{N_d}
        \sum_{\substack{S_1, S_2 \subset \Skk \\ | S_1 \cap S_2 | = d}}  
        h\left(S_1\right) h\left(S_2\right) 
    + (A_1 - A_{1, 0})
        \frac{1}{N_0}
        \sum_{\substack{S_1, S_2 \subset \Skk \\ | S_1 \cap S_2 | = 0}} 
        h\left(S_1\right) h\left(S_2\right).
\end{align*}
\end{small}
For the RHS of above equation, multiply and divide  $M_{d, k}$  \eqref{eq:def:phi} inside $\sum_{d=1}^k$:
\begin{align*}
\begin{split}
\psi_{k}\left(\Skk\right) -  \psi_{0}\left(\Skk\right)
    = &\sum_{d=1}^{k}
    \underbrace{
    \left[
    A_1 \frac{1}{N_d} M_{d, k} \right]}_
    {w_d}
    \underbrace{
    \frac{1}{M_{d, k}}
    \sum_{S_1, S_2 \subset \Skk, \, | S_1 \cap S_2 | = d} 
    h\left(S_1\right) h\left(S_2\right)}_
    {\varphi_d(\Skk)}
    \\
    &+ \underbrace{
    \left[
    (A_1 - A_{1, 0})\frac{1}{N_0}  M_{0, k}\right]}_
    {w_0}
    \underbrace{
    \frac{1}{M_{0, k}}
    \sum_{S_1, S_2 \subset \Skk, \, | S_1 \cap S_2 | = 0} 
     h\left(S_1\right) h\left(S_2\right)}_
     {\varphi_0 (\Skk)}.
\end{split}
\end{align*}
We denote 
\begin{align*}
    & w_d :=  \frac{A_1}{N_d} M_{d, k}, 
        \quad \text{for } d = 1, 2, ...,k; \\    
    &w_0 := \frac{(A_1 - A_{1, 0}) M_{0, k}}{N_0};   \\
    &\varphi_d(\Skk) :=\frac{1}{M_{d, k}}
    \sum_{\substack{S_1, S_2 \subset \Skk \\ | S_1 \cap S_2 | = d}} 
    h\left(S_1\right) h\left(S_2\right), 
        \quad \text{for } d = 1, 2, ..., k;\\
    &\varphi_0 (\Skk) := \frac{1}{M_{0, k}}\sum_{\substack{S_1, S_2 \subset \Skk \\ | S_1 \cap S_2 | = 0}} 
     h\left(S_1\right) h\left(S_2\right).
\end{align*} 
Thus we have
\begin{align*}
\psi_{k}\left(\Skk\right) -  \psi_{0}\left(\Skk\right)      
=  \left[ \sum_{d=1}^{k} w_d  \varphi_d (\Skk) \right] 
    + w_0 \varphi_0 (\Skk) 
  =   \sum_{d=0}^{k} w_d  \varphi_d (\Skk).
\end{align*}

Given that $\sum_{d = 0}^{k} w_d = 0$ is true (to be proved soon), then $w_0 = -\sum_{d = 1}^k w_d$. Therefore,
\begin{align*}
\begin{split}
    \psi\left(\Skk \right) 
    & = \sum_{d = 0}^{k} w_d \varphi_d \left(\Skk \right) 
     = \sum_{d = 1}^{k} w_d \varphi_d \left(\Skk \right) 
        - \left( \sum_{d=1}^k w_d \right) \varphi_0 \left(\Skk \right) \\
    & = \sum_{d = 1}^{k} w_d \left[\varphi_d \left(\Skk \right) - \varphi_0 \left(\Skk \right)\right].
\end{split}
\end{align*}    

\noindent \textbf{Proof of Equation  Equation \eqref{eq:weights for phi's: 2} }

First, we show that $\sum_{d = 0}^{k} w_d = 0$.
As discussed above, $w_d$ is a product of three normalization constants:
$A_1 = \binom{n}{2k} \binom{n}{k}^{-2}$ and $A_{1, 0} = \binom{n}{2k} \binom{n}{k}^{-1}\binom{n-k}{k}^{-1}$ are the normalization constant to rewrite $Q(k)$ and $Q(0)$ as a U-statistic;
$M_{d, k} := \binom{2k}{d} \binom{2k-d}{d} \binom{2k-2d}{k-d}$  is the number of pairs $S_1, S_2 \subset \Skk$ s.t. $|S_1 \cap S_2| = d$;
$N_d = \binom{n-2k+d}{d}$ is defined in Equation  \eqref{eq:psi: a size-2k kernel for Q(k')}.

\begin{align*}
w_0 = \frac{(A_1-A_{1, 0}) M_{0, k}}{N_0}; \,\, w_d = \frac{A_1 M_{d, k}}{N_d}, \quad \text{for } d= 1, 2, ..., k.
\end{align*}
Since $A_{1, 0} > A_1 > 0$, $M_{d, k} > 0$, $N_d > 0$, we have $
    w_d > 0, \forall d \geq 1$ and
    $w_0 < 0$. 
Then we show $\sum_{d = 0}^k w_d = 0$. Though this can be justified by direct calculation, we present a more intuitive proof. 
Recall $\hat V_u = Q(k) - Q(0)$. By the definition of $Q(k)$, $Q(k)$ can be represented as a weighted sum of $h(S_1) h(S_2)$, i.e., $\sum_{1 \leq i,j \leq n} a_{ij} h(S_i) h(S_j)$, where $\sum_{1 \leq i <j \leq n} a_{ij} = 1$.
Thus, $Q(k) - Q(0)$  can be represented in a similar way: 
\begin{align*}
    Q(k) - Q(0) = \sum_{1 \leq i < j \leq n} a_{ij}' h(S_i) h(S_j),
\end{align*}
where $\sum_{1 \leq i,j \leq n} a_{ij}' = 0$.
Therefore  $\psi_{k}\left(\Skk\right) -  \psi_{0}\left(\Skk\right)$, as the kernel of U-statistic $Q(k) - Q(0)$, can also be represented in the form of a weighted sum:
\begin{align}
\label{eq:psi as a weighted sum}
    \psi_{k}\left(\Skk\right) -  \psi_{0}\left(\Skk\right)
    = \sum_{1 \leq i <j \leq n} b_{ij} h(S_i) h(S_j),
\end{align}
where $\sum_{1 \leq i <j \leq n} b_{ij} = 0$ since $\psi_{k}\left(\Skk\right) -  \psi_{0}\left(\Skk\right)$ is an unbiased estimator of $Q(k) - Q(0)$.
On the other hand, for $d = 0, 1, 2, ..., k$, $\varphi_d (\Skk)$  is still a U-statistic, which can be represented in the form of a weighted sum:
\begin{align}
\label{eq:phi_d as a weighted sum}
    \varphi_d \left(\Skk\right) 
    = \sum_{1 \leq i <j \leq n} c_{ij}^{(d)} h(S_i) h(S_j),
\end{align}
where $\sum_{1 \leq i <j \leq n} c_{ij}^{(d)} = 1$.
Since $\psi_{k}\left(\Skk\right) -  \psi_{0}\left(\Skk\right) = \sum_{d = 0}^k w_d  \varphi_d \left(\Skk\right)$, by comparing Equation \eqref{eq:psi as a weighted sum} and \eqref{eq:phi_d as a weighted sum}, we have 
$\sum_{d = 1}^k w_d \sum_{1 \leq i <j \leq n} c_{ij}^{(d)} = \sum_{1 \leq i <j \leq n} b_{ij}$. Since $\sum_{1 \leq i <j \leq n} c_{ij}^{(d)} = 1$ and $\sum_{1 \leq i <j \leq n} b_{ij} = 0$, we can take $h(S_i) = 1$ for $i = 1, 2, ..., n$ and conclude that 
\begin{align*}
    \sum_{d = 0}^{k} w_d = 0.
\end{align*}

Secondly, we present the details to bound $w_d = A_1 M_{d, k} / N_d$, for $d = 1, 2, ..., k$. Plug in the expression of $A_1, M_{d, k}, N_d$, we have
\begin{small}
\begin{align*}
w_d =& \left[ \binom{n}{2k} \binom{n}{k}^{-2} \right]
        \left[\binom{2k}{d} \binom{2k-d}{d} \binom{2k-2d}{k-d} \right]
        / \binom{n-2k+d}{d} 
    \\
    =& \left[ 
            \frac{n!}{(n-2k)! (2k)!} 
            \frac{(n-k)!(n-k)!k!k!}{n! n!}\right]
        \left[
            \frac{(2k)!(2k-d)!(2k-2d)!}{(2k-d)! d! (2k-2d)! d! (k-d)! (k-d)!}
        \right] 
       \left[ \frac{d!(n-2k)!}{(n-2k+d)!}
        \right].
\end{align*}
\end{small}
After direct cancellation of the same factorials, we have
\begin{align}
\label{eq:w_d 3 parts}
w_d 
    = \underbrace{\frac
    {(n-k)!(n-k)! }
    { n! (n-2k+d)! }}_
    {\text{Part I}}
    \underbrace{\frac
    {k! k!}
    {(k-d)! (k-d)!}}_
    {\text{part II}}
    \underbrace{
    \frac{1}{d!}}_
    {\text{part III}}.
\end{align}
For Part I in \eqref{eq:w_d 3 parts},
\begin{align*}
\frac{(n-k)!(n-k)! }{ n! (n-2k+d)! }
     %= \frac{(n-k) (n-k-1) ... (n-k+d +1)}{n (n-1) ... (n-k+1)} 
     = \frac{ \prod_{i = 0}^{k-d -1} (n-k-i)}
        {\prod_{i = 0}^{k-1} (n-i)}
     = [1 + o(1)] \frac{1}{n^d}.
\end{align*}
The last equality is because for any $k = o(\sqrt{n}), d \leq k$, we have
\begin{align*}
\frac{ \prod_{i = 0}^{k-d -1} (n-k-i)}
    {\prod_{i = 0}^{k-1} (n-i)} 
    \leq  \frac{ \prod_{i = 0}^{k-d -1} (n-i)}
    {\prod_{i = 0}^{k-1} (n-i)}
    =  \frac{ 1 }
    {\prod_{i = 0}^{d-1} (n-k+d -i)}
    \leq \frac{1}{(n-k)^d}
    = [1 + o(1)] \frac{1}{n^d}.
\end{align*}
On the other hand,
\begin{align*}
\frac{ \prod_{i = 0}^{k-d -1} (n-k-i)}
    {\prod_{i = 0}^{k-1} (n-i)} 
    \geq 
     \frac{1}{n^d}.
\end{align*}
Combining $\leq$ and $\geq$, we have $= [1 + o(1)] \frac{1}{n^d}$.
For Part II \eqref{eq:w_d 3 parts},
\begin{align*}
\frac{k! k!}{(k-d)! (k-d)!}
    = [k(k-1) ... (k-d+1)]^2] \leq k^{2d}.
\end{align*}
Particularly, when $d$ is fixed, we have $\frac{k! k!}{(k-d)! (k-d)!}
    = k(k-1) ... (k-d+1)]^2 = [1 + o(1)] k^{2d}$.
Combining Part I, II, III in \eqref{eq:w_d 3 parts}, we have 
\begin{align*}
w_d  = 
   \begin{cases}
     [1 + o(1)] \left [\frac{1}{d!}(\frac{k^2}{n})^d \right] 
        & \forall \text{ finite $d$;}
        \\
     \calO \left [\frac{1}{d!}(\frac{k^2}{n})^d \right]
    & \forall d = 1, 2, ..., k.
   \end{cases}
\end{align*}

\end{proof}

% \documentclass[../main.tex]{subfiles}

% \begin{document}

\section{Proof of Results in Methodology Section}
\label{sec:append:methoddology}

%%%%%%%%%%%%%%%%%%%%%%
\subsection{Variance of Incomplete U-statistics \texorpdfstring{$\UincMatch$}{Umatch}}
\label{sec:append:varUmatch}

\begin{proof}[Proof of Proposition \ref{prop:var incU, match}]
\label{proof:proof of inc var}
This is an extension of the results by \citet[Section 4.1.1]{wang2012investigation} and \citet{wang2014variance}.

Comparing $\Var(\UincMatch) = ( 1-\frac{1}{B}) Var(U_n) +  \frac{1}{MB} \vh$ \eqref{eq:var Uinc, partition} with $ \Var(\Uinc) 
    = \varU + \E[ \Var(\Uinc | \calX_n) ]$  \eqref{eq:var Uinc general} , it suffices to show that 
\begin{align*}
    \E\left[ \Var(\Uinc | \calX_n)  \right] = \frac{1}{MB} \vh -  \frac{1}{B} \varU.
\end{align*}

Here we adopt an alternative view of a complete U-statistic $U_n$ with $k \leq n/2$ by \citet{wang2014variance}. 
Follow our notation of ``matched group'',  we can always take $M = \floor{n/k}$ mutually disjoint subsamples $S_1, ..., S_M$ from $(X_1, ..., X_n)$, such that  $|S_i \cap S_j| = 0$ for $1 \leq i < j \leq M$. 
\citet{wang2014variance} take integer $M = n/k$ while we allow $2 \leq M \leq \floor{n/k}$.
Recall such $(S_1^{(b)}, ..., S_M^{(b)})$ as a ``matched group'', where $b$ is the index of group. 
Let $\Gn$ be the collection of all such matched groups constructed from $n$ samples, i.e., 
\begin{align}
\label{eq:def:Gn, group}
    \Gn = \big\{ (S_1^{(b)}, \ldots, S_M^{(b)}) : \cup_j S_j^{(b)} \subset {\cal X}_n, \,\, \text{and} \,\, S_i^{(b)} \cap S_j^{(b)}  = \emptyset, \forall 1 \leq i, j \leq M \big\}.
\end{align}
Then, an alternative representation of $U_n$ is
\begin{align}
\label{eq:U_n alternative}
    U_n = \frac{1}{M |\Gn|} \sum_{b = 1}^{|\Gn|} \sum_{i=1}^M S_i^{(b)}.
\end{align}
This form seems redundant because there are some replicate subsample among all $S_i^{(b)}$'s. However, for incomplete U-statistic $\UincMatch$, each $(S_1^{(b)}, ..., S_M^{(b)})$ can be viewed as a sample from $\Gn$. Hence,  \citet{wang2012investigation} show that
$B \cdot \Var(\Uinc | \calX_n) = \frac{1}{|\Gn|} \sum_{b=1}^{|\Gn|} (\bar{h}^{(b)} - U_n)^2 
$, where $\bar{h}^{(b)} = \frac{1}{M} \sum_{i=1}^M h(S_i^{(b)})$, $S_i^{(b)}$'s are all subsamples associated with the complete U-statistic $U_n$ on $\calX$. However, \citet{wang2012investigation} and \citet{wang2014variance} do not provide a simple expression in the form of $\vh$ and $\varU$. 
We further simplify $B \Var(\Uinc | \calX_n)$ as follows,
\begin{align*}
B \cdot \E\left[ \Var(\Uinc | \calX_n)  \right] 
    =& \E \left (\frac{1}{|\Gn|} \sum_{b=1}^{|\Gn|} (\bar{h}^{(b)} - U_n)^2 \right) \\
    =& \E \left [\frac{1}{|\Gn|} \sum_{b=1}^{|\Gn|} \left( (\bar{h}^{(b)} -\E(U_n)) -  (U_n - \E(U_n)) \right)^2 \right] \\
    =& \frac{1}{|\Gn|} \left[ \sum_{b=1}^{|\Gn|} \E \left (  \left(\bar{h}^{(b)} -\E(U_n) \right)^2 \right) \right]
    + \frac{1}{|\Gn|} \left[\sum_{b=1}^{|\Gn|} \E  \left( \left(U_n - \E(U_n) \right)^2 \right)\right] \\
    &
     - 2 \E \left( \frac{1}{|\Gn|} \sum_{b=1}^{|\Gn|} (\bar{h}^{(b)} - \E(U_n))(U_n - \E(U_n) \right)
    \\
    =& \Var\left(\bar{h}^{(1)}\right) - \Var(U_n) \\
    =& \frac{1}{M} \vh - \varU.
\end{align*}
In the above equations, the first equality is the conclusion by \citet{wang2012investigation}; the next-to-last equality holds since  $\frac{1}{|\Gn|} \sum_{b=1}^{|\Gn|} \bar{h}^{(b)} = U_n$; the last equality holds since $h(S_1^{(1)}), ..., h(S_M^{(1)})$ are independent.

\end{proof}

%%%%%%%%%%%%%%%
\subsection{Unbiasedness of Variance Estimators}

\begin{proof}[Proof of Proposition \ref{prop:var_subTree}]
First, we restrict the discussion given $k \leq n/2$. 
We first show that $\hatvs = \sum_{d= 0}^k \gammad \hatxidtilde$. 
By the discussion in Section \ref{sec::est joint}, we have $N_{d,k,n} = \binom{n}{k}^2 \gammad$. 
For a complete U-statistic with $k \leq n/2$, $\Nd = \binom{n}{k} \binom{n-k}{k-d} \binom{k}{d}$ and we denote $\NN = \binom{n}{k} > 0$. Then,
\begin{align*}
\sum_{d= 0}^k \gammad \hatxidtilde 
    =& \sum_{d=0}^k \NN^{-2} \Nd \hatxidtilde \\
    =&  \NN^{-2}  \sum_{1 \leq i \leq \NN} \sum_{1 \leq j \leq \NN} \sum_{d = 0}^k 
    \mathbbm{1}\{ {|S_i \cap S_j| = d} \} [h(S_i) - h(S_j)]^2 /2 \\
    =&  \NN^{-2} \left[  2 \sum_{1 \leq i < j \leq \NN} \sum_{d = 0}^{k-1}
    \mathbbm{1}\{ {|S_i \cap S_j| = d} \} [h(S_i) - h(S_j)]^2 /2 \right]\\
    =&  \frac{\NN (\NN-1)}{\NN^2} \left[\binom{\NN}{2}^{-1} \sum_{1 \leq i < j \leq \NN}  [h(S_i) - h(S_j)]^2 /2 \right]\\
    =&  \frac{\NN-1}{\NN} \left[ \frac{1}{\NN-1}  \sum_{i=1}^\NN [h(S_i) - U_n]^2 \right] \\
    =& \frac{1}{\NN}  \sum_{i=1}^\NN [h(S_i) - U_n]^2 = \hatvs
\end{align*}
Here, the second equality holds by plugging in the definition of $\Nd$  and interchanging the finite summation $\sum_{d \in \calD}$ with $\sum_{1 \leq i < j \leq B}$. The third equality omits the cases with $i =j$, where $h(S_i) - h(S_j) = 0$.
The second to last equality holds because the sample variance is essentially an order-2 U-statistic, with kernel $(h(S_i) - h(S_j))^2 / 2$.

Then, as we demonstrated in Section \ref{sec::est joint},  $\E(\hatxidtilde) = \xidtilde$ for $d= 0, 1, ..., k$. Hence we conclude that 
\begin{align*}
    \E\left( \hatvs \right) =  \sum_{d= 0}^k \gammad \E \left(  \hatxidtilde  \right) 
    = \sum_{d= 0}^k \gammad \xidtilde = \vs.
\end{align*}

Secondly, we extend the previous argument to the setting $n/2 < k < n$.
We denote  $\calD = \{d \in N^* | 0 \leq d \leq k, \gammad > 0\}$.
We can define $\hatvs = \sum_{d \in \calD}\gammad \hatxidtilde$. We want to show that
\begin{align}
\label{eq:sum gamma xitilde -> sample variance 2}
      \hatvs = \frac{1}{\NN}  \sum_{i=1}^\NN [h(S_i) - U_n]^2, \\
      \E (\hatvs) = \vs.
\end{align}

Similar to previous proof
\begin{align*}
\hat V^{(S)'} 
    =& \sum_{d \in \calD} \NN^{-2} \Nd \hatxidtilde \\
    =&  \NN^{-2} \left[  2 \sum_{1 \leq i < j \leq \NN} \sum_{d \in \calD \backslash \{k\}}
    \mathbbm{1}\{ {|S_i \cap S_j| = d} \} [h(S_i) - h(S_j)]^2 /2 \right]\\
    =&  \frac{\NN (\NN-1))}{\NN^2} \left[\binom{\NN}{2}^{-1} \sum_{1 \leq i < j \leq \NN}  [h(S_i) - h(S_j)]^2 /2 \right]\\
    =&  \frac{\NN-1}{\NN} \left[ \frac{1}{\NN-1}  \sum_{i=1}^\NN [h(S_i) - U_n]^2 \right] \\
    =& \frac{1}{\NN}  \sum_{i=1}^\NN [h(S_i) - U_n]^2
\end{align*}

Since each $\hatxidtilde$ is still an unbiased estimator of $\xidtilde$, similarly, we have $
\E (\hat V^{(S)'} ) = \vs$. 
Remark that the summation in Equation \eqref{eq:sum gamma xitilde -> sample variance 2} is over $d \in \calD$ instead of $d = 0, 1, 2, ..., n$.
This is because $\gammad$ is 0 for small $d$, given $k > n/2$.
In other words, when $k > n/2$, several terms of $\gammad \xid$ in the Hoeffding decomposition \eqref{eq:Decomposition of Var(U_n) Hoeffding} is already 0.

\end{proof}

\begin{proof}[Proof of Proposition \ref{prop:bias of var subsamp}]

Since a sample variance is an order-2 U-statistics, 
\begin{align*}
\hatvsB 
=&  [(BM-1) BM]^{-1}  \sum_{i=1}^M \sum_{j=1}^{B} \sum_{(i', j') \neq (i, j)} \left[\hij - \hji \right]^2/2 \\
=& [(BM-1) BM]^{-1} \sum_{i=1}^M \sum_{j=1}^{B}
    \left( \sum_{(i', j') \in \calA(i,j)} +   \sum_{(i', j') \in \calB(i,j)} \right)
    \left[\hij - \hji \right]^2/2 \\
:=&  [(BM-1) BM]^{-1} \sum_{i=1}^M \sum_{j=1}^{B}
    \left( \sum_{(i', j') \in \calA(i,j)} +   \sum_{(i', j') \in \calB(i,j)} \right) \vijij
    ,
\end{align*}
where $\calA(i,j) = \{ (i', j')| i' \neq i, j' = j, 1 \leq  i' \leq M, 1\leq  j' \leq B \}$; 
$\calB(i,j) = \{ (i', j')| j' \neq j, 1\leq i' \leq M, 1\leq  j' \leq B \}$. 
%$(i', j') \neq (i, j)$ means that $i=i'$ and $j = j'$ do not hold simultaneously.
We note that $|\calA(i,j)| = M-1$ and  $|\calB(i, j)| = [(B-1)(M-1)]$ for any $(i, j)$.
To further simply our notation, we also denote $\vijij = [\hij - \hji ]^2/2$.

Fixing $(i, j)$, for any $(i', j') \in \calA(i, j)$, $S_{i'}^{(j')}$ and $S_{i}^{(j)}$ are the same  $j$ but not identical. Hence, $\vijij$  is an unbiased estimator of $\hatvh$. 
Furthermore,  the sample variance within group $j$ is also a U-statistic, which can be alternatively represented as an order-2 U-statistic:
$
 [M(M-1)]^{-1} \sum_{i=1}^M 
 \sum_{i'\neq i} \vijij.
$
Thus, by summation over all $j$ and the symmetry, we have
\begin{align}
\label{eq:vijij1}
& [(BM-1) BM]^{-1} \sum_{i=1}^M \sum_{j=1}^{B}
 \sum_{(i',j') \in \calA (i, j)} \E(\vijij)  \nonumber \\
=& [(BM-1) BM]^{-1} \sum_{i=1}^M \sum_{j=1}^{B} |\calA(i,j)| \vh
%= \frac{M(M-1)}{(BM-1)BM} \sum_{j = 1}^B [M(M-1)]^{-1} \sum_{i=1}^M  \sum_{i'\neq i} \vijij
 := \delta_{M, B} \vh,
\end{align}
where $\delta_{M, B} = \frac{M-1}{MB-1}$.

Fixing $(i, j)$, for any $(i', j') \in \calB(i, j)$,  $S_{i}^{(j)}$ and $S_{i'}^{(j')}$ are in different matched group. Since each matched group are sampled independently,  $S_{i}^{(j)}$ and $S_{i'}^{(j')}$ are independently sampled from $\calX_n$. 
By the theory of finite population sampling \citep{cochran1977sampling}, for $(i', j') \in \calB(i, j)$,
\begin{align*}
    \E(\vijij) =  \E[ \E(\vijij | \calX_n)] = \E\left[ \binom{n}{k}^{-1} \sum_{S_i \in \calX_n} (h(S_i) - U_n)^2 \right] = \vs
\end{align*}
Thus, the normalized summation over all such $\vijij$ satisfies that
\begin{align}
\label{eq:vijij2}
 & [[(BM-1) BM]^{-1} \sum_{i=1}^M \sum_{j=1}^{B}
 \sum_{(i',j') \in \calB (i, j)} E(\vijij) \nonumber \\
=& [(BM-1) BM]^{-1} \sum_{i=1}^M \sum_{j=1}^{B} |\calB(i,j)| \vs
 := (1 - \delta_{M, B}) \vs.
\end{align}

Combining Equations \eqref{eq:vijij1} and \eqref{eq:vijij2}, we conclude that
\begin{align*}
        \E \left( \hatvsB \right) = (1 - \delta_{M, B}) \vs + \delta_{M, B} \vh.
\end{align*}

\end{proof}

\begin{proof}[Proof of Proposition \ref{prop:inc var est unbiase}]

On one hand, by Proposition \ref{prop:bias of var subsamp}
\begin{align*}
    \E\left(\widehat{\Var}(\UincMatch) \right) 
    =& \E(\hatvhB) - \frac{MB-1}{MB} \E (\hatvsB)   \\
    =& \vh -   \frac{MB-1}{MB}\left[(1 - \frac{M-1}{MB-1}) \vs + \frac{M-1}{MB-1} \vh \right]    \\
    =&  \frac{MB - M + 1}{MB} \vh - \frac{B-1}{B} \vs.
\end{align*}

On the other hand, by $\varU = \vh - \vs$ \eqref{eq:varU = varh - sum} and Proposition and \ref{prop:var incU, match},
\begin{align*}
    \Var(\UincMatch) =  \frac{B-1}{B}  Var(U_n) +  \frac{1}{MB} \vh =   \frac{MB - M + 1}{MB} \vh - \frac{B-1}{B} \vs.
\end{align*}

Hence, we conclude the unbiasedness of our incomplete variance estimator:
\begin{align*}
   \E\left(\widehat{\Var}(\UincMatch) \right)  = \Var(\UincMatch) .
\end{align*}

\end{proof}

%%%%%%%%%%%%%%%%%%%%%%%%%%

\subsection{Equivalence of complete Variance Estimators}
\label{sec:append:equivalence,complete}

We denote the complete variance estimator by us, \citet{schucany1989small}, and \citet{wang2014variance} as $\vu$ \eqref{eq:our est, complete}, $\vu^{(S \&B)}$, and $\vu^{(W\&L)}$ respectively.

First, our complete U-statistic variance estimator is identical to the estimator in page 79 of \citet{folsom1984probability}'s work.

Secondly, we restrict $k = 2$ and show that $\vu = \vu^{(S \&B)}$. \citet{schucany1989small} estimate two terms, $\xi_{1,2}^2$ and $\xi_{2,2}^2$ in the Hoeffding decomposition as $\tilde \zeta_1^2$ and $\tilde \zeta_{2,2}^2$ respectively as follows. We adapt their notation to simplify $h(X_i, X_j)$ as $h_{ij}$.
\begin{align*}
  &\tilde \zeta_{1, 2}^2 = \binom{n}{3}^{-1} \sum_{i < j < l} h_0^*(X_i, X_j, X_l) - \binom{n}{4}^{-1} \sum_{i < j < l < m} h_1^*(X_i, X_j, X_l, X_m), \\
  &\tilde \zeta_{2, 2}^2 = \binom{n}{4}^{-1} \sum_{i < j < l < m} g^*(X_i, X_j, X_l, X_m), 
\end{align*}
where 
\begin{align*}
    & h_0^*(X_i, X_j, X_l) = \frac{1}{3} \left[h_{ij} h_{il} + h_{ij} h_{jl} + h_{il} h_{jl} \right]\\
    & h_1^* (X_i, X_j, X_l, X_m) = \frac{1}{3} \left[h_{ij} h_{lm} + h_{il} h_{jm} + h_{im} h_{jl} \right],\\
    & g^*(X_i, X_j, X_l, X_m) = \frac{1}{6} \left[ (h_{ij} - h_{lm})^2 + (h_{il} -h_{jm})^2 + (h_{im} - h_{jl})^2\right].
\end{align*}
Then, by estimating corresponding terms in the Hoeffding decomposition \eqref{eq:Decomposition of Var(U_n) Hoeffding}, 
\begin{align*}
    \vu^{(S \&B)} = \binom{n}{2}^{-1} \binom{2}{1} \binom{n-1}{1} \tilde \zeta_{1, 2}^2 + \binom{n}{2}^{-1} \tilde \zeta_{2, 2}^2.
\end{align*}

By our proposed decomposition, $\Var(U_n) = \vh - \vs$ and Proposition \ref{prop:var_subTree}, our estimation approach is equivalent to estimate $\xi_{1, 2}^2$ and $\xi_{2, 2}^2$ by $\hatvh - {\hat{\tilde\xi}_{1, 2}^2}$ and $\hatvh$ respectively. When $k = 2$,  $\hatvh = (\binom{n}{2} \binom{n-2}{2})^{-1} \sum_{ |S_i \cap S_j| = 0} \left(h(S_i) - h(S_j)\right)^2 / 2$ and ${\hat{\tilde\xi}_{1, 2}^2} = (2\binom{n}{2}  \binom{n-2}{1})^{-1}\sum_{ |S_i \cap S_j| = 1  } \left(h(S_i) - h(S_j)\right)^2 / 2$.
Hence, to show that $\vu = \vu^{(S \&B)}$, it suffices to show that $\hatvh -  {\hat{\tilde\xi}_{1, 2}^2} = \tilde{\zeta_{1}^2}$ and $\hatvh = \tilde{\zeta_{2}^2}$ respectively. 
For the first equality, we can simplify these terms as follows.
\begin{align*}
\small
    &\tilde \zeta_{1, 2}^2 = 
        \left(3 \binom{n}{3} \right)^{-1} \sum_{i=1}^n \sum_{j: j>i.} \sum_{l: l\neq i, j.} h_{ij} h_{jl} 
        - \left(\binom{n}{2} \binom{n-2}{2} \right)^{-1} \sum_{i=1}^n \sum_{j: j >i.} \sum_{l: l\neq i, j.} \sum_{m: m \neq i, j; m>l.} h_{ij} h_{lm} 
    , \\
    &\hatvh = \binom{n}{2} ^{-1} \sum_{i=1}^n \sum_{j: j>i.} h_{ij}^2 - \left(\binom{n}{2} \binom{n-2}{2} \right)^{-1} \sum_{i=1}^n \sum_{j: j >i} \sum_{l: l\neq i, j.} \sum_{m: m \neq i, j; m>l.} h_{ij} h_{lm} 
    , \\
    &\hat{\tilde\xi}_{1, 2}^2 = 
        \binom{n}{2} ^{-1} \sum_{i=1}^n \sum_{j: j>i.} h_{ij}^2 
        - \left(3 \binom{n}{3} \right)^{-1} \sum_{i=1}^n \sum_{j: j>i.} \sum_{l: l\neq i, j.} h_{ij} h_{jl}.
\end{align*}
Therefore, $\tilde \zeta_{1, 2}^2 = \hatvh -  \hat{\tilde\xi}_{1, 2}^2$.
For the latter equality, $\tilde \zeta_{2}^2$'s kernel $g^*$ is composed by sample variance between two kernels with disjoint subsamples, such as $h_{ij}$ and $h_{lm}$, $\tilde \zeta_{2, 2}^2$ is a redundant version of our $\hatvh$, which implies that $\hatvh = \tilde \zeta_{2, 2}^2$. This, concludes the equivalence. 

We remark that \citet{schucany1989small} consider an alternative estimator of $\xi_{1, 2}^2$, denoted as $\hat \zeta_1^2$ \citep{sen1960leading}. However, that one shows connection to the work of \citet{lucas:old} and \citet{lucas:new}  (see Section \ref{sec:connection}) but is not the focus of this appendix.

Thirdly, note that \citet{wang2014variance}'s estimator involves the definition of their partitioning scheme and we use the notation of our matching group and assume $M = n / k$ to present their estimator (see $\Gn$ \eqref{eq:def:Gn, group} and in Appendix \ref{sec:append:varUmatch}).
Here we present \citet{wang2014variance}'s estimator in their ANOVA form, which is the alternative to their second-moment view. This alternative form uses the within and between-variances of the groups \citep[see][page 1122]{wang2014variance}. However, the form is still different from ours. 
To simplify the notation,  we denote $\BB = |\Gn|$ and $\phibbar = \frac{1}{M} \sum_{i = 1}^M \phiib$, %$\phibar = \frac{1}{M\BB} \sum_{b=1}^\BB \sum_{i=1}^M \phiib$ 
% \phibar is Lindsay's U_n.
for  $b = 1, 2, ..., \BB$. 
Under this notation, the alternative form (with $\Gn$)  of $U_n$ \eqref{eq:U_n alternative} is $\frac{1}{M\BB}\sum_{i=1}^{M} \sum_{b=1}^\BB \hib$.
Then our estimator, $\vu$, and \citet{wang2014variance}'s estimator, $\vu^{(W\&L)}$, can be represented as follows.
\begin{align*}
    & \vu = \hatvh - \hatvs, \\
    & \vu^{(W\&L)} = \sigma_{WP}^2 / M - \sigma_{BP}^2,
\end{align*}
where $\hatvs = \binom{n}{k}^{-1} \sum_{i=1}^{\binom{n}{k}} (h(S_i) - U_n)^2$ \eqref{eq:hat vs}; $\hatvh = \binom{n}{k} \binom{n-k}{d} \sum_{|S_i \cap S_j| = 0} [h(S_i) - h(S_j)]^2 / 2$ \eqref{eq:hat vh}; $\sigma_{WP}^2 = \frac{1}{\BB} \sum_{b = 1}^\BB \frac{1}{M-1} \sum_{i = 1}^M (\hib - \phibbar)^2$; $\sigma_{BP}^2 = \frac{1}{\BB} (\phibbar - U_n)^2$.

\begin{proposition}
Our complete variance estimator $\vu$ is equivalent to the estimator $\vu^{(W\&L)}$ proposed by \citet{wang2014variance}.
\end{proposition}

\begin{proof}

To simplify the notation, we denote
\begin{align}
\label{eq:A1, A2, A3}
    A_1 := \phisqbar, \,\, A_2 := U_n^2, \,\, A_3 := \phibsqbar.
\end{align}
To show the equivalence between $\vu$ and $\vu^{(W\&L)}$, we will show that they are the same linear combination of $A_1, A_2, A_3$. 
First, it is trivial to verify that $\sigma_{WP}^2$ and $\sigma_{BP}^2$ are linear combinations of $A_1, A_2, A_3$:
\begin{align}
\label{eq:sigma WP, BP, A123}
    \sigma_{WP}^2 = \frac{M}{M-1} ( A_1 -A_3 ), \,\,
    \text{and}\,\,
    % &\hatvs =  \frac{1}{M\BB} \sum_{i=1}^M \sum_{b=1}^\BB (\hib - U_n)^2 = A_1 - A_2, \\
    \sigma_{BP}^2 = A_3 - A_2.
\end{align}

Secondly, we show that $\hatvh = \sigma_{WP}^2$ by showing that $\hatvh$ also equals to $\frac{M}{M-1} ( A_1 -A_3 )$. 
Considering the summation $\frac{1}{M-1} \sum_{i = 1}^M (\hib - \phibbar)^2$ in $\sigma_{WP}^2$, it can be represented as
\begin{align*}
    \binom{M}{2}^{-1} \sum_{i = 1}^M \sum_{j \neq i} (\hib - h(S_{j}^{(b)}))^2 / 2.
\end{align*}
Since $\Gn$ is a set of all permutation of disjoint $(S_1^{(b)}, ..., S_M^{(b)})$, we have 
\begin{align*}
    \sigma_{WP}^2 
    & = \frac{1}{\BB} \sum_{b = 1}^\BB  \binom{M}{2}^{-1} \sum_{i = 1}^M \sum_{j \neq i} (\hib - h(S_{j}^{(b)}))^2 / 2 \\
    & = \binom{n}{k}^{-1} \binom{n-k}{k}^{-1} \sum_{|S_i \cap S_j| = 0} [h(S_i) - h(S_j)]^2 / 2
    = \hatvh.
\end{align*}

Thirdly, we show that $\hatvs$ is also a linear combination of $A_1, A_2, A_3$.
We start with
$\hatvs =  \binom{n}{k}^{-1} \sum_{i = 1}^{\binom{n}{k}} h(S_i)^2 - U_n^2$.
Due to the definition of $\Gn$, the collection $\{S_i^{(b)} \}_{i, b}$ are basically replications of $\{S_i\}_{i=1}^{\binom{n}{k}}$. Hence,
$\binom{n}{k}^{-1} \sum_{i = 1}^{\binom{n}{k}} h(S_i)^2 = A_1$, which implies that 
\begin{align}
\label{eq:hatvs, A123}
\hatvs = A_1 - A_2.
\end{align}

Therefore, we can represent both $\vu$ and $\vu^{(W\&L)}$ with $A_1, A_2, A_3$ by \eqref{eq:sigma WP, BP, A123} and \eqref{eq:hatvs, A123}  as follows:
\begin{align*}
    \vu = \hatvh - \hatvs =   \frac{M}{M-1} ( A_1 -A_3 ) -  ( A_1 - A_2) = \frac{1}{M-1} A_1 + A_2 - \frac{M}{M-1} A_3, \\
    \vu^{(W\&L)} = \sigma_{WP}^2 / M - \sigma_{BP}^2 = \frac{1}{M-1} ( A_1 -A_3 ) - (A_3 - A_2) = \frac{1}{M-1}  A_1 + A_2 - \frac{M}{M-1} A_3.
\end{align*}
This conclude that $\vu  = \vu^{(W\&L)}$. Note that $\hatvh - \hatvs = \sigma_{WP}^2 / M - \sigma_{BP}^2$, however, $\hatvh \neq \sigma_{WP}^2 /M$ and $\hatvs \neq \sigma_{BP}^2$.
Our and \citet{wang2014variance}'s estimators are proposed under different perspectives.
\end{proof}

\subsection{Equivalence of Incomplete Variance Estimators}
\label{sec:append:equivalence,incomplete}

Only \citet{wang2014variance} and our paper propose variance estimator for in complete U-statistics. Similar to the analysis in Appendix \ref{sec:append:equivalence,complete}, we will show that our incomplete Variance estimator \eqref{eq:est var Uinc match} is equivalent to the counterpart in \citet[page 1124]{wang2014variance}. 
Given $B$ matching groups and $M$ subsamples in each group,  we denote the above estimators as $\vu^{(inc)}$ and $\vu^{(inc, W\&L)}$ respectively: 
\begin{align}
\label{eq:vu inc our and lindsay}
    \vu^{(inc)} := \hatvhB -  \frac{MB-1}{MB} \hatvsB, \,\,
    \vu^{(inc, W\&L)} := \tilde \sigma_{WP}^2 / M - \tilde \sigma_{BP}^2,
\end{align}
where $\tilde \sigma_{WP}^2 := \frac{1}{(M-1)B} \sum_{b=1}^B \sum_{i = 1}^M (\phiib - \phibbarinc)^2$, $\tilde \sigma_{BP}^2 := \frac{1}{B} \sum_{b=1}^B (\phibbarinc - \phibar)^2$, and $\phibbarinc = \frac{1}{M} \sum_{i = 1}^M \phiib$.

As analogues to $A_1, A_2, A_3$ \eqref{eq:A1, A2, A3} in Appendix \ref{sec:append:equivalence,complete}, we denote
\begin{align*}
    \tilde A_1 := \frac{1}{MB} \sum_{b = 1}^{B} \sum_{i = 1}^M (h(S_{i}^{(b)}))^2, \,\, 
    \tilde A_2 := \UincMatch^2, \,\, 
    \tilde A_3 := \frac{1}{B} \sum_{b = 1}^{B}  (\bar h_{(b)})^2.
\end{align*}
Similarly, it is trivial to verify that $\hatvhB$ \eqref{eq:varh est inc}, $\hatvsB$ \eqref{eq:hatvsB},  $\tilde \sigma_{WP}^2$ and $\tilde \sigma_{BP}^2$ can be represented as linear combinations of $\tilde A_1, \tilde A_2$ and $\tilde A_3$ as follows:
\begin{align*}
    \hatvhB = \tilde \sigma_{WP}^2 = \frac{M}{M-1} (\tilde A_1 - \tilde A_3),\,\,
    \hatvsB = \frac{MB}{MB-1} (\tilde A_1 - \tilde A_2), \,\,
    \sigma_{BP}^2 = \tilde A_3 - \tilde A_2.
\end{align*}
By plugging the above into equation \eqref{eq:vu inc our and lindsay}, we have
\begin{align*}
    &\vu^{(inc)} = \hatvhB -  \frac{MB-1}{MB} \hatvsB
        = \frac{1}{M-1} \tilde A_1 + \tilde A_2 - \frac{M}{M-1} \tilde A_3  \\
    &\vu^{(inc, W\&L)} = \tilde \sigma_{WP}^2 / M - \tilde \sigma_{BP}^2
    = \frac{1}{M-1} \tilde A_1 + \tilde A_2 - \frac{M}{M-1} \tilde A_3.
\end{align*}
Hence, we conclude that $\vu^{(inc)} = \vu^{(inc, W\&L)}$.

\section{Technical Propositions and Lemmas}
\label{sec:technical lemmas}   

In this section, we present the technical propositions and lemmas. The proofs of these results are collected in Appendix \ref{sec:append:proof of tech lemma}.

\begin{proposition}
\label{prop:zero-mean}
The value of $\psi\left(\Skk \right)$ does not depend on $\E \left[ h(X_1,.., X_k) \right]$. Therefore, WLOG, we can assume the kernel is zero-mean, i.e., $\E \left[ h(X_1,.., X_k) \right] = 0$
\end{proposition}

The proof of this proposition is collected in Appendix  \ref{sec:proof:prop:zero-mean}. %we prove this  straightforwardly by the alternative representation of $\vu$ \eqref{eq:alternative form of phi}.

\subsection{Results of \texorpdfstring{$\sigma_{c,2k}^2$}{sigma}}
\label{sec:append:lemma:sigma}

First, we present Propositions \ref{lemma:bound each sigma_c^2 for finite c} and \ref{lemma:bound sigma_c^2 for any c}. The former provides a 
 precise bound of $\sigma_{c, 2k}^2$ for some fixed $c$ while 
the latter provides rough bound for  $1 \leq c \leq 2k$.

\begin{proposition}[Bound $\sigma_{c,2k}^2$ for finite $c$]
\label{lemma:bound each sigma_c^2 for finite c}
Fix $T_1 = \floor{\frac{1}{\epsilon}}+1$. Under \allassump, for any c that $1 \leq c \leq T_1$, 
\begin{align*}
    \sigma_{c,2k}^2 = \calO\left(\frac{k^2}{n^2} F_c^{(k)}\right). 
\end{align*}
Based on the upper bound, we define $\check \sigma_{c, 2k}^{2} :=  \frac{C k^2}{n^2} F_c^{(k)}$, where $C$ is a generic positive constant. Equivalently, we write it as $\check \sigma_{c, 2k}^{2} \asymp  \frac{k^2 F_c^{(k)}}{n^2}$.
\end{proposition}

\begin{proposition}
[Bound $\sigma_{c,2k}^2$ for any $c$]
\label{lemma:bound sigma_c^2 for any c}
Under \allassump,
for any $1 \leq c\leq 2k$, we have
\begin{align*}
\sigma_{c, 2k}^2 
    = \calO(F_c^{(k)}) .
\end{align*}
\end{proposition}

The proof of the above propositions is collected in Appendix \ref{sec:proof:lemma:bound each sigma_c^2 for finite c} and Appendix \ref{sec:proof:lemma:bound sigma_c^2 for any c} respectively. 
% This a generalization of Lemma \ref{lemma:bound sigma_1^2}. 
Note that the upper bound in Proposition \ref{lemma:bound each sigma_c^2 for finite c} actually works for any fixed and finite $c$ but it suffices to restrict $c \leq T_1$ to show our main results.
These two propositions depend on the further decomposition of $\sigma_{c, 2k}^2$ into weighted sum of $\etacd$'s, which is later discussed in Appendix \ref{sec:append:lemma:eta}. In particularly, we can show that $\eta_{c, 2k}^2(1, 1)$ dominates $\sigma_{c, 2k}^2$ for $c = 1, 2, ..., T_1$. 
As a corollary of the above results, we can show the following lemma. 

\begin{lemma}[Truncated Variance Lemma I]
\label{lemma:Finite truncated variance I}
Under \allassump,  there exists a constant $T_1 = \floor{\frac{1}{\epsilon}}+1$, such that
\begin{align}
\label{eq:finite truncation lemma I}
    \lim_{n \to \infty} \frac{\Var \left(\hat V_u \right) - \Var^{(T_1)} \left(\hat V_u \right)}{\widecheck{\Var}^{(T_1)} \left(\hat V_u \right)} = 0,
\end{align}
where
\begin{align}
\label{eq:def:Var(vu) T_1}
\Var^{(T_1)}\left(\hat V_u\right)
    &:= \binom{n}{2k}^{-1} \sum_{c=1}^{T_1} \binom{2k}{c} \binom{n-2k}{2k-c} \sigma_{c, 2k}^2; \\
\label{eq:def:check Var(vu) T_1}
\widecheck{\Var}^{(T_1)}\left(\hat V_u\right)
    &:= \binom{n}{2k}^{-1} \sum_{c=1}^{T_1} \binom{2k}{c} \binom{n-2k}{2k-c} \widecheck{\sigma}_{c, 2k}^2.
\end{align}
Here, we denote $\widecheck{\sigma}_{c, 2k}^{2}$ as the upper bound of $\sigma_{c, 2k}^{2}$ given by Proposition \ref{lemma:bound each sigma_c^2 for finite c}. 
\end{lemma}

The proof of Lemma \ref{lemma:Finite truncated variance I} is collected in Appendix \ref{sec:proof:lemma:Finite truncated variance I}. This implies that to bound $\Var(\hat V_u)$, it suffices to bound the weighted average of first $T_1$ terms of $\sigma_{c, 2k}^2$, instead of all $2k$ terms.
Note that we use $\widecheck{\Var}$ instead of $\Var$ in the denominator of \eqref{eq:finite truncation lemma I}.
Here $T_1 = \floor{\frac{1}{\epsilon}}+1$ does not grow with $n$. It only relies on $\epsilon$, which quantifies the growth rate of $k$ with respect to $n$ (see Assumption \ref{assumption: k = n^(1/2 - epsilon)},). For example, if $k = n^{1/3}$, i.e., $\epsilon = 1/6$, then we can choose $T_1 = 7$.
Hence, to show $\sigma_{1, 2k}^2$  dominates in $\Var (\hat V_u )$ when $n \to \infty$, it suffices to show that $\sigma_{1, 2k}^2$  dominates in $T_1$-truncated $\Var^{(T_1)} (\hat V_u )$ when $n \to \infty$.
% This requires us to upper bound all $\sigma_{c, 2k}^2$ for $1 \leq c \leq T_1$. 

\subsection{Results of \texorpdfstring{$\etacd$}{eta}}
\label{sec:append:lemma:eta}

Given the decomposition $\sigma_{c, 2k}^{2} = \sum_{d_1 = 1}^{k} \sum_{d_2 = 1}^{k} w_{d_1} w_{d_2} \etacd$ (see \eqref{eq:sigma_c^2 decomposition} in  Proposition \ref{prop:sigma_c^2 written as eta_c^2}). To bound $\sigma_{c, 2k}^2$, we should study $\etacd$ \eqref{eq:sigma_c^2 decomposition}. The results of $\etacd$ are presented in this section.

\begin{lemma}
\label{lemma:eta_c (d1, d2)}
Under \allassump, for $1 \leq c\leq T_1, 1 \leq d_1, d_2 \leq T_2$,
\begin{align*}
    \etacd 
    = \calO \left( \frac{1}{k^2} F_c^{(k)} \right).
\end{align*}
\end{lemma}

\begin{lemma}
\label{lemma:bound every eta_c^2}
Under \allassump, for $c = 1, 2, ..., 2k$, $d_1, d_2 = 1, 2, ..., k$,
\begin{align*}
    \etacd = \calO(\fk_c).
\end{align*}
\end{lemma}
Similar to the ``two-type'' upper bounds of $\sigma_{c, 2k}^2$,  
Lemma \ref{lemma:eta_c (d1, d2)} provides a precise bound of $\etacd$ for bounded $c$ and $d_1, d_2$  while  Lemma \ref{lemma:bound every eta_c^2} provides a rough bound of $\etacd$ for all $c, d_1, d_2$. 
Again, the result of Lemma \ref{lemma:eta_c (d1, d2)} actually holds for any fixed and finite $c, d_1, d_2$.
The proof of the above lemmas are collected in Appendix \ref{sec:proof:lemma:eta_c} and Appendix  \ref{sec:proof:lemma:bound every eta_c^2} respectively. The proof demonstrates the cancellation pattern by matching $\covh$ \eqref{eq:def:rho = Cov(h1h2, h3h4)}.

With the above bounds on $\etacd$, we introduce the following truncated $\sigma_{c, 2k}^2$ and show Lemma \ref{lemma:Finite truncated variance II}, which implies that to bound $\sigma_{c, 2k}^2$, it suffices to bound the first finite $\etacd$ terms in its decomposition \eqref{eq:sigma_c^2 decomposition}.

\begin{definition}[Truncated $\sigma_{c, 2k}^2$]
\label{def:truncated covariance}

Let $T_2 = \floor{\frac{1}{\epsilon}}+1$. 
We define $\psi^{(T_2)}$, a $T_2$-truncated $\psi$ as 
\begin{align}
\label{eq:def:psi, truncated}
\psi^{(T_2)}(\Skk_1)
    := \sum_{d = 1}^{T_2} w_c \left(\varphi_d \left(\Skk \right) - \varphi_0 \left(\Skk \right)\right).
\end{align}
Hence, given two size-$2k$ subsamples $\Skk_1$ and $\Skk_2$ that $|\Skk_1 \cap  \Skk_2 | = c$, a $T_2$-truncated of $\sigma_{c, 2k}^2$ are defined as:
\begin{align}
    \label{eq:def:truncated sigma_c^2}
    \sigma_{c, 2k, (T_2)}^{2} 
    := \Cov (\psi^{(T_2)} (\Skk_1), \psi^{(T_2)} (\Skk_2))
     = \sum_{d_1 = 1}^{T_2} \sum_{d_2 = 1}^{T_2} w_{d_1} w_{d_2} \etacd.
\end{align}
\end{definition}

\begin{lemma}[Truncated Variance Lemma II]
\label{lemma:Finite truncated variance II}
Under \allassump, there exists a constant $T_2 = \floor{\frac{1}{\epsilon}}+1$, such that for any $c \leq T_1$, 
\begin{align*}
    \lim_{k \to \infty} \frac{\sigma_{c, 2k}^2 - \sigma_{c, 2k, (T_2)}^{2}}{\widecheck{\sigma}_{c, 2k}^2} = 0,
\end{align*}
where $\widecheck{\sigma}_{c, 2k, (T_2)}^{2}$ is the upper bound of $\sigma_{c, 2k, (T_2)}^{2}$ given in Proposition \ref{lemma:bound each sigma_c^2 for finite c}.
\end{lemma}
The proof is collected in Appendix \ref{sec:proof:lemma:Finite truncated variance II}. 
Similar to the idea of Lemma \ref{lemma:Finite truncated variance I}, by 
Lemma \ref{lemma:Finite truncated variance II}
the upper bound of $\sigma_{c, 2k}^2$ \eqref{eq:sigma_c^2 decomposition} only involves the sum of $T_2^2$ terms, i.e.,  $\sigma_{c, 2k, (T_2)}^{2}$ rather than $k^2$ terms.
Here $T_2$ is again finite and does not grow with $k$.
Though $T_1$ and $T_2$ take the same value, we note that $T_1$ is the truncation constant for $\Var(\vu)$ in  \eqref{eq:def:Var(vu) T_1} while $T_2$ is the truncation constant for $\sigma_{c, 2k}^2$ in \eqref{eq:def:truncated sigma_c^2}.

%%%%%%%%%%%%%%%%%%%%%%%%%%%%%%%%%%%%%%%%%%%%%%%%%%%%%%

\section{Proof of Technical Propositions and Lemmas}
\label{sec:append:proof of tech lemma}
\subsection{Proof of Proposition \ref{prop:zero-mean}}
\label{sec:proof:prop:zero-mean}
\begin{proof}[Proof of Proposition  \ref{prop:zero-mean}]
\label{proof:prop:zero-mean}

Suppose $\E(h(S_1)) = \mu$ and we rewrite $h(S_1) = \hh(S_1) + \mu$, where $\E(\hh(S_1)) = 0$. Then $\varphi_d(\Skk)$ defined in \eqref{eq:def:phi} can be written as
\begin{align}
\label{eq:phi with non-zero mean h}
\varphi_d \left(\Skk \right) 
    = \frac{1}{M_{d, k}} \sum_{S_1, S_2 \subset \Skk, | S_1 \cap S_2 | = d}  \hh(S_1) \hh(S_2) + [\hh(S_1) + \hh(S_2)] \mu + \mu^2.
\end{align}

Plug  Equation \eqref{eq:phi with non-zero mean h} into Equations \eqref{eq:def:vu U-stat size-2k} and \eqref{eq:alternative form of phi}.
By the fact that $\sum_{d = 0}^k w_d = 0$ and the symmetry of U-statistic, the terms of $\mu$ and $\mu^2$ are cancelled. 
Consequently, $\vu$ does not depend on $\mu$, so W.L.O G., we  assume that $\mu =0$ 
\end{proof}

\subsection{Proof of Proposition \ref{lemma:bound each sigma_c^2 for finite c}}
\label{sec:proof:lemma:bound each sigma_c^2 for finite c}

\begin{proof}[Proof of Proposition \ref{lemma:bound each sigma_c^2 for finite c}]
\label{proof:lemma:bound each sigma_c^2 for finite c}
This proof relies on the technical lemmas in Appendix \ref{sec:append:lemma:eta}.
First, by Lemma \ref{lemma:Finite truncated variance II}, to upper bound $\sigma_{c, 2k}^2$, it suffices to upper bound the following 
$\sigma_{c, 2k, (T_2)}^{2}$ \eqref{eq:def:truncated sigma_c^2}
\begin{align*}
%\label{eq:sigma_c^2, decomposition under eta}
    \sigma_{c, 2k, (T_2)}^{2} 
    = \sum_{d_1 = 1}^{T_2} \sum_{d_2 = 1}^{T_2} w_{d_1} w_{d_2} \etacd,
\end{align*}
where $w_d = [1 + o(1)] \left [\frac{1}{d!}(\frac{k^2}{n})^d \right], \forall d \leq T_2$.

By Lemma \ref{lemma:eta_c (d1, d2)}: fixing any $c$ s.t. $1 \leq c \leq T_2$,  $\etacd = \calO(\frac{F_c^{(k)}}{k^2}), \forall d_1, d_2 \leq T_2$. Besides, since $d_1$ and $d_2$ are bounded by a constant $T_2$. $w_1^2$ dominates the summation $\sum_{d_1 = 1}^{T_2} \sum_{d_2 = 1}^{T_2} w_{d_1} w_{d_2}$. We have
\begin{align*}
\sigma_{c, 2k, (T_2)}^{2} 
    & = \sum_{d_1 = 1}^{T_2} \sum_{d_2 = 1}^{T_2} w_{d_1} w_{d_2} \etacd 
     = \calO(\frac{F_c^{(k)}}{k^2})  \sum_{d_1 = 1}^{T_2} \sum_{d_2 = 1}^{T_2} w_{d_1} w_{d_2} \\
    & =\calO(\frac{F_c^{(k)}}{k^2})  [1 + o(1)] w_1^2  
     = \calO(\frac{k^2}{n^2}) F_c^{(k)}.
\end{align*}
Here, the last equality is derived by plugging in $w_1 = [1 + o(1)] \frac{k^2}{n}$.

% The challenging part is to prove Lemma \ref{lemma:eta_c (d1, d2)}.

% old version of sigma_c^2 (given the proof of sigma_1^2)
% \iffalse
% We want to generalize the proof from $\sigma_{1, 2k}^{2 (T_2)}= \calO(\frac{5d_1 d_2}{4k^2} F_1^{(k)})$ to  $\sigma_{c, 2k, (T_2)}^{2}$, $\forall 1 \leq c \leq T_1$. $\sigma_{c, 2k, (T_2)}^{2}$ is defined in \ref{def:truncated covariance}:
% \begin{align}
% \label{eq:sigma_c^2, decomposition under eta}
%     \sigma_{c, 2k, (T_2)}^{2} 
%     = \sum_{d_1 = 1}^{T_2} \sum_{d_2 = 1}^{T_2} w_{d_1} w_{d_2} \etacd,
% \end{align}
% where $w_d = [1 + o(1)] \left [\frac{1}{d!}(\frac{k^2}{n})^d \right], \forall d \leq T_2$.

% By Lemma \ref{lemma:eta_c (d1, d2)}: fixing any $c \leq T_2$,  $\etacd = \calO(\frac{F_c^{(k)}}{k^2}), \forall d_1, d_2 \leq T_2$. Therefore, by the analysis in Appendix \ref{sec:proof:lemma:bound sigma_1^2}, $w_1^2\eta_{c, 2k}^2(1, 1)$ again dominates $\sigma_{c, 2k, (T_2)}^{2}$
% \begin{align*}
% \sigma_{c, 2k, (T_2)}^{2} 
%     = \sum_{d_1 = 1}^{T_2} \sum_{d_2 = 1}^{T_2} w_{d_1} w_{d_2} \etacd
%     \leq [1 + o(1)] w_1^2 \calO(\frac{1}{k^2}) F_c^{(k)}.
% \end{align*}
% Then plug in $w_1 = [1 + o(1)] \frac{k^2}{n}$, we achieve 
% \begin{align*}
% \sigma_{c, 2k, (T_2)}^{2} 
%     \leq  \calO(\frac{k^2}{n^2}) F_c^{(k)}.
% \end{align*}
% The challenging part is to prove Lemma \ref{lemma:eta_c (d1, d2)}.
% \fi
\end{proof}

\subsection{Proof of Proposition \ref{lemma:bound sigma_c^2 for any c}}
\label{sec:proof:lemma:bound sigma_c^2 for any c}

\begin{proof}[Proof of Proposition \ref{lemma:bound sigma_c^2 for any c}]
\label{proof:lemma:bound sigma_c^2 for any c}

%This is a rough bound for any $\sigma_{c,2k}^2$. As mentioned before, the first inequality is a corollary of Proposition \ref{prop:sigma_c^2 written as eta_c^2} and Equation \eqref{eq:def:F_c}. The second inequality is concluded from the Assumption \ref{assumption:4th moment}. 
This lemma again relies on the upper bound of $\etacd$ in Appendix \ref{sec:append:lemma:eta}. 
By Proposition \ref{prop:sigma_c^2 written as eta_c^2}, we can decompose $\sigma_{c, 2k}^2$ as
\begin{align*}
    \begin{split}
    \sigma_{c, 2k}^{2} 
    = \sum_{d_1 = 1}^{k} \sum_{d_2 = 1}^{k} w_{d_1} w_{d_2} \etacd,
    \end{split}
\end{align*}
First, we investigate the coefficient of $\etacd$. By Proposition \ref{prop: double U-statistic: alternative form of phi}, we have $\sum_{d_1 = 1}^{k}w_{d_1} < 1$. Thus, we attain
\begin{align*}
   \sum_{d_1 = 1}^{k} \sum_{d_2 = 1}^{k} w_{d_1} w_{d_2} 
   = (\sum_{d_1 = 1}^{k}  w_{d_1}) (\sum_{d_2 = 1}^{k} w_{d_2}) < 1.
\end{align*}
By Lemma \ref{lemma:bound every eta_c^2}, we have $\etacd = \calO(F_c)$ for $c = 1, 2, 3, ..., 2k$.
Hence, combining the bounds on $w_d$ and $\etacd$, we conclude that
\begin{align*}
   \sigma_{c, 2k}^{2} = \calO(F_c^{(k)}) 
\end{align*}

We remark that the summation $\sum_{d_1 = 1}^{k}w_{d1}$ can attain a lower order of 1, which may imply a tighter bound of $\sigma_{c, 2k}^2$.

\end{proof}

\subsection{Proof of Lemma \ref{lemma:Finite truncated variance I}}
\label{sec:proof:lemma:Finite truncated variance I}

\begin{proof}

Let $T_1 = \floor{\frac{1}{\epsilon}}+1$. Recall Equation  \eqref{eq:Var(Vu) expression} $\Var\left(\hat V_u\right)
    = \binom{n}{2k}^{-1} \sum_{c=1}^{2k} \binom{2k}{c} \binom{n-2k}{2k-c} \sigma_{c, 2k}^2$.
We first present the intuition of this lemma. $\Var(\vu)$ is a weighted sum of $\sigma_{c, 2k}^2$, where the coefficient of $\sigma_{c, 2k}^2$ decays with $c$ at a rate even faster than a geometric rate. If the growth rate of $\sigma_{c, 2k}^2$ is not too fast, then the tail terms can be negligible.
% To show $
%     \lim_{n \to \infty} \frac{\Var \left(\hat V_u \right) - \Var^{(T_1)} \left(\hat V_u \right)}{\widecheck{\Var}^{(T_1)} \left(\hat V_u \right)} = 0$, 
% since we only want to upper bound $\Var(\vu)$, it suffices to            consider ${\widecheck{\Var}^{(T_1)} (\hat V_u )}$  \eqref{eq:def:check Var(vu) T_1} instead of $\Var^{(T_1)} (\hat V_u )$ in the denominator.
This involves both the precise upper bound of $\sigma_{1, 2k^2}$ (Proposition \ref{lemma:bound each sigma_c^2 for finite c}) and the rough upper bound of  $\sigma_{c,2k}^2$ for $c \geq T_1 + 1$ (Proposition \ref{lemma:bound sigma_c^2 for any c}).

First,  $\binom{n}{2k}^{-1} \binom{2k}{1} \binom{n-2k}{2k-1} \widecheck{\sigma}_{1, 2k}^2 \leq \widecheck{\Var}^{(T_1)} \left(\hat V_u \right)$ since the former is the first tern in the latter and all the other terms are positive. Therefore, it suffices to show 
\begin{align}
\label{eq:reminder in Var(Vu) is negligible}
    \frac
    {\Var \left(\hat V_u \right) - \Var^{(T_1)} \left(\hat V_u \right)}
    {\binom{n}{2k}^{-1} \binom{2k}{1} \binom{n-2k}{2k-1} \widecheck{\sigma}_{1, 2k}^2}
    = \frac
    {\sum_{c=T_1 + 1}^{2k} \binom{n}{2k}^{-1}  \binom{2k}{c} \binom{n-2k}{2k-c} \sigma_{c, 2k}^2}
    {\binom{n}{2k}^{-1} \binom{2k}{1} \binom{n-2k}{2k-1} \widecheck{\sigma}_{1, 2k}^2}
    \to 0.
\end{align}
We bound the numerator and denominator in Equation \eqref{eq:reminder in Var(Vu) is negligible} separately. For the denominator, 
by the analysis of Equation \eqref{eq:leading term, asymp form}, we have
\begin{align}
\label{eq:denominator asympt}
    \binom{n}{2k}^{-1} \binom{2k}{1} \binom{n-2k}{2k-1} \widecheck{\sigma}_{1, 2k}^2= \frac{[1 + o(1)] 4k^2}{n} \widecheck{\sigma}_{1, 2k}^2.   
\end{align}
For the numerator, by Proposition \ref{lemma:bound sigma_c^2 for any c} and assumption \ref{assumption:4th moment}, $\sigma_{c, 2k}^2 = \calO(F_c^{(k)}) = o(c^{a_2} F_1^{(k)})$. Therefore, it suffices to show that 
\begin{align}
\label{eq:reminder in Var(Vu) is negligible 2}
\begin{split}
\frac
    {\sum_{c=T_1 + 1}^{2k} \binom{n}{2k}^{-1}  \binom{2k}{c} \binom{n-2k}{2k-c} o(c^{a_2} F_1^{(k)})} 
    { \frac{4k^2}{n} \widecheck{\sigma}_{1, 2k}^2} 
    = 
\frac
    { \calO \left(\{  \binom{n}{2k}^{-1}  \binom{2k}{c} \binom{n-2k}{2k-c} o(c^{a_2} F_1^{(k)}) \}_{c = T_1 + 1} \right)
    } 
    {\frac{4k^2}{n} \widecheck{\sigma}_{1, 2k}^2} 
    \xrightarrow{n\to \infty} 0.
\end{split}
\end{align}
The equality in \eqref{eq:reminder in Var(Vu) is negligible 2} is given by Proposition \ref{prop:T1 +1 term dominates the reminder}. The followed  $ \xrightarrow{n\to \infty} 0$ in \eqref{eq:reminder in Var(Vu) is negligible 2} is given by Proposition \ref{prop:T1 +1 term vs 1st term}. This completes the proof.

% To show that the LHS of Equation \eqref{eq:reminder in Var(Vu) is negligible 2} $\to 0, \text{as~} n \to \infty$ ,  it suffices to show the followed  and  \ref{prop:T1 +1 term vs 1st term}. In short,
%  Proposition \ref{prop:T1 +1 term dominates the reminder} claims that the numerator in \eqref{eq:reminder in Var(Vu) is negligible 2} is dominated by the upper bound of its first term ($c = T_1 + 1$);
% Proposition \ref{prop:T1 +1 term vs 1st term} claims the upper bound of $c = T_1 + 1$ term in the numerator is a lower order term  compared to the denominator in Equation  \eqref{eq:reminder in Var(Vu) is negligible 2}.
\end{proof}

\begin{proposition}
\label{prop:T1 +1 term dominates the reminder}
Under \allassump,
\begin{align}
\label{eq:prop:T1 +1 term dominates the reminder}
    \frac
    {\sum_{c=T_1 + 2}^{2k}\binom{n}{2k}^{-1}  \binom{2k}{c} \binom{n-2k}{2k-c} c^{a_2} F_1^{(k)}} 
    {\binom{n}{2k}^{-1} \binom{2k}{T_1 + 1} \binom{n-2k}{2k-T_1 -1} {(T_1+1)}^{a_2}
    F_1^{(k)}}
    \to 0, \, \text{as $n \to \infty$} 
\end{align}.
\end{proposition}

\begin{proof}
The proof  of Proposition \ref{prop:T1 +1 term dominates the reminder} is similar to the proof of Proposition \ref{thm:leading variance of U-stat}. The idea is that the sum of tail coefficients is a geometric sum and thus dominates the growth rate of moments.
 
First we consider the coefficient $\frac
{\sum_{c=T_1 + 2}^{2k} \binom{n}{2k}^{-1}  \binom{2k}{c} \binom{n-2k}{2k-c}} 
{\binom{n}{2k}^{-1} \binom{2k}{T_1 + 1} \binom{n-2k}{2k-T_1 -1}}$.  
By Proposition \ref{prop:geometric ratio} and our analysis in Equation \eqref{eq:geometric sum: pre-form} and \eqref{eq:geometric sum}, let $b_n = \frac{4k^2}{n-2k+1}$ which is the common ratio in the geometric sequence.
 \begin{align}
 \label{eq:tail coef}
     \frac
    {\sum_{c=T_1 + 2}^{2k} \binom{n}{2k}^{-1}  \binom{2k}{c} \binom{n-2k}{2k-c}} 
    {\binom{n}{2k}^{-1} \binom{2k}{T_1 + 1} \binom{n-2k}{2k-T_1 -1}}
    \leq \sum_{c = T_1 + 2}^{2k}[1+o(1)] \frac{(T_1 + 1)!}{c!} b_n^{c - (T_1+1)}.
    %\leq \sum_{c = T_1 + 2}^{2k}[1+o(1)]  b_n^{c - (T_1+1)}
 \end{align}
 
Second, combining $c^{a_1}$ with \eqref{eq:tail coef}, it's again the problem  of geometric series with common ratio $b_n = o(1)$. We have
\begin{align}
\begin{split}
\text{LHS of \eqref{eq:prop:T1 +1 term dominates the reminder}}
    \leq \sum_{c=T_1+2}^{2k} \calO \left(\frac{c^{a_2}}{c!}\right) b_n^{c-(T_1+1)} 
    \leq \sum_{c=T_1+2}^{2k} \calO(1) b_n^{c-(T_1+1)} 
    \leq \sum_{c=1}^{\infty} \calO(1)  b_n^{c} 
    = \calO(1) b_n \to 0,  
\end{split}
\end{align}
where the last equality is concluded by the sum of geometric series.
\end{proof}

\begin{proposition}
\label{prop:T1 +1 term vs 1st term}
Under \allassump,
\begin{align}
\label{eq:T1 +1 term vs 1st term}
\frac
    {\binom{n}{2k}^{-1} \binom{2k}{T_1 + 1} \binom{n-2k}{2k-T_1 -1} C' {(T_1+1)}^{a_2}
    F_1^{(k)}}
    {\frac{k^2}{n} \checksigone} = o(\frac{1}{k^2}) \to 0.
\end{align}
In Equation \eqref{eq:T1 +1 term vs 1st term}, the upper bound of $c = T_1 + 1$ term of the numerator is a lower order term compared to the denominator.
\end{proposition}

\begin{proof}

It suffices bound two separate parts in Equation \eqref{eq:T1 +1 term vs 1st term},
\begin{align}
\label{eq:part 1:coef}
    \frac
    {\binom{n}{2k}^{-1} \binom{2k}{T_1 + 1} \binom{n-2k}{2k-T_1 -1} }
    {\frac{k^2}{n}} = o(\frac{1}{n^2}), \\
\label{eq:part 2:moment}
    \frac{C' {(T_1+1)}^{a_2}
    F_1^{(k)}}
    {\checksigone} = \calO (\frac{n^2}{k^2}).
\end{align}
Then, combining Equation \eqref{eq:part 1:coef} and \eqref{eq:part 2:moment}, we have 
\begin{align*}
\text{LHS of \eqref{eq:T1 +1 term vs 1st term}}
     \leq o\left(\frac{1}{n^2}\right) \calO(\frac{n^2}{k^2}) 
     = o\left(\frac{1}{k^2}\right) \to 0.
\end{align*}

We first show Equation \eqref{eq:part 1:coef}, i.e., bound the ratio of coefficient. Similar to the analysis for Equation \eqref{eq:geometric sum}, by Proposition \ref{prop:geometric ratio}, we have
\begin{align*}
    \binom{n}{2k}^{-1} \binom{2k}{T_1 + 1} \binom{n-2k}{2k-T_1 -1} 
    \leq \frac{(2k)^{2 (T_1+1)}}{(T_1+1) !} \frac{1}{(n-2k+1)^{T_1 + 1}} 
    \leq \frac{1}{(T_1 + 1)!} b_n^{T_1 + 1},
\end{align*}
where $b_n = \frac{4k^2}{n-2k+1}$. 
Therefore the ratio of coefficient,
\begin{align*}
    \frac{\binom{n}{2k}^{-1} \binom{2k}{T_1 + 1} \binom{n-2k}{2k-T_1 -1}}{\frac{4k^2}{n}}
    \leq [1 + o(1)] \frac{1}{(T_1 + 1)!} b_n^{T_1}.
\end{align*}
It remains to show $b_n^{T_1} = o(\frac{1}{n^2})$.
By $T_1 = \floor{\frac{1}{\epsilon}}+ 1$ and $k = \calO(n^{1/2 - \epsilon})$, we have
\begin{align}
\label{eq:bound ratio of coefficient}
    b_n^{T_1} \leq (\frac{4k^2}{n})^{\floor{1/\epsilon} + 1} = \calO((n^{-2\epsilon})^{\floor{1/\epsilon} + 1} ) = o(\frac{1}{n^2}).
\end{align}

Second, we show \eqref{eq:part 2:moment}, i.e., bound the ratio of moments. By Lemma \ref{lemma:Finite truncated variance II} and \ref{lemma:bound each sigma_c^2 for finite c}, $\sigma_{1, 2k}^2 = \calO( k^2 / n^2 F_1^{(k)})$ and $\checksigone \asymp k^2 / n^2 F_1^{(k)}$. We have
\begin{align}
\label{eq:bound ratio of moment}
    \frac{C' {(T_1+1)}^{a_2}F_1^{(k)}}{\checksigone} = \calO \left( n^2 / k^2 (T_1+1)^{a_2} \right) = \calO(\frac{n^2}{k^2}).
\end{align}

\end{proof}

%%%%%%%%%%%%%%%%%

\subsection{Proof of Lemma \ref{lemma:eta_c (d1, d2)}}
\label{sec:proof:lemma:eta_c}

% \documentclass[../main.tex]{subfiles}

% \begin{document}

\begin{proof}[Proof of Lemma \ref{lemma:eta_c (d1, d2)}]
\label{proof: lemma:eta_c (d1, d2) with 9 dfs}

First, we present a sketch of this proof. Given $\Skk_1$ and $\Skk_2$, our strategy tracks the distribution of \emph{Influential Overlaps}, i.e., the samples in $\Skk_1 \cap \Skk_2$. 
We will decompose $\etacd$ as a finite weighted sum:
\begin{align}
\label{eq:eta_c = sum ai bi}
    \etacd = \sum_{i} a_i b_i,
\end{align}
where $i$ is the summation index to be specified later. Based on this form, we will show $a_i = \calO(\frac{1}{k^2})$ and $b_i = \calO(F_c^{(k)})$ for each $i$.
Since \eqref{eq:eta_c = sum ai bi} is a finite sum, we can conclude $\etacd = \calO\left(\frac{1}{k^2} F_c^{(k)}\right)$. 
Details of \eqref{eq:eta_c = sum ai bi} will be presented later. 
We remark that it is straightforward to upper bound $\eta_{1, 2k}^2 (d_1, d_2)$ by enumerating all the possible 4-way overlapping cases of $S_1, S_2, S_3, S_4$ given $c, d_1, d_2$ for small $c$. However, the growth of $c$ from $1$ to $2, 3, 4, ..., T_2$ makes ``enumerating''  impossible.

We start the proof by reviewing the definition of $\etacd$ \eqref{eq:def:eta}:
\begin{align*}
    \etacd 
    = \Cov \left[ \varphi_{d_1}\left(\Skk_1 \right) - \varphi_0\left(\Skk_1 \right), \varphi_{d_2}\left(\Skk_2 \right) - \varphi_0\left(\Skk_2 \right) \right],
\end{align*}
where $\varphi_d \left(\Skk \right) 
    = [\binom{2k}{d} \binom{2k-d}{d} \binom{2k-2d}{k-d}]^{-1} \sum_{S_1, S_2 \subset \Skk, | S_1 \cap S_2 | = d} h(S_1) h(S_2)$ \eqref{eq:def:phi}. 
The following proof is organized in two parts. First, we propose an alternative representation of the covariance $\Cov[\varphi_{d_1}(\Skk_1), \varphi_{d_2}(\Skk_2)]$, which helps discover the cancellation pattern of $\etacd$ \eqref{eq:def:eta}. Secondly, we derive Equation \eqref{eq:eta_c = sum ai bi} and specify $a_i$'s and $b_i$'s.

First, we notice that $\varphi_d \left(\Skk \right)$ \eqref{eq:def:phi} is a weighted average of the product of two kernels $h(S_1) h(S_2)$:
\begin{align*}
    \varphi_d \left(\Skk \right) 
    = \left[\binom{2k}{d} \binom{2k-d}{d} \binom{2k-2d}{k-d} \right]^{-1} \sum_{S_1, S_2 \subset \Skk, | S_1 \cap S_2 | = d} h(S_1) h(S_2).
\end{align*}
Denote $\sum_{P_{12}}$ as summation over all pairs of $S_1, S_2$, s.t. $S_1, S_2 \subset \Skk, | S_1 \cap S_2 | = d_2$. Similarly, we can also denote $\sum_{P_{34}}$.
Then, we can represent the covariance $\Cov [ \varphi_{d_1} (\Skk_1 ) , \varphi_{d_2} (\Skk_2 )]$ as 
\begin{align}
\label{eq:cov(phi_d1, phi_d2) - original}
\Cov \left[ \varphi_{d_1}(\Skk_1 ) , \varphi_{d_2}(\Skk_2 ) \right]
& = \Cov \left[ M_{d_1}^{-1} \sum_{P_{12}} h(S_1) h(S_2),
    M_{d_2}^{-1} \sum_{P_{34}} h(S_3) h(S_4)\right] \\
\label{eq:cov(phi_d1, phi_d2) - original 2}
& = M_{d_1}^{-1} M_{d_2}^{-1} 
    \sum_{P_{12}} \sum_{P_{34}} \covh \\
\label{eq:cov(phi_d1, phi_d2) - integrated form}
& = \sum_{\text{feasible }  \unr} \pr \covh.
\end{align}
In the above equations, 
\begin{align*}
    \pr :=  M_{d_1}^{-1} M_{d_2}^{-1} \sum_{(S_1, S_2, S_3, S_4)} \mathbbm{1}\{\text{overlapping structure of $S_1, S_2, S_3, S_4$ satisfies $(\unr, d_1, d_2)$}\};
\end{align*}
$M_d = \binom{2k}{d} \binom{2k-d}{d} \binom{2k-2d}{k-d}$ is the number of pairs of sets in the summation.
The equality in \eqref{eq:cov(phi_d1, phi_d2) - integrated form} holds by combining the $\covh$ terms with the same $\unr$ (see definition of $\unr$ in Appendix \ref{sec:append:r}). 
Since it is difficult to figure out the exact value of $\pr$, we further propose the following proposition to show an alternative representation of $\Cov [ \varphi_{d_1}(\Skk_1 ) , \varphi_{d_2}(\Skk_2 ) ]$.

\begin{lemma}
\label{prop:cov phi phi, rMarg eq}
Denote $\rx{i} = \sum_{j = 0}^2 r_{ij}$ and $\ry{j} = \sum_{i = 0}^2 r_{ij}$, for $i, j = 0, 1, 2$ and vector $\unrMarg := ( \rx{0}, \rx{1}, \rx{2}, \ry{0}, \ry{1}, \ry{2})$, we have 
\begin{small}
\begin{align}
\label{eq:cov(phi_d1, phi_d2) - marginal r form II}
& \Cov \left[ \varphi_{d_1}\left(\Skk_1 \right) , \varphi_{d_2}\left(\Skk_2 \right)  \right] 
  = & \sum_{\text{feasible}\,\, \unrMarg} 
    \px
    \py 
     g(\unrMarg, d_1, d_2).
\end{align}
\end{small}
Here $\px$ and $\py$ are non-negative and satisfy the following. For non-negative integers $x_0, x_1, x_2$ that $x_0 + x_1 + x_2 \leq c$,
\begin{small}
\begin{align}
\label{eq:PMF for important samples}
&\pxx[d]
    := {\binom{d}{x_0} \binom{k-d}{x_1} \binom{k-d}{x_2} \binom{d}{c - x_0 - x_1 - x_2}}
        {\binom{2k}{c}}^{-1} = \calO(k^{x_1 + x_2 - c}).
\end{align}
\end{small}
$g(\unrMarg, d_1, d_2)$ is the following weighted average of $\covh$, where the weight is some constant $\prrMarg$ satisfying that $\sum_{\unrMarg} \prrMarg = 1$.
\begin{small}
\begin{align}
    \label{eq:def:g}
& g(\unrMarg, d_1, d_2)
    := \sum_{\unr} \prrMarg \covh.
\end{align}
\end{small}
\end{lemma}

The proof of Lemma \ref{prop:cov phi phi, rMarg eq} is deferred to the end of Appendix \ref{sec:proof:lemma:eta_c}. 
Under Assumption \ref{assumption:low dependency}, $\covh$ does not depend on $d_1, d_2$. Hence, $g(\unrMarg, d_1, d_2)$ also does not depend on $d_1, d_2$. We further denote $G(\unrMarg) = g(\unrMarg, d_1, d_2)$. Then, 
\begin{align}
    \label{eq:cov(phi_d1, phi_d2) - marginal r form III}
\Cov \left[ \varphi_{d_1}\left(\Skk_1 \right) , \varphi_{d_2}\left(\Skk_2 \right)  \right]
    = &%\sum_{(\rx{0}, \rx{1}, \rx{2})} \sum_{(\ry{0}, \ry{1}, \ry{2})}
    \sum_{\text{feasible}\, \unrMarg}
    \px
    \py 
     G(\unrMarg).
\end{align}
We remark that $\px$ and $p_{(\ry{0}, \ry{1}, \ry{2}, d_2, c)}$ can be viewed as some probability mass function with parameters $c, d_1, d_2$ (see the proof of Lemma \ref{prop:cov phi phi, rMarg eq}). 
As a corollary of \eqref{eq:PMF for important samples}, when $c - x_1 - x_2 > d$, $\pxx[d] = 0$. In particular, when $d = 0$ and $x_1 + x_2 \leq c-1$, $\pxx[d]$ is always 0.
% {
% \color{red}
% Note that by \eqref{eq:PMF for important samples}, $\px[d] = \binom{d_1}{\rx{0}} \binom{k-d_1}{\rx{1}} \binom{k-d_1}{\rx{2}} \binom{d_1}{c - \rx{0} - \rx{1} - \rx{2}} / 
%         {\binom{2k}{c}} = 0$, for infeasible $\rx{0}, \rx{1}, \rx{2}$.
% For example, $\rx{0} = 1$ while $d = 0$ is infeasible since no {Influential Overlaps} falling in $S_1 \cap S_2$.  
% }
% {
% \color{red}
% Note that $\px[d] = \binom{d_1}{\rx{0}} \binom{k-d_1}{\rx{1}} \binom{k-d_1}{\rx{2}} \binom{d_1}{c - \rx{0} - \rx{1} - \rx{2}} / 
%         {\binom{2k}{c}} = 0$, for infeasible $\rx{0}, \rx{1}, \rx{2}$.
% For example, $\rx{0} = 1$ while $d = 0$ is infeasible since no {Influential Overlaps} falling in $S_1 \cap S_2$.  
% }

Notice that since $\Skk_1$ is independent of $\Skk_2$, $\sum_{\text{feasible}\, \unrMarg}$ can be written as two sequential sums: $\sum_{(\rx{0}, \rx{1}, \rx{2})} \sum_{(\ry{0}, \ry{1}, \ry{2})}$. Therefore, by plugging the expression of $\Cov [ \varphi_{d_1}(\Skk_1) , \varphi_{d_2}(\Skk_2)]$ \eqref{eq:cov(phi_d1, phi_d2) - marginal r form III} into $\etacd$  \eqref{eq:def:eta}, we have
\begin{small}
\begin{align}
& \etacd = \Cov \left[ \varphi_{d_1}\left(\Skk_1 \right) - \varphi_0\left(\Skk_1 \right), \varphi_{d_2}\left(\Skk_2 \right) - \varphi_0\left(\Skk_2 \right) \right]  
    \nonumber\\
\label{eq:cov(phi_d1, phi_d2) final}
    = &%\sum_{(\rx{0}, \rx{1}, \rx{2})} \sum_{(\ry{0}, \ry{1}, \ry{2})}
     \sum_{\text{feasible}\, \unrMarg}
    \underbrace{
    \left[\px - \px[0] \right]
    \left[\py - \py[0] \right]
    }_{a_i}
    \underbrace{
        G(\unrMarg)
    }_{b_i}
\end{align}
\end{small}

We have two observations on the above Equation \eqref{eq:cov(phi_d1, phi_d2) final}. First, this $\etacd = \sum_{i} a_i b_i$ is a finite summation because $c \leq T_1$. Hence,  to show $\eta_{c, 2k}^2 = \calO(F_c^{(k)} / {k^2})$, it suffices to bound every term $a_i \cdot b_i$. 
Secondly, by Lemma \ref{prop:cov phi phi, rMarg eq}, $b_i = G(\unrMarg)$ is a weighted average of $\covh$ where the non-negative weights $\sum_{\unrMarg} \prrMarg = 1$. Hence, each $b_i$ is naturally bounded by the upper bound of $\covh$. We conclude that $b_i = \calO(F_c^{(k)})$.
Therefore, it remains to show that $a_i = \calO(k^{-2})$ for every $i$. This is provided by the following lemma. 
\begin{lemma} 
\label{prop:cancellation of the prob}
Fixing integer $d, c \geq 0$, for any tuple of non-negative integers $(x_0, x_1, x_2)$ s.t. $\sum_{i = 0}^2 x_i \leq c$, 
\begin{align}
\label{eq:difference of probability}
   p_{(x_0, x_1, x_2, d, c)} - p_{(x_0, x_1, x_2, 0, c)}  = \calO(\frac{1}{k}).
\end{align}
\end{lemma}
The proof is collected later in Appendix \ref{sec:proof:lemma:eta_c}. This completes the proof of Lemma \ref{lemma:eta_c (d1, d2)}. We remark that though there exists $p_{(x_0, x_1, x_2, d, c)} \asymp 1$ for some $(x_0, x_1, x_2)$, $p_{(x_0, x_1, x_2, d, c)} - p_{(x_0, x_1, x_2, 0, c)}$  is always at the order of $\calO(\frac{1}{k})$. 
\end{proof}

\begin{remark}
This proof has proceeded under Assumption \ref{assumption:low dependency}.
It  can be adapted to a weaker assumption: Assumption \ref{assumption:low dependency (relaxed)}.The according proof using this new assumption will be present in Appendix \ref{sec:proof under relaxed dependency}, where we cannot exactly cancel two $\covh$ with the same $\unr$ but different $d$.
\end{remark}

%%%%%%%%%%%%

We present the proof of two important technical facts in the above proof: Lemma \ref{prop:cov phi phi, rMarg eq} and Lemma \ref{prop:cancellation of the prob}.

\begin{proof}[Proof of Lemma \ref{prop:cov phi phi, rMarg eq}]

First, we derive Equation \eqref{eq:cov(phi_d1, phi_d2) - marginal r form II} from Equation \eqref{eq:cov(phi_d1, phi_d2) - integrated form}:
\begin{align*}
\Cov \left[ \varphi_{d_1}(\Skk_1 ) , \varphi_{d_2}(\Skk_2 ) \right]
= \sum_{\text{feasible }  \unr} \pr \covh.
\end{align*}
Given that $c = |\Skk_1, \Skk_2|$,  $d_1 =|S_1 \cap S_2|$ and $d_2 = |S_3 \cap S_4|$, suppose we randomly sample a feasible $S_1, S_2, S_3, S_4$ from all possible cases, we can use a 9-dimension random variable $\Rrand$ to denote the 4-way overlapping structure of $S_1, S_2, S_3, S_4$. Hence, the the coefficient $\pr$ in \eqref{eq:cov(phi_d1, phi_d2) - integrated form} is $P(\Rrand = \unr|d_1, d_2, c)$.
Then, denote a 6-dimension random variable $\RrandM = (\Rx{0}, \Rx{1}, \Rx{2}, \Ry{0}, \Ry{1}, \Ry{2})$, taking all possible values of $\unrMarg$ given $d_1, d_2, c$. By Bayesian rule, 
\begin{align}
\label{eq:P(r) bayes}
    P(\Rrand = \unr | d_1, d_2, c) = P(\Rrand = \unr | \RrandM = \unrMarg, d_1, d_2, c) P (\RrandM = \unrMarg | d_1, d_2, c).
\end{align} 
Since $S_1, S_2 \subset \Skk_1$, $S_3, S_4 \subset \Skk_2$ and $\Skk_1$ is independent from $\Skk_2$,  $(\Rx{0}, \Rx{1}, \Rx{2})$ are independent from $(\Ry{0}, \Ry{1}, \Ry{2}))$. Hence, we can further decompose $P (\RrandM = \unrMarg | d_1, d_2, c)$ as $P(\Rx{0} = \rx{0}, \Rx{1} = \rx{1}, \Rx{2} = \rx{2})|d_1, c) \cdot
        P(\Ry{0} = \ry{0}, \Ry{1} = \ry{1}, \Ry{2} = \ry{2})|d_2, c)$.
To simplify the notations, we denote 
\begin{align*}
&\pr :=P (\RrandM = \unrMarg | d_1, d_2, c), \\
&\prrMarg := P(\Rrand = \unr | \RrandM = \unrMarg, d_1, d_2, c), \\
&\px := P(\Rx{0} = \rx{0}, \Rx{1} = \rx{1}, \Rx{2} = \rx{2})|d_1, c), \\ 
&\py := P(\Ry{0} = \ry{0}, \Ry{1} = \ry{1}, \Ry{2} = \ry{2})|d_2, c).
\end{align*}
Given $\RrandM$, the distribution of $\Rrand$ does not depend on $d_1, d_2, c$ so we omit the subscript $d_1, d_2, c$ in $\prrMarg$. We also remark that $\sum_{\unrMarg} \prrMarg = 1$ since $\Rrand|\RrandM$ can be viewed as a random variable.
Based on these notations and \eqref{eq:P(r) bayes}, we can rewrite  Equation \eqref{eq:cov(phi_d1, phi_d2) - integrated form} as
\begin{align*}
&\sum_{\text{feasible}\,  \unrMarg} \sum_{\text{feasible}\,\unr} \prrMarg \px \py \covh 
\\
=& \sum_{\text{feasible}\,  \unrMarg} \px \py \underbrace{
    \sum_{\text{feasible}\,\unr} \prrMarg  \covh ]
}_{\text{denote as $g(\unrMarg, d_1, d_2)$}}. 
\end{align*}
This justifies both Equations \eqref{eq:cov(phi_d1, phi_d2) - marginal r form II} and \eqref{eq:def:g}.

Secondly, we show Equation \eqref{eq:PMF for important samples}. Since $\px$ and $\py$ can be analyzed in the same way, our discussion focuses on $\px$, which is boiled down to the distribution of $(\Rx{0}, \Rx{1}, \Rx{2})$.  
Given $c$ and $d_1$, $c$ \emph{Influential Overlaps} can fall into 4 different ``boxes'' in $\Skk_1$: $S_1 \cap S_2$, $S_1 \backslash S_2$, $S_2 \backslash S_2$, and $\Skk_1 \backslash (S_1 \cup S_2)$, with ``box size'' as $d_1$, $k-d_1$, $k-d_1$, $d_1$ respectively. 
The number of samples in each ``box'' follows a hypergeometric distribution. This is illustrated by the following table.
\begin{table}[H]
\centering
\small
\begin{tabular}{lllll}
\hline
Index & 0 & 1 & 2 & 3  \\ \hline
``box'' & $S_1 \cap S_2$ & $S_1  \backslash  S_2$ & $S_2  \backslash S_1$  & $\Skk_1 \backslash (S_1 \cup S_2)$ 
    \\
``box size'' & $d_1$  & $k-d_1$ & $k-d_1$ & $d_1$ 
    \\
\# of \emph{Influential Overlaps} & $\rx{0}$ & $\rx{1}$ & $\rx{2}$  & $c - |\underline{r}|$ 
    \\ \hline
\end{tabular}
\caption{The distribution of \emph{Influential Overlaps} in $\Skk_1$.}
\label{tab:4 different locations in Skk_1}
\end{table}
Hence, the probability mass function of $(\Rx{0}, \Rx{1}, \Rx{2})$: $P \left( \Rx{0} = \rx{0}, \Rx{1} = \rx{1}, \Rx{2} = \rx{2} | d_1, c \right)$  is
\begin{small}
\begin{align}
    %\label{eq:PMF for important samples}
    \px = &
        {\binom{d_1}{\rx{0}} \binom{k-d_1}{\rx{1}} \binom{k-d_1}{\rx{2}} \binom{d_1}{c - \rx{0} - \rx{1} - \rx{2}}}
        {\binom{2k}{c}}^{-1}.
\end{align}
\end{small}
It remains to show that $\px = \calO(k^{\rx{1} + \rx{2} - c})$ for any fixed $c, d_1$. In the following, to simplify the notation, we denote $x_i = \rx{i}$ for $i = 0, 1, 2$ and $x_3 = c - \rx{0} - \rx{1} - \rx1$. Then Equation \eqref{eq:PMF for important samples} can be written as
\begin{align}
\label{eq:PMF rewritten}
 &\frac{c!}{x_0! x_1! x_2! x_3!}
    {\frac{d_1!}{(d_1-x_0)!}
            \frac{(k-d_1)!}{(k-d_1-x_1)!}
            \frac{(k-d_1)!}{(k-d_1-x_2)!}
            \frac{d_1!}{(d_1-x_3)!}
        }/
        {\frac{(2k)!}{(2k-c)!}
        }.
\end{align}
Before the formal justification,  we remark that \eqref{eq:PMF rewritten} looks similar to the probability mass function of a multinomial distribution: $\frac{c!}{x_0! x_1! x_2! x_3!}
    \left(\frac{d_1}{2k}\right)^{x_0}
    \left(\frac{k-d_1}{2k}\right)^{x_1}
    \left(\frac{k-d_1}{2k}\right)^{x_2}
    \left(\frac{d_1}{2k}\right)^{x_3}$, which is obviously $\calO(k^{x_1 + x_2 - c})$. 

We decompose Equation \eqref{eq:PMF rewritten} as a production of three parts,  denoting $\text{Part I}:= \frac{c!}{x_0! x_1! x_2! x_3!}$, $\text{Part II}:=\frac{d_1!}{(d_1-x_0)!}
            \frac{(k-d_1)!}{(k-d_1-x_1)!}
            \frac{(k-d_1)!}{(k-d_1-x_2)!}
            \frac{d_1!}{(d_1-x_3)!}$, $\text{Part III} := \frac{(2k)!}{(2k-c)!}$. 
Since $x_0, x_1, x_2, x_3, c$ are finite, Part I can be viewed as a constant in the asymptotic analysis. 
For Part II,
again, $\frac{d_1!}{(d_1-x_0)!} \frac{d_1!}{(d_1-x_3)!}$ does not depend on $k$ and thus can be treated as a constant. For the rest part:
\begin{align*}
\frac{(k-d_1)!}{(k-d_1-x_1)!}
\frac{(k-d_1)!}{(k-d_1-x_2)!}
   = \left[ \prod_{i = 0}^{x_1-1} (k-d_1 - i) \right]
        \left[ \prod_{i = 0}^{x_2-1} (k-d_1 - i) \right]
   \leq (k - d_1)^{x_1 + x_2}.
\end{align*}
For Part III,
\begin{align*}
\frac{(2k)!}{(2k-c)!} 
    = \prod_{i = 0}^{c-1} (2k - i)
    \geq k^c.
\end{align*}

Combining Part I, II, III, we have 
\begin{align*}
p_{(x_0, x_1, x_2, d_1, c)} =  \calO\left[ (k - d_1)^{x_1 + x_2} / k^c \right]
    = \calO( k^{-c + x_1 + x_2} ).
\end{align*}
This completes the proof.
\end{proof}

\begin{proof}[Proof of Lemma \ref{prop:cancellation of the prob}]
\label{proof:prop:cancellation of the prob} 

To show $p_{(x_0, x_1, x_2, d, c)} - p_{(x_0, x_1, x_2, 0, c)}  = \calO(\frac{1}{k})$, we study two cases separately, where case I is $x_1 + x_2 \leq c-1$ and case II is $x_1 + x_2 = c$. This is motivated by the conclusion of Proposition \ref{prop:cov phi phi, rMarg eq}, 
$\pxx = \calO(k^{x_1 + x_2 - c})$.

% Fixing positive integers $d_1, d_2$ $c$, by Proposition \ref{prop:cov phi phi, rMarg eq}, 
% $\pxx$is $\calO(k^{-x_0 - x_3})$.
% Motivated by this proposition, we partition all the cases of $(x_0, x_1, x_2, x_3)$  as follows:
% %%%%%%% comment %%%%%%%
% \iffalse
% \begin{figure}[H]
%     \centering
%     \includegraphics[scale = 0.8]{figs_append/partition_v3.pdf}
%     \caption{Partition cases}
%     \label{fig:partition}
% \end{figure}
% \fi
% %%%%%%% comment %%%%%%%
% \begin{itemize}
%     \item For $(x_0, x_1, x_2, x_3)$ in $\varphi_{d}(\Skk_1)$, there are 2 parts:  
%         $A := \{(x_0, x_1, x_2, x_3) |x_0 + x_3 = 0\}$;
%         $\bar{A} := \{(x_0, x_1, x_2, x_3) |x_0 + x_3 \geq 1\}$.
    
%     \item For $(x_0, x_1, x_2, x_3)$  in $\varphi_{0}(\Skk_1)$: there is only 1 part: $A = \{(x_0, x_1, x_2, x_3) |x_0 + x_3 = 0\}$.   This is because $|S_1 \cap S_2| = |\Skk_1 \backslash (S_1 \cup S_2)| = 0$. Any $x_0 + x_3 \geq 1$ case happens with 0 probability.
% \end{itemize}

We first study case I. 
For any finite $c, d$, since $x_1 + x_2 \leq c - 1$, by \eqref{eq:PMF for important samples}, $\pxx = \calO(1/k)$. In particular, $\pxx[0] = 0$.
Therefore, 
\begin{align*}
    \pxx - \pxx[0] = \calO(\frac{1}{k}) - 0 = \calO(\frac{1}{k}).
\end{align*}

Secondly, we study case II. For any finite $c, d$, since $x_1 + x_2 = c$, $\pxx \asymp 1$ and $\pxx[0] \asymp 1$. Hence, we can not conclude the order of $\pxx - \pxx[0]$ directly from the order of each term. We need to study $\pxx = {\binom{d}{x_0} \binom{k-d}{x_1} \binom{k-d}{x_2} \binom{d_1}{c - x_0 - x_1 - x_2}} {\binom{2k}{c}}^{-1}$ a bit more carefully.
It is equivalent to showing that 
\begin{align*}
    \left[
    \pxx - \pxx[0]
    \right]
    / \pxx 
  =    \pxx / \pxx - 1
    = \calO(\frac{1}{k}).
\end{align*}
To prove the above, we denote $q(d) := \pxx / \pxx[0]$. It suffices to show that
\begin{align*}
    q(d) =\pxx / \pxx[0] = 1 + \calO(\frac{1}{k}).
\end{align*}
Since $x_0 + x_1 + x_2 + x_3 = c$ and $x_1 + x_2 = c$, we have $x_0 = x_3 = 0$ in $\pxx$. Therefore,
\begin{align}
q(d)
    =&         \frac
        {\binom{d}{0} \binom{k-d}{x_1} \binom{k-d}{x_2} \binom{d}{0}}
        {\binom{2k}{c}}
        /
        \frac
        {\binom{0}{0} \binom{k}{x_1} \binom{k}{x_2} \binom{0}{0}}
        {\binom{2k}{c}} \nonumber
    \\
    =&  \left[\frac{(k-d)!}{(k-d-x_1)! x_1!} \frac{(k-d)!}{(k-d-x_2)! x_2!} 
    \right]
    / 
    \left[\frac{k!}{(k-x_1)! x_1!} / \frac{k!}{(k-x_2)! x_2!}\right].
\end{align}
By direct cancellations of factorials, the above equation can be simplified as 
\begin{align}
\label{eq:q(d)}
    q(d) = \frac{\prod_{i = 0}^{d-1} (k-x_1 -i) \prod_{i = 0}^{d-1} (k-x_2-i)}
    {\prod_{i = 0}^{d-1} (k-i)\prod_{i = 0}^{d-1} (k-i)}
    = \left[
    \prod_{i = 0}^{d-1} \frac{k-x_1 -i}{k-i}
    \right]
    \left[
    \prod_{i = 0}^{d-1} \frac{k-x_2 -i}{k-i} 
    \right].
\end{align}
To upper bound these two products in Equation \eqref{eq:q(d)}, we consider a general argument.
For any integer $b \geq a \geq x \geq 0$, we have 
\begin{align*}
    \frac{a-x}{b-x} \leq ... \leq \frac{a-1}{b-1}  
    \leq \frac{a}{b}.
\end{align*}
Therefore,
\begin{align*}
\left(\frac{a-x}{b-x} \right)^x 
    \leq \frac{a (a-1) ... (a-x+1)}{b (b-1) ... (b-x+1)}
    \leq \left(\frac{a}{b}\right)^x    
\end{align*}
Hence, let $a = k-x_1, b = k, x = d-1$, we can bound $\prod_{i = 0}^{d-1} \frac{k-x_1 -i}{k-i}$ in Equation \eqref{eq:q(d)} as
\begin{align*}
    (\frac{k-d+1-x_1}{k-d+1})^d \leq \prod_{i = 0}^{d-1} \frac{k-x_1 -i}{k-i} \leq (\frac{k-x_1}{k})^d.
\end{align*}
Similarly, we can bound $\prod_{i = 0}^{d-1} \frac{k-x_2 -i}{k-i}$ in Equation \eqref{eq:q(d)} as $(\frac{k-d+1-x_2}{k-d+1})^d \leq \prod_{i = 0}^{d-1} \frac{k-x_2 -i}{k-i} \leq (\frac{k-x_2}{k})^d$. Therefore, the Equation \eqref{eq:q(d)} can be upper and lower bounded as 
\begin{align}
\label{eq:lb and ub of q}
    \left (\frac{k-d+1-x_1}{k-d+1} \right)^d 
    \left (\frac{k-d+1-x_2}{k-d+1} \right)^d 
    \leq 
    q(d)
    \leq 
    \left( \frac{k-x_1}{k} \right)^d 
    \left( \frac{k-x_2}{k} \right)^d. 
\end{align}

We will show both LHS and RHS of Equation \eqref{eq:lb and ub of q} is $1 + \calO (\frac{1}{k})$. First, consider the terms in the RHS of Equation \eqref{eq:lb and ub of q}. Recall that $d, x_1$ are finite compared to $k$, by binomial theorem
\begin{align*}
\left( \frac{k-x_1}{k} \right)^d 
    = \left( 1 - \frac{x_1}{k} \right)^d
    = \sum_{i = 0}^{d} \binom{d}{i} \left(-\frac{x_1}{k}\right)^i
    = 1 + \calO(\frac{1}{k}).
\end{align*}
Similarly, for the other term in the RHS  of Equation \eqref{eq:lb and ub of q}, we achieve
\begin{align*}
    \left( \frac{k-x_2}{k} \right)^d = 1 + \calO(\frac{1}{k}).
\end{align*}
Similarly, for the two terms in the LHS  of Equation \eqref{eq:lb and ub of q}, we have
\begin{align*}
\left( \frac{k-x_1 - d + 1}{k-d+1} \right)^d 
    = 1 + \calO(\frac{1}{k-d+1}) = 1 + \calO(\frac{1}{k}),
    \\
\left( \frac{k-x_2 - d + 1}{k-d+1} \right)^d 
    = 1 + \calO(\frac{1}{k-d+1}) = 1 + \calO(\frac{1}{k}).
\end{align*}
Putting the above analysis together for Equation \eqref{eq:lb and ub of q}, we get
\begin{align*}
    \left[ 1 + \calO(\frac{1}{k})\right]
    \left[ 1 + \calO(\frac{1}{k})\right]
    \leq & q(d)
    \leq 
    \left[ 1 + \calO(\frac{1}{k})\right]
    \left[ 1 + \calO(\frac{1}{k})\right]
\\
\implies
    \left[ 1 + \calO(\frac{1}{k})\right]
    \leq & q(d)
    \leq 
    \left[ 1 + \calO(\frac{1}{k})\right].
\end{align*}

This completes the proof.
\end{proof}

% \end{document}

\subsection{Proof of Lemma \ref{lemma:bound every eta_c^2}}
\label{sec:proof:lemma:bound every eta_c^2}

\begin{proof}[Proof of Lemma \ref{lemma:bound every eta_c^2}]
\label{proof:lemma:bound every eta_c^2}
By \ref{eq:def:eta}, given $\Skk_1, \Skk_2 s.t. |\Skk_1 \cap  \Skk_2 | = c$,
\begin{align*}
    \etacd 
    &= \Cov \left[ \varphi_{d_1}\left(\Skk_1 \right) - \varphi_0\left(\Skk_1 \right), \varphi_{d_2}\left(\Skk_2 \right) - \varphi_0\left(\Skk_2 \right) \right] \\
    &\leq  \Cov \left[ \varphi_{d_1}\left(\Skk_1 \right) , \varphi_{d_2}\left(\Skk_2 \right)  \right]
    + \Cov \left[ \varphi_{0}\left(\Skk_1 \right) , \varphi_{0}\left(\Skk_2 \right)  \right],
\end{align*}
where the last inequality is by the non-negativity of $\covh$.
By the definition of $\varphi_d$ in Equation \eqref{eq:def:phi}, the RHS of above equation is upper bounded by
\begin{align*}
    & 2 \max_{S_1, S_2 \subset \Skk_1, s.t. | S_1 \cap S_2 | \leq d_1, S_3, S_4 \subset \Skk_2, s.t. | S_3 \cap S_4 | \leq d_2} \covh = \calO(\fk).
\end{align*}

\end{proof}

\subsection{Proof of Lemma \ref{lemma:Finite truncated variance II}}
\label{sec:proof:lemma:Finite truncated variance II}

\begin{proof}[Proof of Lemma \ref{lemma:Finite truncated variance II}]
\label{proof:lemma:Finite truncated variance II}
We apply the strategies we used in the proof of  Lemma \ref{lemma:Finite truncated variance I}. The truncation parameter is $T_2 = \floor{\frac{1}{\epsilon}} + 1$. 
Recall in  Lemma \ref{lemma:Finite truncated variance I}, $\Var(\vu) = \sum_{c = 1}^{2k} \binom{n}{2k}^{-1} \binom{2k}{c} \binom{n-2k}{2k-c} \sigma_{c, 2k}^2$, where
\begin{align*}
    \sigma_{c, 2k}^2 = \sum_{d_1=1}^{k} \sum_{d_1=1}^{k} w_{d_1} w_{d_2} \etacd, \,\, \text{for } c = 1, 2, ..., T_1.
\end{align*}

We decompose  $\sigma_{c, 2k}^2$ into three parts:
\begin{align}
\label{eq:decompose sigma_c^2 as A+2B+C}
\sigma_{c, 2k}^2 = 
    & \underbrace{ \sum_{d_1=1}^{T_2} \sum_{d_2=1}^{T_2} w_{d_1} w_{d_2} \etacd}_
    {A}  
     + 2
    \underbrace{
      \sum_{d_1= 1}^{T_2} \sum_{d_2=T_2 + 1}^{k} w_{d_1} w_{d_2} \etacd}_{B} \nonumber
    \\
     &+ \underbrace{
     \sum_{d_1=T_2 + 1}^{k} \sum_{d_2=T_2 + 1}^{k} w_{d_1} w_{d_2} \etacd}_
    {C}.    
\end{align}
Similarly, denote 
\begin{align}
\label{eq:check A in sigma_c^2}
\Check{A} :=\widecheck{\sigma}_{c, 2k, (T_2)}^2 = \sum_{d_1=1}^{T_2} \sum_{d_2=1}^{T_2} w_{d_1} w_{d_2} \widecheck{\eta}_{c, 2k}^2 (d_1, d_2),    
\end{align}
where $\widecheck{\eta}_{c, 2k}(d_1, d_2)$ is the upper bound given in Lemma \ref{lemma:eta_c (d1, d2)}.
To prove this lemma, it suffices to show
\begin{align}
\label{eq:lim 2B+C / A}
    \lim_{k \to \infty} \frac{2B+C}{\check{A}} = 0.
\end{align}
It remains to bound $\check{A}, B, C$ as 
\begin{align}
\label{eq:bound A, B, C}
    \check{A} = \asymp \frac{w_1^2 F_c}{k^2}, 
    B = \calO\left( w_1 \check{w}_{T_2+1} F_c \right), 
    C = \calO\left( \check{w}_{T_2+1}^2 F_c \right),
\end{align}
where $\check{w}_d = \calO(\frac{k^{2d}}{d! n^d})$ is the rough upper bound of $w_d$ in \eqref{eq:weights for phi's: 2}.
We need to quantify two parts, the coefficients $w_{d_1} w_{d_2}$ and the covariance $\etacd$.
Let us fix one $c \leq T_1$ and  first quantify $\etacd$. By Lemma \ref{lemma:eta_c (d1, d2)} and \ref{lemma:bound every eta_c^2}, we have 
\begin{align}
    \etacd &= \calO(\frac{1}{k^2} F_c), \,\,\text{for } c \leq T_1, d_1, d_2 \leq T_2, \\
\label{eq:rough upper bound for etac}
    \etacd &= \calO( F_c), \,\,\text{for } c \leq 2k, d_1, d_2 \leq k. 
\end{align}
By Proposition \ref{lemma:bound each sigma_c^2 for finite c}, we have $A = \sigma_{c, 2k, (T_2)}^2 = \calO(\frac{k^2}{n^2}F_c)$. Since $\check{A}$ is the upper bound of $A$ given in Proposition \ref{lemma:bound each sigma_c^2 for finite c},  $\check{A} = \calO(\frac{k^2}{n^2}F_c)$.
For $B, C$, we upper bound $\etacd$ by $\calO( F_c)$ in Equation \eqref{eq:rough upper bound for etac}. Hence, we can reduce the analysis for both coefficients and covariance to the analysis on only coefficients $w_d$, for
\begin{align*}
    B &= \calO(F_c) \left[ \sum_{d_1= 1}^{T_2} w_{d_1} \right] \left[\sum_{d_2=T_2 + 1}^{k}  w_{d_2} \right]\\
    C &= \calO(F_c) \left[ \sum_{d_1= T_2 + 1}^{k} w_{d_1} \right] \left[\sum_{d_2=T_2 + 1}^{k}  w_{d_2} \right].
\end{align*}
To be more specific, it remains to show that
\begin{align*}
     \sum_{d= 1}^{T_2}  w_{d}  = \calO(w_1), 
     \quad \sum_{d=T_2 + 1}^{k}  w_d = \calO(\check{w}_{T_2+1}).
\end{align*}
where $\check{w}_d = \calO(\frac{k^{2d}}{d! n^d})$ is the rough upper bound of $w_d$ in \eqref{eq:weights for phi's: 2}.

For $\sum_{d = 1}^{T_2} w_d$, by Equation \eqref{eq:weights for phi's: 2}, each $w_d = [1 + o(1)] \frac{k^{2d}}{d!n^d} \leq [1 + o(1)] \frac{k^{2d}}{n^d}$. The common ratio of geometric decay is $k^2/n = o(1)$. Therefore, the first term $w_1$ dominates $\sum_{d = 1}^{T_2} w_d$.
For $\sum_{d=T_2 + 1}^{k}  w_d$, by Equation \eqref{eq:weights for phi's: 2}, we have each $w_d = \calO (\frac{k^{2d}}{d!n^d}) = \calO (\frac{k^{2d}}{n^d})$. Similarly, by geometric decay with common ratio $k^2/n$,
\begin{align*}
    \sum_{d=T_2 + 1}^{k}  w_d \leq 
    \sum_{d=T_2 + 1}^{k}  \calO (\frac{k^{2d}}{n^d}) = \calO \left( (\frac{k^2}{n})^{T_2+1} \right). 
\end{align*}
Hence we define $\check{w}_{T_2+1} \asymp \frac{k^2}{n})^{T_2+1}$. Then $\sum_{d=T_2 + 1}^{k}  w_d = \calO(\check{w}_{T_2 + 1}$).

We have proved the bounds in Equation \eqref{eq:bound A, B, C}. Then plug Equation \eqref{eq:bound A, B, C} into the LHS of Equation \eqref{eq:lim 2B+C / A}, we can conclude that
\begin{align}
\label{eq:2B+C / A}
\frac{2B+C}{\check{A}}
    =  \calO \left(\frac{k^2 (w_1 \check{w}_{T_2+1} + \check{w}_{T_2+1}^2)}{w_1^2}\right).
\end{align}
For \eqref{eq:2B+C / A}, plugging in $T_2 = \floor{1/\epsilon} + 1$ and the the upper bound of $w_d = [1 + o(1)] \frac{k^{2d}}{d!n^d}$ and $\check{w_d} = \calO(\frac{k^{2d}}{d!n^d})$ from Equation \eqref{eq:weights for phi's: 2} and \eqref{eq:weights for phi's: 2}, we conclude 
\begin{align*}
    \frac{2B+C}{\check{A}}
    =  \calO \left(k^2 n^{-2\epsilon (\floor{1/\epsilon}+1)}\right) = \calO \left(k^2 n^{-2\epsilon (\floor{1/\epsilon} + 1)}\right)
    = o(k^2 n^{-2}) = o(1).
\end{align*}

This completes the proof.
\end{proof}

%%%%%%%%

\section{Relaxation of Assumption \ref{assumption:low dependency} and According Proof}
\label{sec:proof under relaxed dependency}

In this section, we first present Assumption \ref{assumption:low dependency (relaxed)} as a relaxation of Assumption \ref{assumption:low dependency}, Then, we show that the technical lemmas can be proved under our relaxed assumption. Since $\covh$ depends on all 11 instead of only 9 \emph{DoF}, we will use $\covhh$ to denote $\covh$ throughout this section.

\subsection{Assumption \ref{assumption:low dependency (relaxed)}}
% Moreover, we present Assumption \ref{assumption:low dependency (relaxed)} as a relaxation of Assumption \ref{assumption:low dependency}. 
% The alternative proof with Assumption \ref{assumption:low dependency} replaced by Assumption \ref{assumption:low dependency (relaxed)} is deferred to Appendix \ref{sec:proof under relaxed dependency}.
Recall that Assumption \ref{assumption:low dependency} implies that $\covhh$ only depends on $\unr$, and thus has 9 \emph{DoF}. 
Though Assumption \ref{assumption:low dependency} is valid in Example \ref{example:linear kernel for low depent. assump}, it is still too restrictive in practice. 
Our new assumption will allow $\covhh$ have 11 \emph{DoF} with some restriction on the impact from \emph{Non-influential Overlaps}, i.e.,samples not in $\Skk_1 \cap \Skk_2$.

In the following Equation \eqref{eq:def:rhort}, We introduce a 9-\emph{DoF} benchmark: $\rhort$ , as the analogue to the $\rhor$ \eqref{eq:def F(r)}. This is a bridge between 9 \emph{DoF} and 11 \emph{DoF}.

\begin{definition}[$\rhort$]
$\forall$ feasible $\underline{r}$, we take $d_1^* =  \rx{0} = \sum_{j = 0}^{2} r_{0j}$ and $d_2^* = \ry{0} = \sum_{i = 0}^{2} r_{i0}$. Then, we define
    \label{def:F(r) in relaxed low dependency}
    \begin{align}
    \label{eq:def:rhort}
        \rhort :=
         \rho(\unr, d_1^*, d_2^*).
    \end{align}
\end{definition}

Since $\rx{0} = |(S_1 \cap S_2) \cap (S_3 \cup S_4)|$ (by \eqref{eq:def:r vec}) and $d_1 = |S_1 \cap S_2|$,  we always have $d_1 \geq \rx{0}$.
Therefore, fixing $\unr$, the constraint $d_1 = \rx{0}$ in \eqref{eq:def:rhort} means that the benchmark $\rhort$'s $d_1$ already takes a smallest feasible value. Similarly, $d_2$ also takes the smallest feasible value.
% Thus, given any feasible $\unr$,  $\rhort$ in Equation \eqref{eq:def:rhort} has the smallest $d_1, d_2$ among all $\covhh$. 
Since $\rhort$ is a ``benchmark'' of $\covhh$ (given $\unr$),  we impose the further relaxed assumption, where $\covhh$ deviates from benchmark $\rhort$ based on its $d_1$ and $d_2$.

\begin{assump}[Relaxation of Assumption \ref{assumption:low dependency}]
    \label{assumption:low dependency (relaxed)}

For any finite $d_1, d_2 \geq 0$, $\underline{r}$, $s.t. d_1 \geq \rx{0} , d_2 \geq \ry{0} $, there exist  constant $B$ and $B(\unr)$ s.t. $B(\underline{r}) \leq B < \infty$, 
\begin{align}
\label{eq:low dependency (relaxed) I: finite simpled}
     \covhh
      = \left[1 +  B(\underline{r}) \frac{d_1 - \rx{0} + d_2 - \ry{0}}{k}  + \calO(\frac{1}{k^2})  \right] 
      \rhort.
\end{align}

Besides, for any $d_1 \geq \rx{0}, d_2 \geq \ry{0}$,
\begin{align}
\label{eq:low dependency (relaxed) II: bound}
\covhh 
      \leq  B \rhort.  
\end{align}

\end{assump}

In Equation \eqref{eq:low dependency (relaxed) I: finite simpled}, we refer the benchmark $\rhort$ as the \emph{main effect}, capturing the overlapping between $\Skk_1$ and $\Skk_2$ while the rest $\calO(\frac{1}{k}) \rhort$ part as \emph{additional effect}, capturing the overlapping within $\Skk_1$ and $\Skk_2$.
We interpret this assumption as follows, where we mainly focus on $d_1 = |S_1 \cap S_2|$. A similar analysis can be performed on $d_1$.
First, $S_1 \cup S_2$ can decomposed into two sets: $\calA = (S_1 \cup S_2) \cap (\Skk_1 \cap \Skk_2)$, i.e., the set of \emph{Influential Overlaps}  and  $\calB = (S_1 \cup S_2) \backslash (\Skk_1 \cap \Skk_2)$. Moreover, $\calB$ can be decomposed into two subsets $\calB_1 = (S_1 \cap S_2) \backslash (\Skk_1 \cap \Skk_2)$ and $\calB_2 = ((S_1 \cup S_2)  \backslash (S_1 \cap S_2) )\backslash (\Skk_1 \cap \Skk_2)$.
We note that $|\calB_1| = 0$ for the benchmark $\rhort$. 
When we fix $\unr$ while increase $d_1$, $\calA$ does not change while $|\calB_1|$ becomes larger and $|\calB_2|$ becomes smaller. This causes $\covhh$ deviates from $\rhort$. 
Secondly, since $\unr$ is fixed and $\calB$ does not involve any \emph{influential overlaps}, we expect the change of $\calB_1$ and $\calB_2$ by increasing $d_1$ has a lower order, i.e., $\calO(\frac{1}{k})$, impact on $\covhh$.
% First, we the two extra \emph{DoF}: $d_1, d_2$ contribute to $\covhh$.
% In $\Cov[h(S_1), h(S_2)]$,  $h(S_1)$ is a symmetric kernel regarding set $S_1$. However, in $\covhh$,  $h(S_1) h(S_2)$ is asymmetric regarding set $S_1 \cup S_2$. The asymmetry comes from the heterogeneity between two subsets in $S_1 \cup S_2$: 1) size-$d_1$ set $S_1 \cap S_2$ and 2) size-$2(k-d_1)$ set $(S_1 \cup S_2)\backslash (S_1 \cap S_2)$.
% Second, recall that we have defined \emph{Influential Overlaps} as the samples in $\Skk_1 \cap \Skk_2$; here we refer to \emph{Non-influential Overlaps} in $S_1 \cup S_2$ as the samples in $(S_1 \cup S_2) \backslash (\Skk_1 \cap \Skk_2)$. 
% For the benchmark $\rhort$, none of \emph{Non-influential Overlaps} locate in $S_1 \cap S_2$.
% However, when fixing the \emph{Influential Overlaps} but increasing $d_1$, $S_1 \cap S_2$ accommodates more and more \emph{Non-influential Overlaps}, which renders the new $\covhh$ (with increased $d_1$) deviates from the benchmark $\rhort$, i.e., the  \emph{main effect}.
% Third, we expect such ``deviation'' be a ``small perturbation'' , i.e., at a lower order of $\covhh$, since it comes from \emph{Non-influential Overlaps} instead of \emph{Influential Overlaps}.
% In particular, suppose we flip only one \emph{Non-influential Overlaps} from  $(S_1\cup S_2) \backslash (S_1 \cap S_2)$ to $S_1 \cap S_2$, i.e., $d_1$ increases by 1, the change of $\covhh$ will be of $\calO(\frac{1}{k})$ proportion. 
Hence, fixing $\unr$, for finite $d_1$, i.e $d_1/k \to 0$, the proportion of $\covhh$'s ``deviation'' from $\rhort$ is at a scale of 
% \begin{align}
% \label{eq:non-influential overlaps influence}
% \frac{\text{\# of Non-influential overlaps in $S_1 \cap S_2$}}{\text{\# of Non-influential overlaps in $S_1 \cup S_2$}}
% := 
%     \frac{|(S_1\cap S_2) \backslash (\Skk_1 \cap \Skk_2)|}
%         {|(S_1\cup S_2) \backslash (\Skk_1 \cap \Skk_2)|}
%     = [1 + o(1)] \frac{d_1' + d_2'}{k}.
% \end{align}
\begin{align}
\label{eq:non-influential overlaps influence}
[1 + o(1)] \frac{|\calB_1|}{|\calB|} = 
    [1 + o(1)] \frac{|(S_1\cap S_2) \backslash (\Skk_1 \cap \Skk_2)|}
        {|(S_1\cup S_2)|}
    = [1 + o(1)] \frac{d_1'}{k}.
\end{align}
Similar analysis can also be performed on $d_2$ and $S_3, S_4$. 
In addition, Equation \eqref{eq:low dependency (relaxed) II: bound} bounds the \emph{additional effect} up to the same order term as the \emph{main effect}. Again, this is because $\covhh$ should be dominated by the \emph{influential overlaps}, i.e., the samples in $\Skk_1 \cap \Skk_2$.
Note that in previous Example \ref{example:linear kernel for low depent. assump}, there is only the \emph{main effect}, i.e., $B(\unr) \equiv 0$ and thus
Assumption \ref{assumption:low dependency (relaxed)} degenerates to Assumption \ref{assumption:low dependency}.

\subsection{Proof under Assumption \ref{assumption:low dependency (relaxed)}}

In our previous proof, only two fundamental lemmas directly rely on Assumption \ref{assumption:low dependency}: 
Lemma \ref{lemma:eta_c (d1, d2)} (the precise bound for $\etacd$) and Lemma \ref{lemma:bound every eta_c^2} (the rough bound for $\etacd$). Based on these two lemmas, we can derive the upper bounds of $\sigma_{c, 2k}^2$ and hence upper bound $\Var(\vu)$ (see the proof roadmap in Appendix \ref{sec:append:roadmap}).
Therefore, when Assumption \ref{assumption:low dependency} is relaxed to Assumption \ref{assumption:low dependency (relaxed)}, it suffices to show the results in Lemma \ref{lemma:eta_c (d1, d2)} and \ref{lemma:bound every eta_c^2}.

First, it is trivial to validate a relaxed Lemma \ref{lemma:bound every eta_c^2} under Assumption \ref{assumption:low dependency (relaxed)}.  Assumption \ref{assumption:low dependency (relaxed)} does not change the upper bound of $\covhh$, which is $\fk_c = \Cov[h(S')^2, h(S'')^2]$ s.t. $|S' \cap S''| = c$ \eqref{eq:F_c as upper bound}. 
The proof of Lemma \ref{lemma:bound every eta_c^2} in Appendix \ref{sec:proof:lemma:bound every eta_c^2} only requires the upper bound $\fk_c$, thus it still works.

Second, we need to adapt Proof of Lemma \ref{lemma:eta_c (d1, d2)} in Appendix \ref{sec:proof:lemma:eta_c}. $\covhh$ can no longer be represented as $\rhor$. Thus, $\covhh - \rho(\unr, d_1', d_2')$ is not necessary 0 for $(d_1, d_2) \neq (d_1', d_2')$.

\begin{proof}[Proof of Lemma \ref{lemma:eta_c (d1, d2)} under \alterassump]

We adopt the beginning part of the proof in Appendix \ref{sec:proof:lemma:eta_c} until Lemma \ref{prop:cov phi phi, rMarg eq}.
We note that  Lemma \ref{prop:cov phi phi, rMarg eq} does not rely on Assumption \ref{assumption:low dependency}. Hence, we still have Equation \eqref{eq:cov(phi_d1, phi_d2) - marginal r form II}:
\begin{small}
\begin{align*}
 \Cov \left[ \varphi_{d_1}\left(\Skk_1 \right) , \varphi_{d_2}\left(\Skk_2 \right)  \right] 
  =  \sum_{\text{feasible}\,\, \unrMarg} 
    \px
    \py 
     g(\unrMarg, d_1, d_2),
\end{align*}
\end{small}
where $g(\unrMarg, d_1, d_2)$ is given in Equation \eqref{eq:def:g}.{eq:def:g}
Since $\covhh$ in $ g(\unrMarg, d_1, d_2)$  \eqref{eq:def:g} depends $d_1, d_2$, we cannot further simplify $g(\unrMarg, d_1, d_2)$ as $G(\unrMarg)$. 
Further, we denote 
\begin{align*}
\pg{d_1}{d_2}
    &:=  \px[d_1]
    \py[d_2]   g(\unrMarg, d_1, d_2, c); \\
\dpg
    &:= 
    \pg{d_1}{d_2} - \pg{0}{d_2} -\pg{d_1}{0} + \pg{0}{0}.
\end{align*}
Then $\etacd$ can be represented as
\begin{align}
\label{eq:etacd, sumx sumy}
\etacd = 
    \sum_{(\rx{0}, \rx{1}, \rx{2})} \sum_{(\ry{0}, \ry{1}, \ry{2})}\dpg.
\end{align}

We apply the strategy used in the proof of Lemma \ref{prop:cancellation of the prob}, partitioning the summation as follows:
\begin{align}
\label{eq:etacd, 4 partitions}
    \etacd = 
    \left(\sum_{\rx{1} + \rx{2} = c} + \sum_{\rx{1} + \rx{2}  \leq c-1}\right)
    \left(\sum_{\ry{1} + \ry{2}  = c} +  \sum_{\ry{1} + \ry{2}  \leq c - 1} \right)
    \dpg.
\end{align}
There are 4 cases. Case A: $\rx{1} + \rx{2} = c$ and $\ry{1} + \ry{2} = c$; case B: $\rx{1} + \rx{2} = c$ and $\ry{1} + \ry{2} \leq c-1$; case C: $\rx{1} + \rx{2} \leq c-1$ and $\ry{1} + \ry{2} = c$; case D: $\rx{1} + \rx{2} \leq c-1$ and $\ry{1} + \ry{2} \leq c-1$. 
Since $c$ is finite, $\rx{i}$'s and $\ry{j}$'s are also finite. Thus,  \eqref{eq:etacd, 4 partitions} is a finite summation. 
To show $\etacd = \calO(\frac{1}{k^2} \fk_c)$, it suffices to show that  the summand $\dpg = \calO(\frac{1}{k^2} \fk_c)$ in all 4 cases. 

First, we study case A. For the $g(\unrMarg, d_1, d_2)$ defined in Equation \eqref{eq:def:g}, by approximation of $\rho$ in Assumption \ref{assumption:low dependency (relaxed)}: $     \covhh
      = \left[1 +  B(\unr) \frac{d_1' + d_2'}{k}  + \calO(\frac{1}{k^2})  \right] 
      \rhort ,$  we have  
\begin{align*}
    g(\unrMarg, d_1, d_2) = [1 + \frac{d_1 + d_2}{k} B(\unr)  + \calO(\frac{1}{k^2})] g(\unrMarg, 0, 0),
\end{align*}
Therefore, $\dpg$ can be simplified as 
\begin{align*}
\begin{split}
\dpg = 
    &g(\unrMarg, 0, 0) \big\{( \px[d_1] - \px[0] ) ( \py[d_2] - \py[0] ) \\
    &  + \frac{d_1 B(\unr) }{k} \px[d_1] (\py[d_2] - \py[0]) \\ 
    & + \frac{d_2 B(\unr) }{k} \py[d_2] (\px[d_1] - \px[d_2])  
     + \calO(\frac{1}{k^2}) \big\} ,
\end{split}
\end{align*}
By Lemma \ref{prop:cancellation of the prob}, $\px[d_1] - \px[0] = \calO(1/k)$, $\py[d_2] - \py[0] = \calO(1/k)$. Besides, $g(\unrMarg, 0, 0, c) \leq \fk_c$.
Thus, $\dpg = \calO(\frac{1}{k^2} \fk_c)$.

Secondly, we study case $B$. Recall that $\py[0] = 0$ by \eqref{eq:PMF for important samples}. Therefore, $\pg{d_1}{0} = \pg{0}{0} = 0$. Hence, $\dpg$ can be simplified as
\begin{align}
\label{eq:Bi}
    \dpg = \pg{d_1}{d_2}- \pg{0}{d_2}.
\end{align}
By Assumption \ref{assumption:low dependency (relaxed)}, we can approximate $g(\unrMarg, d_1, d_2)$ as 
\begin{align*}
    g(\unrMarg, d_1, d_2) = [1 + \frac{d_1}{k} B(\unr)  + \calO(\frac{1}{k^2})] g(\unrMarg, 0, d_2).
\end{align*}
Then, plug the above approximation into Equation \eqref{eq:Bi}:
\begin{align*}
    \dpg =& g(\unrMarg, 0, d_2) \big [
    \py[d_2] (\px[d_1] - \px[0]) \nonumber \\
    &+ \frac{d_1 B(\unr)}{k} \px[d_1] \py[d_2] + \calO(\frac{1}{k^2}) \big].
\end{align*}
Similarly, we have $\py[d_2] = \calO(1/k)$ and $\px[d_1] - p_0 (\unx) = \calO(1/k)$. Besides, $g(\unrMarg, 0, d_2, c) \leq \fk_c$. Therefore, we conclude that
$\dpg = \calO(\frac{1}{k^2} \fk_c) $.

Thirdly, by a similarly analysis in case B, we can bound $\dpg = \calO(\frac{1}{k^2} \fk_c)$ in case C.

Finally, we study case D. Since $\rx{1} + \rx{2} \leq c-1$ and $\ry{1} + \ry{2}  \leq c - 1$,  $\px[0] = \py[0] = 0$. 
Therefore,
\begin{align*}
\dpg{d_1}{d_2}
    &= \px[d_1]  \py[d_2]   g(\unrMarg, d_1, d_2).
\end{align*}
Since $\px[d_1] = \calO(1/k)$, $ \py[d_2] = \calO(1/k)$, and  $g(\unrMarg, d_1, d_2) \leq \fk_c$. 
Therefore, we can bound $\dpg = \calO(\frac{1}{k^2} \fk_c)$ in case D.

This completes the proof.
\end{proof}

\section{Low Order Kernel \texorpdfstring{$h$}{h} Illustration}
\label{sec:append:low order kernel}

The purpose of this appendix is two-fold.
First,  we present a simple analysis of ratio consistency under a oversimplified kernel $h$, which illustrate the basic idea to unitize the \emph{double U-statistic} structure. 
Secondly, we discuss the gap between a simple kernel and a general kernel $h$ in Appendix \ref{sec:difficulities and tools for general h}, which illustrates the motivation of imposing our assumptions on $\Cov[h(S_1) h(S_2), h(S_3) h(S_4)]$ (see assumptions in Section \ref{sec:assumptions and notations}). 

% \documentclass[../main.tex]{subfiles}

% \begin{document}

% The analysis for low order kernel.

\subsection{Linear Average Kernel \texorpdfstring{$h$}{h}}
\label{sec:linear average kernel}

\begin{theorem}
\label{thm:linear kernel, ratio consitency}
$X_1, X_2, ..., X_n$ iid distribution F, s.t. $\E(X_1) = 0, \gamma^2 := \Var(X_1) > 0$.
Suppose kernel function $ h(X_1, ..., X_k) = \frac{1}{k} \sum_{i = 1}^k X_i$
Then, we have the ratio consistency of the estimator,
\begin{align*}
    \frac{\Var (\hat V_u)}{\left(\E(\hat V_u) \right)} = \calO(\frac{1}{n}).
\end{align*}
\end{theorem}

To prove theorem \ref{thm:linear kernel, ratio consitency}, it suffices to prove the following lemma.

\begin{lemma}
\label{lemma:linear kernel, E(Vu) and Var(Vu)}
Under the conditions in Theorem \ref{thm:linear kernel, ratio consitency}, we have
\begin{align*}
    \E(\hat V_u) = \Omega(\frac{1}{n}), \,\, \text{and} \,\,
    \Var (\hat V_u) = \calO (\frac{1}{n^3}),
\end{align*}
where $\Omega$ indicates an asymptotic lower bound (see Table \ref{tab:summary of notations}). 
\end{lemma}
We remark that we are able to generalize Lemma \ref{lemma:linear kernel, E(Vu) and Var(Vu)} from the linear average kernel $h$ to the intrinsic low order kernel $h$ (Definition \ref{def:low order kernel}).
After showing the proof of Lemma \ref{lemma:linear kernel, E(Vu) and Var(Vu)}, we will summarize the benefits of the low order structure of kernel $h$. 
Then, we present the potential difficulties and solutions when low order kernel assumption no longer holds.

\begin{proof}
\textbf{Part I: show $\E(\hat V_u) = \Omega(\frac{1}{n})$.}

For linear average kernel, $U_n = \binom{n}{k}^{-1} \sum_{S_i} h(S_i) = \frac{1}{n} \sum_{i = 1}^n X_i$,  $V_u$ is an unbiased estimator of $\Var (U_u)$.
Therefore,
\begin{align*}
    \E(\hat V_u) = \Var(U_n) 
    = \Var \left( \frac{1}{n} \sum_{i = 1}^n X_i \right) = \frac{1}{n} \Var(X_1) = \frac{\gamma^2}{n} 
    \asymp \frac{1}{n}.
\end{align*}

\textbf{Part II: show $\Var (\hat V_u) = \calO (\frac{1}{n^3})$.}

We perform a 3-step analysis. 
By step 1 and 2, the order of $\vu$ becomes $\calO (\frac{k^2}{n} \frac{1}{k^2}) = \calO (\frac{1}{n})$. 
Further by the coefficient of leading variance in $\Var(\vu)$ in step 3, we can achieve 
$
    \Var (\hat V_u) = [\calO (\frac{k^2}{n} \frac{1}{k^2})]^2 \calO(\frac{1}{n}) = \calO(\frac{1}{n^3}).
$
The rigorous analysis is presented as follows.

\textbf{Step 1: double U-statistic structure}

We reiterate the double U-statistic structure of $\vu$  (Proposition \ref{prop: double U-statistic: alternative form of phi}) as follows
\begin{align}
\label{eq:linear kernel layer 1: varphi}
\hat V_u =  \binom{n}{2k}^{-1} \sum_{\Skk \subseteq \mathcal{X}_{n}} \psi\left(\Skk \right), \,\,
\psi\left(\Skk \right) = \sum_{d = 1}^{k} w_d \left[\varphi_d \left(\Skk \right) - \varphi_0 \left(\Skk \right)\right],
\end{align}
where $w_1 = [1 + o(1)] (\frac{k^2}{n}) $, $w_d = \calO \left [\frac{1}{d!}(\frac{k^2}{n})^d \right], \forall d = 1, 2, ...., k$; $\varphi_d$ is the asymptotic kernel U-stat defined in Equation \eqref{eq:def:phi}.  
We observe that $w_d$ decays with $d$ fast. 
The following analysis shows that the first summand in Equation \eqref{eq:linear kernel layer 1: varphi} dominates the sum.
% We leave the the details in the structure of $\varphi_d \left(\Skk \right)$ for step 2.

% By Proposition \ref{prop: double U-statistic: alternative form of phi}, when $d$ grows, the weight $w_d$ decays faster than the geometric rate. 
% We expect that the first summand in Equation \eqref{eq:linear kernel layer 1: varphi} may dominate the sum and thus the order of $\psi(\Skk)$ is reduced to the order of $w_1$: $\calO(\frac{k^2}{n})$. 
% A rigorous justification requires bounding $\varphi_d (\Skk ) - \varphi_0 (\Skk )$ as well, which is provided in the following step 2 analysis.

\textbf{Step 2: analyze $\varphi_d \left(\Skk \right) - \varphi_0 \left(\Skk \right)$}

We further investigate the cancellation pattern $\varphi_d(\Skk ) - \varphi_0 (\Skk )$ in Equation \eqref{eq:linear kernel layer 1: varphi} by the following proposition.
\begin{proposition}
\label{prop:layer2 cancel in varphi}
Under linear average structure of $h$, i.e., $h(X_1, ..., X_k) = \frac{1}{k} \sum_{i = 1}^k X_i$, 
\begin{align}
\label{eq:phi_d - phi_0 for linear}
\varphi_d \left(\Skk \right) - \varphi_0 \left(\Skk \right)
    = \frac{d}{k^2} \left( \overline{X^2}^{(2k)} - \overline{XY}^{(2k)}\right),
\end{align}
where $\overline{X^2}^{(2k)}:= \frac{1}{2k} \sum_{X_i \in \Skk} X_i^2$ and $\overline{XY}^{(2k)} :=\frac{2}{2k(2k-1)} \sum_{X_i,X_j \in \Skk, j>i}  X_i X_j$.
\end{proposition}

Note that in Equation \eqref{eq:phi_d - phi_0 for linear}, the coefficient of $\overline{X^2}^{(2k)}$ and $\overline{XY}^{(2k)}$, $\frac{d}{k^2}$, is resulting from the cancellation.
The proof of above proposition is collected in Appendix \ref{sec:proof:linear kernel prop}, by plugging the explicit form of $h$ into $\varphi(\Skk)$ (Equation \eqref{eq:def:phi}). Combining the results in Equation \eqref{eq:linear kernel layer 1: varphi}, \eqref{eq:linear kernel, cancel in varphi} and \eqref{eq:weights for phi's: 2}, we have 
\begin{align}
\label{eq:linear kernel:layer1 and 2 psi: combine}
\psi\left(\Skk \right) 
    = \frac{1}{k^2} 
    \left( \overline{X^2}^{(2k)} - \overline{XY}^{(2k)}\right)
    \sum_{d = 1}^{k} d w_d
    = [1 + o(1)] \frac{1}{n}  
    \left( \overline{X^2}^{(2k)} - \overline{XY}^{(2k)}\right).
\end{align}

\textbf{Step3: show the leading variance term}

Based on Equation \eqref{eq:linear kernel:layer1 and 2 psi: combine}, we use  classical approach in the U-statistic literature \citep{hoeffding1948}, showing $\sigma_{1, 2k}^2$ dominating $\Var (\hat V_u) = \binom{n}{2k}^{-1} \sum_{c=1}^{2k} \binom{2k}{c} \binom{n-2k}{2k-c} \sigma_{c, 2k}^2$, i.e.,
\begin{align*}
    \Var(\vu) = [1 + o(1)] \frac{4k^2}{n} \sigma_{1, 2k}^2.
\end{align*}
To get our desired upper bound $\Var(\vu) = \calO(\frac{1}{n^3})$, it suffices to  bound $\sigma_{1, 2k}^2$ by $\calO(1/(k^2 n^2))$ and further check $\sigma_{2k, 2k}^2 / (2k \sigma_{1, 2k}^2) = \calO(1)$ \citep{Romano2020}. 
By Proposition \ref{prop:layer3 approach 2, sigma1 and sigma2k} in Appendix \ref{sec:proof:linear kernel prop} we have
\begin{align}
\label{eq:layer3 approach 2, sigma1 and sigma2k}
    \Cov \left(\overline{X^2}^{(2k)} - \overline{XY}^{(2k)},
                \overline{X^2}^{(2k)'} - \overline{XY}^{(2k)'}\right) 
    = \begin{cases}
        \asymp \frac{1}{k^2} & \text{if } |\Skk_1 \cap \Skk_2| = 1; \\
        \calO(\frac{1}{k})    & \text{if } |\Skk_1 \cap \Skk_2| = 2k,
    \end{cases}
\end{align}
Thus, combining the above conclusion \eqref{eq:layer3 approach 2, sigma1 and sigma2k}, the simplification form of $\psi(\Skk)$ \eqref{eq:linear kernel:layer1 and 2 psi: combine}, and the definition of $\sigma_{c, 2k}^2$ \eqref{eq:def:sigma_c^2}, we can bound $\sigma_{1, 2k}^2$ and $\sigma_{2k, 2k}^2$ as 
$
    \sigma_{1, 2k}^2  = \Theta(\frac{1}{n^2 k^2}), \,\,
    \sigma_{2k, 2k}^2  = \calO(\frac{1}{n^2 k}).
$
This implies that 
$\sigma_{2k, 2k}^2 / (2k \sigma_{1, 2k}^2) = \calO(1)$
Thus we conclude that
\begin{align*}
    \Var(\vu) = [1 + o(1)] \frac{4k^2}{n^2} \sigma_{1, 2k}^2 
    = \calO(\frac{1}{n^3}).
\end{align*}

This completes the proof of Lemma \ref{lemma:linear kernel, E(Vu) and Var(Vu)}.
\end{proof}
\begin{remark}
An alternative approach to bound $\Var(\vu)$ is by further simplifying $\vu$ as an order-2 U-statistic: $\hat V_u 
    =\frac{[1 + o(1)]}{n}  
    ( \overline{X^2}^{(n)} - \overline{XY}^{(n)})$.
We present the analysis in Appendix \ref{sec:proof:linear kernel prop}. The insight is that the lower-order structure of kernel $k$ implies the low order structure of $\vu$.
\end{remark}

The reremaining part of this section is organized as follows. 
In Section \ref{sec:intrinsic low order kernel}, we discuss the generalization of this 3-step analysis for low order kernel $h$ (see Definition \ref{def:low order kernel}). 
In Section \ref{sec:difficulities and tools for general h} we show the issues of adapting this the 3-step analysis for a general kernel $h$ (without any low order structure); then, we discuss the motivations of new tools to overcome this difficulty.

\subsection{Kernel \texorpdfstring{$h$}{h} with Low Order Structure}
\label{sec:intrinsic low order kernel}

The analysis in step 1 depends on the double U-statistic structure of $\vu$. It definitely works with any kernel $h$. To generalize theorem \ref{thm:linear kernel, ratio consitency}, it remains to discuss the analysis in steps 2 and 3. 

\begin{definition}[intrinsic low-order kernel]
\label{def:low order kernel}
$\exists  L \in \N^+$, which does not depend on $k$
\begin{align*}
    h(X_1, ..., X_k) = \sum_{l=1}^{L}  h^{(l)} (X_1,..., X_k),
\end{align*}
where $h^{(l)}$ has lower order-$l$ structure defined with the fixed order kernel $g^{(l)}$:
\begin{align*}
    h^{(l)}(X_1,...,X_k) =  \binom{k}{l}^{-1} \sum_{i_1 < ... < i_l} g^{(l)} (X_{i_1}, ... , X_{i_l}).
\end{align*}
\end{definition}
For example, if we fix  $L=1$ and $g^{(1)}$ as identical map, then $h$ becomes the linear average kernel we discussed in the last section: $h(X_1,..., X_k) = \frac{1}{k} \sum_{i=1}^{k} X_i$. 
Suppose $h(X_1, ..., X_k)$ has the structure in Definition \ref{def:low order kernel}.
In step 2, similar to Equation \eqref{eq:linear kernel, varphi_d form} and \eqref{eq:linear kernel, cancel in varphi}, we can specify the form of $(\varphi_d - \varphi_0)(\Skk)$ by plugging in $h$'s low order structure.
Therefore, we are still able to have cancellations resulting in the order of $\frac{\poly(d)}{k^2}$.
In step 3, additional $\calO(\frac{1}{n})$ rate comes from the variance of fixed order U-statistic. Similarly, since the summand in $\vu$ is in the form of $h(S_1) h(S_2)$, we can verify that $\vu$ is a linear combination of U-statistic up to kernel order $2L$, where $L$ is the low order parameter in Definition \ref{def:low order kernel}. Therefore, we reduce infinite order U-statistic to finite order U-statistic. 
Note that We only discussed the intrinsic low-order kernel cases without proof because they are trivial extensions of the previous arguments.
% but the calculations are much more tedious due to the higher-order kernel $h$.  

\subsection{Discussion about General Kernel \texorpdfstring{$h$}{h}}
\label{sec:difficulities and tools for general h}

For general kernel $h$ without low order structure, we show the difficulty in the 3-step analysis, which motivates our new strategies, including the assumptions (Appendix \ref{sec:append:assump}) and technical lemmas (Appendix \ref{sec:technical lemmas}). 

The first difficulty is that $\varphi_d (\Skk) - \varphi_0(\Skk)$ no longer has a simple expression since  $h(S)$ does not have an explicit low order form. 
However, we still believe that a smaller $d$ can imply a smaller difference between $\varphi_d$ and $\varphi_0$. 
Our remedy is to quantify the implicit cancellation in covariance $\etacd := Cov[\varphi_{d_1}(\Skk_1) - \varphi_{0}(\Skk_1), \varphi_{d_2}(\Skk_2) - \varphi_{0}(\Skk_1)]$ (see Equation \eqref{eq:def:eta}). The further decomposition of $\etacd$ involves the following covariance \eqref{eq:def:rho = Cov(h1h2, h3h4)}
\begin{align*}
    \covh :=Cov[h(S_1) h(S_2), h(S_3) h(S_4)].
\end{align*}
As demonstrated in Section \ref{sec:assumptions and notations}, $\covh$ is determined by 11 free parameters.  Assumption \ref{assumption:low dependency} reduce \emph{DoF} from 11 to 9 of 11, which enables us to further reduce \emph{DoF} to 6 in the  proof of Lemma \ref{lemma:eta_c (d1, d2)} and bound $\etacd = \calO(\fk_c / k^2)$ for finite $c, d_1, d_2$.
Note that a relaxation of Assumption \ref{assumption:low dependency} is presented in Appendix \ref{sec:append:assump}.

% \subsubsection{Issue in Step 3: Leading Covariance Term in \texorpdfstring{$\Var(\vu)$}{Var(hat Vu)}}
The second difficulty is to show that $\sigma_{1, 2k}^2$ (or its upper bound) dominating $\Var(\vu)$ \eqref{eq:Var(Vu) expression}. 
When we have low lower structure of $h$, $\vu$ can be simplified so it is easy to bound $\sigma_{2k, 2k}^2 / \sigma_{1, 2k}^2$. However, a general $h$ does not has a explicit form and we may not bound $\sigma_{2k, 2k}^2 / \sigma_{1, 2k}^2$. Hence, we adopt a fine-grained strategy, bound all $\sigma_{c, 2k}^2$ by a tighter bound on $\sigma_{c, 2k}$ for  $c = 1, 2, .., T_1$ (Proposition \ref{lemma:bound each sigma_c^2 for finite c}) and a looser bound on $\sigma_{c, 2k}$ for $c = T_1 + 1, ..., 2k$ (Proposition \ref{lemma:bound sigma_c^2 for any c}). 
This is more challenging than only focusing on two terms: $\sigma_{1, 2k}^2$ and $\sigma_{2k, 2k}^2$.

%%%%%%%%%%%%%%%%%%%%%%%%%%%%%%%%%%%%%%%%%%%

\subsection{Proof of the Propositions under Linear Average Kernel}
\label{sec:proof:linear kernel prop}

\begin{proof}[Proof of Prop \ref{prop:layer2 cancel in varphi}. ]
\label{proof:prop:layer2 cancel in varphi}

Consider  $S_1, S_2 \subset \Skk$, s.t. $|S_1 \cap S_2| = d$.
Recall $\varphi_d \left(\Skk \right)  = \frac{1}{M_d} \sum_{S_1, S_2 \subset \Skk, | S_1 \cap S_2 | = d} h(S_1) h(S_2)$, which is a U-stat with asymmetric kernel of $h(S_1) h(S_2)$ s.t. $|S_1, S_2| = d, S_1, S_2 \subset \Skk$.
Let first investigate the form of $h(S_1) h(S_2)$. Then $\varphi_d \left(\Skk \right)$ will be an average.

Let $S_1 = (X_1, ..., X_d, Y_1, ..., Y_{k-d})$, $S_2 = (X_1, ..., X_d, Z_1, ..., Z_{k-d})$, where all $X_i, Y_j, Z_j \in \Skk, i = 1, ..., d, j = 1, ..., k-d$ are independent.
\begin{align}
\label{eq:h(S_1) h(S_2) for linear kernel}
\begin{split}
h(S_1) h(S_2)
    &= \left[ 
        \frac{1}{k}
        \left(\sum_{i = 1}^{d} X_i + \sum_{i = 1}^{k-d} Y_i
        \right)\right]
    \left[ 
        \frac{1}{k}
        \left(\sum_{i = 1}^{d} X_i + \sum_{i = 1}^{k-d} Z_i
        \right)\right]
    \\
    &= \frac{1}{k^2}   
        \left[ 
        \sum_{i = 1}^{d} X_i^2 + 
        \sum_{i = 1}^{d} \sum_{i <j \leq d} 2 X_i X_j +
        \sum_{i = 1}^{d} \sum_{j = 1}^{k-d} X_i (Y_j + Z_j) +
        \sum_{i = 1}^{k-d} \sum_{j = 1}^{k-d} Y_i  Z_j
        \right].
\end{split}
\end{align}

In  Equation \eqref{eq:h(S_1) h(S_2) for linear kernel}, the proportion of squared terms $X_i^2$ is $\frac{d}{k^2}$; the sum of proportions of cross terms  $X_i X_j, X_i Y_j$, $X_i Z_j$ and $Y_i Z_j$ is $1 - \frac{d}{k^2}$.
Since $\varphi(\Skk)$ is in the form of U-statistic, it can be viewed as an average. In this average,  the proportion of squared terms and cross terms remain unchanged as $\frac{d}{k^2}$ and $1 - \frac{d}{k^2}$.
WLOG, we denote $\Skk = (X_1, ..., X_{2k})$. We can derive the expression of $\varphi_d (\Skk)$, for $d = 0, 1, 2, ..., k$:
\begin{align}
\label{eq:linear kernel, varphi_d form}
\varphi_d(\Skk)
    = \frac{d}{k^2} \left(
        \frac{1}{2k} \sum_{i = 1}^{2k} X_i^2
        \right)
    + \frac{k^2 - d}{k^2} \left(
        \frac{2}{2k (2k-1)} \sum_{i = 1}^{2k} \sum_{j > i} X_i X_j
        \right).
\end{align}
Then, by the above equation
\begin{align}
\label{eq:linear kernel, cancel in varphi}
\varphi_d(\Skk) - \varphi_0(\Skk)
    = \frac{d}{k^2} 
    \left[
    \left(
        \frac{1}{2k} \sum_{i = 1}^{2k} X_i^2
        \right)
    -  \left(
        \frac{2}{2k (2k-1)} \sum_{i = 1}^{2k} \sum_{j > i} X_i X_j
        \right)
    \right],
\end{align}
\end{proof}
where $1 - \frac{d}{k^2}$ fraction of cross terms are cancelled.

\begin{proposition}
\label{prop:layer3 var(x bar)}
Given the linear kernel structure: $h(X_1, ..., X_k) = \frac{1}{k} \sum_{i=1}^k X_k$, 
$\Var(\overline{X^2}^{(n)} - \overline{XY}^{(n)}) = \calO(\frac{1}{n})$
\end{proposition}

\begin{proof}{Proof of Proposition \ref{prop:layer3 var(x bar)}. }
\label{proof:prop:layer3 var(x bar)}
\begin{align*}
    \Var(\overline{X^2}^{(n)} - \overline{XY}^{(n)}) 
    =  \Var(\overline{X^2}^{(n)}) + \Var(\overline{XY}^{(n)}) 
        - 2\Cov(\overline{X^2}^{(n)}, \overline{XY}^{(n)}).
\end{align*}
Since $\E(X) = 0$ and $\E(X_i X_j, X_i' X_j') \neq 0$ iff $i = i', j = j'$, we it is trivial to verify that 
\begin{align*}
     \Var(\overline{X^2}^{(n)}) = \frac{1}{n} \Var(X_1^2), \,\,
     \Var(\overline{XY}^{(n)}) = \binom{n}{2}^{-1} \Var^2(X_1), \,\,\text{and}\,\,
     \Cov(\overline{X^2}^{(n)}, \overline{XY}^{(n)}) = 0.
\end{align*}
Thus, $\Var(\overline{X^2}^{(n)} - \overline{XY}^{(n)}) = \calO(\frac{1}{n})$.

In particular, we decompose $\Var(\overline{XY}^{(n)})$ into  $\binom{n}{2}^2$ pairs of covariance,  where only $\binom{n}{2}$ pairs have non-zero. Besides,
for $\Cov(\overline{X^2}^{(n)}, \overline{XY}^{(n)})$, every  $\Cov(X_i^2, X_i' X_j') = \E(X_i^2, X_i' X_j') = 0$.

\end{proof}

\begin{proposition}
\label{prop:layer3 approach 2, sigma1 and sigma2k}
\begin{align}
    \Cov \left(\overline{X^2}^{(2k)} - \overline{XY}^{(2k)},
                \overline{X^2}^{(2k)'} - \overline{XY}^{(2k)'}\right) 
    = \begin{cases}
        \asymp \frac{1}{k^2}, & \text{when } |\Skk_1 \cap \Skk_2| = 1; \\
        \calO(\frac{1}{k}),    & \text{when } |\Skk_1 \cap \Skk_2| = 2k,
    \end{cases}
\end{align}
where
$\overline{X^2}^{(2k)} =  \frac{1}{2k} \sum_{X_i \in \Skk_1}^n X_i^2$;
$\overline{XY}^{(2k)} =  \frac{1}{k(2k-1)} \sum_{X_i, X_j \in \Skk_1, i < j} X_i X_j$

$\overline{X^2}^{(2k)'} =  \frac{1}{2k} \sum_{X_i \in \Skk_2}^n X_i^2$;
$\overline{XY}^{(2k)'} =  \frac{1}{k(2k-1)} \sum_{X_i, X_j \in \Skk_2, i < j}  X_i X_j$
\end{proposition}

%%%%%%%%%%%%%%%%
\begin{proof}[Proof of Proposition \ref{prop:layer3 approach 2, sigma1 and sigma2k}. ]
\label{proof:prop:layer3 approach 2, sigma1 and sigma2k}
Under linear average structure of $h$, i.e., $h(X_1, ..., X_k) = \frac{1}{k} \sum_{i = 1}^k X_i$, 
\begin{align*}
\Cov \left(\overline{X^2}^{(2k)} - \overline{XY}^{(2k)},
        \overline{X^2}^{(2k)'} - \overline{XY}^{(2k)'}\right).
\end{align*}

\emph{Part 1: $|\Skk_1 \cap \Skk_2| = 1$}

W.L.O.G, assume $\Skk_1 = (X_1, X_2, ..., X_k), \Skk_1 = (X_1, X_2', ..., X_k')$, $\Skk_1 \cap \Skk_2 = X_1$, $X_1, .., X_k, X_2', ..., X_k'$ are independent.

For $\Cov(A, B)$, let us only consider $A$, $B$ have overlap; otherwise, it is 0. Therefore, three parts of covariance can be simplified as
\begin{align*}
    &\Cov \left(\overline{X^2}^{(2k)},  \overline{X^2}^{(2k)'}\right) 
    = \frac{1}{4k^2} \Var (X_1^2), \\
    &\Cov \left(\overline{XY}^{(2k)},  \overline{XY}^{(2k)'}\right)
    = \Cov\left(\frac{2}{2k-1} \sum_{j=2}^{k} X_1 X_j, \frac{2}{2k-1}\sum_{j=2}^{2k} X_1 X_j'\right) = 0,
    \\
    & \Cov \left(\overline{X^2}^{(2k)}, \overline{XY}^{(2k)'}\right)
    = \Cov \left(\frac{1}{4k^2} X_1^2, \frac{2}{2k-1}\sum_{j=2}^{2k} X_1 X_j'\right) = 0.
\end{align*}
Since $\Var(X_1) > 0$, we have 
$$\Cov \left(\overline{X^2}^{(2k)} - \overline{XY}^{(2k)},
        \overline{X^2}^{(2k)'} - \overline{XY}^{(2k)'}\right) 
        \asymp \frac{1}{k^2}.$$

\emph{Part 2: $|\Skk_1 \cap \Skk_2| = 2k$}

We borrow the proof of Proposition \ref{prop:layer3 var(x bar)} by replacing  $n$ with $2k$. In other words, we get coefficient $\calO(\frac{1}{k})$ because this is the variance of a fixed-order kernel with $2k$-sample U-statistic.
We have 
\begin{align*}
    \Cov \left(\overline{X^2}^{(2k)} - \overline{XY}^{(2k)},
        \overline{X^2}^{(2k)'} - \overline{XY}^{(2k)'}\right) 
        = \Var\left(\overline{X^2}^{(2k)} - \overline{XY}^{(2k)}\right)
        = \calO(\frac{1}{k}).
\end{align*}
\end{proof}

% \end{document}

% \documentclass[../main.tex]{subfiles}
% \begin{document}

\section{Additional Simulation Results}
\label{sec:append:simu}

\subsection{Ground Truth in the Simulation}

As mentioned in Section \ref{sec:subsec:simulation setting}, we simulate the ground truth of the expectation of forest predictions: $\E (f(\bm x^*))$ and the variance of forest predictions: $\Var (f(\bm x^*))$ by 10000 simulations. 
Since variance estimators are produced by different packages, we use the corresponding package to generate their ground truth.
The result of the central testing sample (see Section \ref{sec:simulation}) is presented in Table \ref{tab:ground truth E} and Table \ref{tab:ground truth Var}. We observe that there is a small difference between different packages though similar tunning parameters are used to train random forests.

In addition, we present the ``oracle'' CI coverage rate in  Table \ref{tab:CI constructed by true variance}, which  matches $1 - \alpha$. To construct these CIs,  we still use the random forest prediction over 1000 simulations but replace the estimated variance with the ``true variance'', $\Var (f(\bm x^*))$. 
This result also shows the normality of the random forest predictor.

\begin{table}[htbp]
\centering
\caption{Ground Truth of $\E (f(\bm x^*))$ evaluated on central testing sample by 10000 simulations. The number in the bracket is the standard deviation of $\E (f(\bm x^*))$. }
\label{tab:ground truth E}
\begin{tabular}{lll|ll|ll} 
\hline
\begin{tabular}[c]{@{}l@{}}\\\end{tabular} &                                   &        & \multicolumn{2}{c|}{MARS}     & \multicolumn{2}{c}{MLR}          \\ 
\hline
\multicolumn{2}{c}{Tree size}                                                  & nTrees & \if0\blind {RLT}\fi \if1\blind{our package}\fi    & grf/ranger   & \if0\blind {RLT}\fi \if1\blind{our package}\fi     & grf/ranger     \\ 
\hline
\multirow{6}{*}{$k \leq n/2$}              & \multirow{2}{*}{$k=n/2$}          & 2000   & 17.82
  (0.01) & 18.18 (0.01) & 0.503
  (0.004) & 0.498 (0.004)  \\
                                           &                                   & 20000  & 17.82 (0.01)   & 18.18 (0.01) & 0.503 (0.004)   & 0.499 (0.004)  \\ 
\cline{2-7}
                                           & \multirow{2}{*}{$k=n/4$}          & 2000   & 17.45 (0.01)   & 18.00 (0.01) & 0.503 (0.003)   & 0.468 (0.003)  \\
                                           &                                   & 20000  & 17.45 (0.01)   & 18.00 (0.01) & 0.503 (0.003)   & 0.468 (0.003)  \\ 
\cline{2-7}
                                           & \multirow{2}{*}{$k=n/8$}          & 2000   & 17.41 (0.01)   & 18.19 (0.01) & 0.503 (0.002)   & 0.424 (0.002)  \\
                                           &                                   & 20000  & 17.41 (0.01)   & 18.18 (0.01) & 0.503 (0.002)   & 0.424 (0.002)  \\ 
\hline
\multirow{2}{*}{$k > n/2$}                 & \multirow{2}{*}{$k=4n/5$} & 2000   & 18.21 (0.01)   & 18.19 (0.01) & 0.499 (0.005)   & 0.498 (0.005)  \\
                                           &                                   & 20000  & 18.21 (0.01)   & 18.19 (0.01) & 0.498 (0.005)   & 0.498 (0.005)  \\
\hline
\end{tabular}
\end{table}

\begin{table}[htbp]
\centering
\caption{Ground Truth of $\Var (f(\bm x^*))$ evaluated on the central testing sample by 10000 simulations.}
\label{tab:ground truth Var}
\begin{tabular}{lll|ll|ll} 
\hline
\begin{tabular}[c]{@{}l@{}}\\\end{tabular} &                                   &        & \multicolumn{2}{c|}{MARS} & \multicolumn{2}{c}{MLR}   \\ 
\hline
\multicolumn{2}{c}{Tree size}                                                  & nTrees & \if0\blind {RLT}\fi \if1\blind{our package}\fi & grf/ranger  & \if0\blind {RLT}\fi \if1\blind{our package}\fi & grf/ranger  \\ 
\hline
\multirow{6}{*}{$k \leq n/2$}              & \multirow{2}{*}{$k=n/2$}          & 2000   & 0.859       & 0.814       & 0.130        & 0.132       \\
                                           &                                   & 20000  & 0.851       & 0.811       & 0.129       & 0.131       \\ 
\cline{2-7}
                                           & \multirow{2}{*}{$k=n/4$}          & 2000   & 0.523       & 0.527       & 0.075       & 0.077       \\
                                           &                                   & 20000  & 0.517       & 0.519       & 0.074       & 0.077       \\ 
\cline{2-7}
                                           & \multirow{2}{*}{$k=n/8$}          & 2000   & 0.349       & 0.378       & 0.044       & 0.044       \\
                                           &                                   & 20000  & 0.342       & 0.370        & 0.043       & 0.044       \\ 
\hline
\multirow{2}{*}{$k > n/2$}                 & \multirow{2}{*}{$k=4n/5$} & 2000   & 1.334       & 1.348       & 0.214       & 0.213       \\
                                           &                                   & 20000  & 1.331       & 1.341       & 0.213       & 0.212       \\
\hline
\end{tabular}
\end{table}

\begin{table}[htbp]
\centering
\caption{$90\%$ CI Coverage Rate averaged on 50 testing samples, where the true variance is used in constructing the CI. The number in the bracket is the standard deviation of coverage over 50 testing samples. IJ estimator is performed by \texttt{grf} package when $k \leq  n/2$ and \texttt{ranger} package when $k > n/2$.}
\label{tab:CI constructed by true variance}
\resizebox{\textwidth}{!}{
\begin{tabular}{lll|ll|ll}
\hline
  & \multicolumn{1}{l}{}      &        & \multicolumn{2}{c|}{MARS}           & \multicolumn{2}{c}{MLR}             \\ \hline
\multicolumn{2}{c}{Tree size} & nTrees & \if0\blind {RLT}\fi \if1\blind{our package}\fi      & grf/ranger       & \if0\blind {RLT}\fi \if1\blind{our package}\fi      & grf/ranger       \\ \hline
\multirow{6}{*}{$k \leq n/2$} & \multirow{2}{*}{$k=n/2$}                              & 2000 & 90.12\% (0.93\%) & 90.00\% (0.97\%) & 89.97\% (0.86\%) & 90.04\% (0.99\%) \\
  &                           & 20000  & 90.10\% (0.96\%) & 89.95\% (0.97\%) & 89.97\% (0.88\%) & 89.96\% (1.00\%) \\ \cline{2-7} 
  & \multirow{2}{*}{$k=n/4$}  & 2000   & 89.87\% (0.76\%) & 89.69\% (0.84\%) & 90.07\% (1.03\%) & 90.06\% (1.25\%) \\
  &                           & 20000  & 89.83\% (0.78\%) & 89.63\% (0.82\%) & 90.14\% (1.04\%) & 89.98\% (1.21\%) \\ \cline{2-7} 
  & \multirow{2}{*}{$k=n/8$}  & 2000   & 89.53\% (0.78\%) & 89.35\% (0.85\%) & 90.22\% (1.13\%) & 89.91\% (1.17\%) \\
  &                           & 20000  & 89.38\% (0.89\%) & 89.28\% (0.85\%) & 90.20\% (1.12\%) & 89.78\% (1.23\%) \\ \hline
\multirow{2}{*}{$k > n/2$}    & \multicolumn{1}{l}{\multirow{2}{*}{$k=\frac{4}{5}n$}} & 2000 & 90.05\% (0.94\%) & 90.02\% (0.97\%) & 89.86\% (1.05\%) & 89.94\% (0.98\%) \\
  & \multicolumn{1}{l}{}      & 20000  & 90.05\% (1.01\%) & 90.00\% (0.96\%) & 89.88\% (0.98\%) & 89.86\% (0.97\%) \\ \hline
\end{tabular}
}
\end{table}

\subsection{Figures of MLR model}

Figure \ref{fig:3 panel results, MLR} shows the performance of different methods on the MLR model. This is a counterpart of Figure \ref{fig:3 panel results} in Section \ref{sec:simulation}. 

\begin{figure}[htbp]
    \centering
      \begin{minipage}[b]{\textwidth}
        \includegraphics[width=\textwidth]{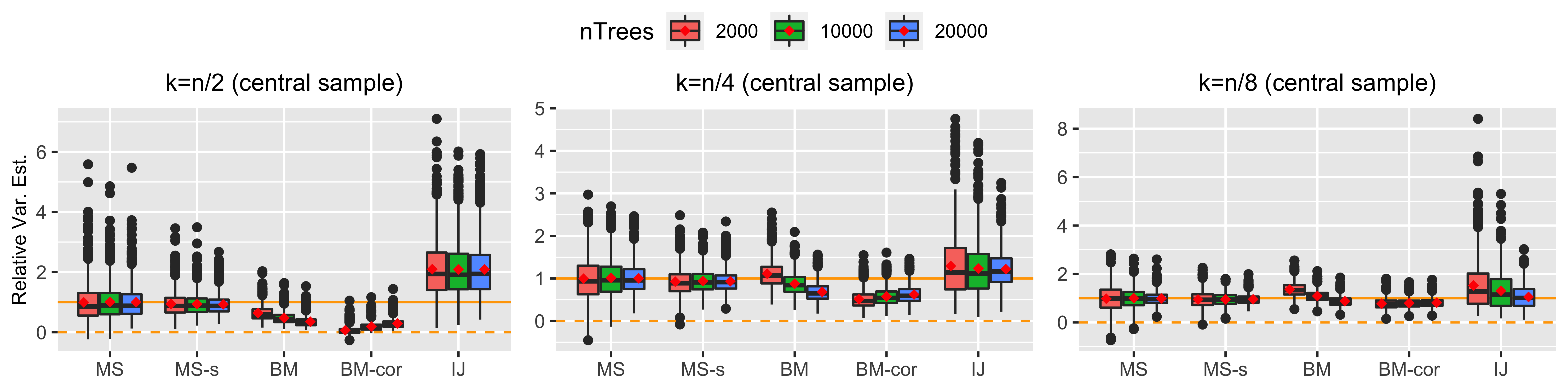}
      \end{minipage}
      \hfill
      \begin{minipage}[b]{\textwidth}
        \includegraphics[width=\textwidth]{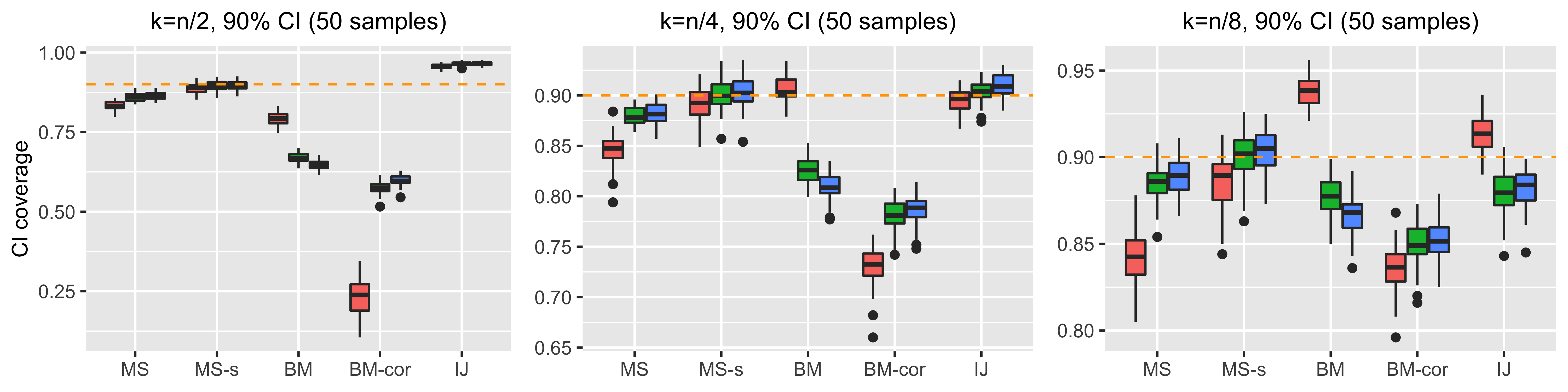}
      \end{minipage}
      \hfill
        \begin{minipage}[b]{\textwidth}
        \includegraphics[width=\textwidth]{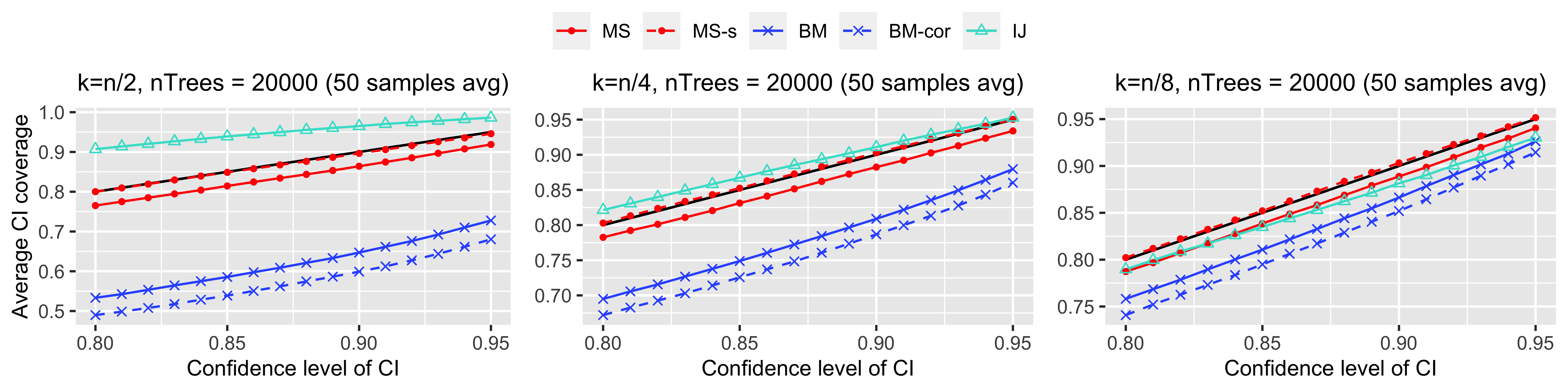}
      \end{minipage}
    \caption{A comparison of different methods on MLR  data. 
    Each column of figure panel corresponds to one tree size: $k = n/2, n/4, n/8$.
    The first row: boxplots of relative variance estimators of the central test sample over 1000 simulations. The diamond symbol in the boxplot indicates the mean.
    The second row:  boxplots of 90\% CI coverage for 50 testing samples. 
    For each method, three side-by-side boxplots represent {\nTrees} as 2000, 10000, 20000. 
    The third row: the coverage rate averaged over 50 testing samples with 20000 {\nTrees}  and the confidence level
    (x-axis) from 80\% to 95\%. The black reference line $y=x$ indicates the desired coverage rate.}
    \label{fig:3 panel results, MLR}
\end{figure}

\section{Additional Information and Results on the Real Data}
\label{sec:append:real}

Table \ref{tab:real data info} describes the covariates of Airbnb data in Section \ref{sec:real data}. 
We use the samples with the price falling in the interval $(0, 500]$ dollars. 
The missing values (NA) in the rating score and bathroom number are replaced. The ``having rating'' covariate is created based on the ``review number''.

\begin{table}[htbp]
\centering
\caption{Covariates information of Airbnb data.}
\label{tab:real data info}
\begin{tabular}{ll}
\hline
Covariate Name & Description \\ \hline
latitude & Latitude of the Airbnb unit. \\
longitude & Longitude of the Airbnb unit. \\
room type & Three types (with \# of samples): Entire home/apt (5547), Private \\
 & room (1839) and Shared room (129). \\
bedroom number & Number of bedrooms in this unit. \\
bathroom number & Number of bathrooms in this unit.  NA values are replaced by 0. \\
accommodates & Maximum accommodates of this unit. \\
reviews number & The number of reviews of this unit. \\
having a rating & It is 1 if the number of reviews is greater than 0; and is 0 otherwise. \\
rating score & The average rating score. NA is replaced by the average score. \\ \hline
\end{tabular}
\end{table}

To train the random forest model, we set {\mtry} (number of variables randomly sampled as candidates at each split) as 3, and set \texttt{nodesize} parameter as 36. 
Here we also present the details of testing samples. The latitude and longitude of SEA Airport, Seattle downtown, and Mercer Island are (47.4502, -122.3088),  (47.6050, -122.3344), and (47.5707, -122.2221) respectively. The ``room type''  ``accommodates'' and ``having a rating'' are fixed as  ``Entire home/apt'', the double of ``bedroom numbers'', and 1 respectively.  We use averages in the training data as the values of ``reviews number'' and ``rating score''.
% \end{document}

\end{appendices}

%\end{document}

\end{document}